\documentclass[11pt]{article} 
\pdfoutput=1
\usepackage[utf8]{inputenc}
\usepackage{amsmath, amsfonts, amsthm, amssymb, graphicx}
\usepackage{mathtools,verbatim,bm}
\usepackage{enumitem,etoolbox}
\usepackage{natbib}
\usepackage{xfrac}

\usepackage{pdfpages,tikz}
\usetikzlibrary{bayesnet}
\usepackage[noend]{algpseudocode}
\usepackage[linesnumbered,ruled]{algorithm2e}
\usepackage{hyperref}
\hypersetup{colorlinks,citecolor=blue}
\usepackage{fancyhdr, lastpage}
\usepackage[vmargin=1.00in,hmargin=1.00in,centering,letterpaper]{geometry}
\DeclareMathOperator{\BigOm}{\mathcal{O}}
\newcommand{\BigOh}[1]{\BigOm\left({#1}\right)}

\DeclareMathOperator{\BigTm}{\Theta}
\newcommand{\BigTheta}[1]{\BigTm\left({#1}\right)}
\DeclareMathOperator{\BigWm}{\Omega}
\newcommand{\BigOmega}[1]{\BigWm\left({#1}\right)}

\newcommand{\unit}[1]{\mathsf{unit}\left({#1}\right)}

\newcommand{\matxsttil}{\widetilde{\matx}_{\star}}
\newcommand{\eventbound}{\calE_{\mathrm{bound}}}
\newcommand\numberthis{\addtocounter{equation}{1}\tag{\theequation}}
\newcommand{\much}{\check{\mu}}

\newcommand{\rmMMSE}{\mathrm{MMSE}}

\newcommand{\cross}{\mathrm{Cross}}
\newcommand{\Ych}{\check{\matY}}
\newcommand{\Xch}{\check{X}}
\newcommand{\Wch}{\check{W}}

\newcommand{\Zch}{\check{\matZ}}
\newcommand{\Ach}{\check{A}}
\newcommand{\Mch}{\check{\matM}}
\newcommand{\uch}{\check{\matu}}
\newcommand{\matXch}{\check{\matX}}

\newcommand{\MMSEch}{\check{\rmMMSE}}
\newcommand{\crossch}{\check{\cross}}
\newcommand{\Psfch}{\check{\Psf}}
\newcommand{\Psfoff}{\Psf^{\mathrm{off}}}
\newcommand{\crossoff}{\cross^{\mathrm{off}}}
\newcommand{\Hch}{\check{H}}
\newcommand{\Hoff}{H^{\mathrm{off}}}

\newcommand{\Freech}{\check{F}}

\newcommand{\Hamil}{H}

\newcommand{\Psf}{\bm{\mathsf{P}}}

\newcommand{\Fbar}{\overline{F}}

\DeclareMathOperator*{\argmax}{arg\,max}

\newcommand{\Xbar}{\overline{X}}

\newcommand{\matZbar}{\overline{\matZ}}

\newcommand{\errhat}{\widehat{\mathsf{\Delta}}}

\newcommand{\Ecross}{\calE_{\mathrm{cross}}}
\newcommand{\iden}{I}

\newcommand{\Eovlap}{\calE_{\mathrm{ovlp}}}

\newcommand{\Ealpha}{\calE_{\boldalpha}}

\newcommand{\EA}{\calE_{\matA}}

\newcommand{\Ex}{\calE_{\xst}}

\newcommand{\matw}{\mathbf{w}}

\newcommand{\matZ}{\mathbf{Z}}
\newcommand{\matz}{\mathbf{z}}

\newcommand{\matx}{\mathbf{x}}
\newcommand{\matxst}{\mathbf{x}_{\star}}
\newcommand{\xst}{\mathbf{x}_{\star}}

\newcommand{\matu}{\mathbf{u}}
\newcommand{\matW}{\mathbf{W}}
\newcommand{\matM}{\mathbf{M}}
\newcommand{\matB}{\mathbf{B}}
\newcommand{\matA}{\mathbf{A}}

\newcommand{\matutil}{\widetilde{\matu}}
\newcommand{\matubar}{\overline{\matu}}

\newcommand{\quielt}{\mathfrak{q}}
\newcommand{\uhat}{\widehat{\mathsf{u}}}

\newcommand{\calV}{\mathcal{V}}

\newcommand{\calD}{\mathcal{D}}

\newcommand{\smallo}{\overline{\mathbf{o}}_d(1)}

\newcommand{\sigalg}{\mathsf{\sigma}}
\newcommand{\bnot}{\mathbf{b}}

\newcommand{\baro}{\overline{\mathbf{O}}_{d}(1)}
\newcommand{\Symm}[1]{\mathrm{Symm}\left({#1}\right)}

\newcommand{\calO}{\mathcal{O}}

\newcommand{\calE}{\mathcal{E}}

\newcommand{\calF}{\mathcal{F}}

\newcommand{\info}{\mathbf{i}}

\newcommand{\MMSE}{\mathrm{MMSE}_{d,\lambda}}

\newcommand{\ovlapd}{\mathtt{ovlap}_{d,\lambda}}
\newcommand{\ovlapst}{\mathtt{ovlap}_*}
\newcommand{\vspan}{\mathrm{span}}

\newcommand{\itSigma}{\mathsf{\Sigma}}
\newcommand{\itP}{\mathsf{P}}
\newcommand{\itT}{\mathsf{T}}
\newcommand{\itS}{\mathsf{S}}

\newcommand{\itV}{\mathsf{V}}

\newcommand{\itVtil}{\widetilde{\mathsf{V}}}

\newcommand{\ovlapdbar}{\overline{\ovlapd}}

\newcommand{\matMtil}{\widetilde{\matM}}

\newcommand{\matY}{\mathbf{Y}}

\newcommand{\matAtil}{\widetilde{\matA}}

\newcommand{\taub}{\tau_{0}}
\newcommand{\mult}{\nu}

\newcommand{\RQA}{\text{RQA}}

\newcommand{\matb}{\mathbf{b}}
\newcommand{\matbtil}{\widetilde{\matb}}

\newcommand{\matX}{\mathbf{X}}

\newcommand{\shiftinvert}{\mathsf{Shift}\hbox{-}\mathsf{and}\hbox{-}\mathsf{Invert}}

\newcommand{\xhat}{\widehat{\mathsf{x}}}
\newcommand{\yhat}{\widehat{\mathsf{y}}}

\newcommand{\calS}{\mathcal{S}}

\newcommand{\boldeps}{\overline{\epsilon}}

\newcommand{\boldlam}{\mathtt{\lambda}}

\newcommand{\boldalpha}{\boldsymbol{\alpha}}
\newcommand{\bolddel}{\boldsymbol{\Delta}}

\newcommand{\calN}{\mathcal{N}}
\newcommand{\gap}{\mathrm{gap}}
\newcommand{\boldgap}{\mathtt{gap}}

\DeclareMathAlphabet{\mathbfsf}{\encodingdefault}{\sfdefault}{bx}{n}
\newcommand{\Prit}{\mathbfsf{P}}

\newcommand{\Symd}{\mathbb{S}^{d}}
\newcommand{\Alg}{\mathsf{Alg}}

\newcommand{\cond}{\mathrm{cond}}
\newcommand{\conddet}{\mathtt{cond}}

\newcommand{\rmd}{\mathrm{d}}
\newcommand{\tr}{\mathrm{tr}}

\newcommand{\op}{\mathrm{op}}
\newcommand{\F}{\mathrm{F}}

\newcommand{\sphere}{\calS^{d-1}}

\newcommand{\Proj}{\mathsf{Proj}}

\newcommand{\N}{\mathbb{N}}
\newcommand{\condnum}{\kappa}

\newcommand{\deltaexp}{\mathtt{\delta}_{\mult,\lambda}(d)}
\newcommand{\dellam}{\mathsf{\delta}_{\lambda}(d)}
\newcommand{\denom}{\mathtt{denom}}

\newcommand{\rmda}{\frac{\rmd}{\rmd a}}

\newcommand{\GOE}{\mathrm{GOE}}

\newcommand{\poly}{\mathrm{poly}}

\newcommand{\lambdainv}{\boldlam^{-1}}

\newcommand{\itZ}{\mathsf{Z}}

\newcommand{\vone}{\mathsf{v}^{(1)}}

\newcommand{\vT}{\mathsf{v}^{(\itT)}}
\newcommand{\vTplus}{\mathsf{v}^{(\itT+1)}}

\newcommand{\vk}{\mathsf{v}^{(k)}}

\newcommand{\vi}{\mathsf{v}^{(i)}}

\newcommand{\viminus}{\mathsf{v}^{(i-1)}}

\newcommand{\viT}{\mathsf{v}^{(i)\top}}

\newcommand{\wzero}{\mathbf{b}}
\newcommand{\wone}{\mathsf{w}^{(1)}}

\newcommand{\wk}{\mathsf{w}^{(k)}}

\newcommand{\wi}{\mathsf{w}^{(i)}}
\newcommand{\wiminus}{\mathsf{w}^{(i-1)}}

\newcommand{\wT}{\mathsf{w}^{(\itT)}}

\newcommand{\matxi}{\boldsymbol{\xi}}

\newcommand{\pnot}{\mathbf{P}_0}

\newcommand{\iidsim}{\overset{\mathrm{i.i.d.}}{\sim}}
\newcommand{\stielt}{\mathfrak{s}}

\renewcommand{\iff}{\text{ iff }}

\newcommand{\err}{\mathrm{Err}}

\newcommand{\PD}{\mathbb{S}_{++}^d}
\newcommand{\sphered}{\mathcal{S}^{d-1}}

\newcommand{\pto}{\overset{\mathrm{prob.}}{\to}}

\newcommand{\R}{\mathbb{R}}
\newcommand{\I}{\mathbb{I}}
\newcommand{\Exp}{\mathbb{E}}
\newcommand{\Q}{\mathbb{Q}}

\newcommand{\Var}{\mathrm{Var}}

\renewcommand{\Pr}{\mathbb{P}}

{
 \theoremstyle{plain}
      
}
\newtheorem{nono-theorem}{Theorem}[]
\theoremstyle{plain}
\newtheorem{thm}{Theorem}[section]
\newtheorem{claim}[thm]{Claim}
\newtheorem{lem}[thm]{Lemma}
\newtheorem{cor}[thm]{Corollary}
\newtheorem{fact}[thm]{Fact}

\newtheorem{prop}[thm]{Proposition}

\theoremstyle{definition}
\newtheorem{defn}{Definition}[section]
\newtheorem{conj}{Conjecture}[section]

\newtheorem{rem}{Remark}[section]

 \newtoggle{stocsub}
 \toggletrue{stocsub}

\title{On the Randomized Complexity of Minimizing a Convex Quadratic Function}

\author{Max Simchowitz\\
UC Berkeley \\
msimchow@berkeley.edu}

\begin{document}
\maketitle
\begin{abstract}
Minimizing a convex, quadratic objective of the form $f_{\mathbf{A},\mathbf{b}}(x) := \frac{1}{2}x^\top \mathbf{A} x - \langle \mathbf{b}, x \rangle$ for $\mathbf{A} \succ 0 $ is a fundamental problem in machine learning and optimization. In this work, we prove gradient-query complexity lower bounds for minimizing convex quadratic functions which apply to both deterministic and \emph{randomized} algorithms. Specifically, for any sufficiently large condition number $\kappa > 0$, we exhibit a distribution over $(\mathbf{A},\mathbf{b})$ with condition number $\mathrm{cond}(\mathbf{A}) \le \kappa$, such that any \emph{randomized} algorithm requires $\Omega(\sqrt{\kappa})$ gradient queries to find a solution $\widehat{\mathsf{x}}$ for which $\|\widehat{\mathsf{x}} - \mathbf{x}_{\star}\| \le \epsilon_0\|\mathbf{x}_{\star}\|$, where $\mathbf{x}_{\star} = \mathbf{A}^{-1}\mathbf{b}$ is the optimal soluton, and $\epsilon_0$ is a small constant. Setting $\kappa = 1/\epsilon$, this lower bound implies the minimax rate of $\mathsf{T} = \Omega(\lambda_1(\mathbf{A})\|\mathbf{x}_{\star}\|^2/\sqrt{\epsilon})$ queries required to minimize an arbitrary convex quadratic function up to error $f(\xhat) - f(\xst) \le \epsilon$. To our knowledge, this is the first lower bound for minimizing quadratic functions with noiseless gradient queries which both applies to randomized algorithms, and matches known upper bounds from Nesterov's accelerated method. In contrast,  the seminal lower bounds of \cite{nemirovskii1983problem} apply only to Krylov methods with a worst-case initialization, and a more recent lower bounds  due to \cite{agarwal2014lower} rely on an adversarial ‘resisting oracle’ which only applies to deterministic methods.

Our lower bound holds for a distribution derived from classical ensembles in random matrix theory, and relies on a careful reduction from adaptively estimating a planted vector $\mathbf{u}$ in a deformed Wigner model. A key step in deriving sharp lower bounds is demonstrating that the optimization error $\mathbf{x}_{\star} - \widehat{\mathsf{x}}$ cannot align too closely with $\mathbf{u}$. To this end, we prove an upper bound on the cosine between $\mathbf{x}_{\star} - \widehat{\mathsf{x}}$ and $\mathbf{u}$ in terms of the minimum mean-squared error (MMSE) of estimating the plant $\mathbf{u}$ in a deformed Wigner model. We then bound the MMSE by carefully modifying a result due to~\cite{lelarge2016fundamental}, which rigorously establishes a general replica-symmetric formula for planted matrix models.

\end{abstract}
\newpage
\section{Introduction}
The problem of minimizing convex, quadratic functions of the form $f_{\matA,\matb}(x) := \frac{1}{2}x^\top \matA x - \langle \matb, x \rangle$ for $\matA \succ 0 $ is a fundamental algorithmic primitive in machine learning and optimization.
Many popular approaches for minimizing $f_{\matA,\matb}$  can be characterized as ``first order'' methods, or algorithms which proceed by querying the gradients $\nabla f_{\matA,\matb}(x^{(i)})$ at a sequence of iterates $x^{(i)}$, in order to arrive at a final approximate minimum $\xhat$. Standard gradient descent, the heavy-ball method, Nesterov's accelerated descent, and conjugate-gradient can be all be expressed in this form. 

The seminal work of~\cite{nemirovskii1983problem} established that for a class of \emph{deterministic}, first order methods, the number of gradient queries required to achieve a solution $\xhat$ which approximates $\matxst := \arg\min_{x} \frac{1}{2}x^\top \matA x - \langle \matb, x \rangle = \matA^{-1}\matb$ has the following scaling:
\begin{itemize}
	\item \textbf{Condition-Dependent Rate:} To attain $\|\xhat - \matxst\|_2 \le \epsilon$, one needs $\BigTheta{\sqrt{\cond(\matA)} \log(1/\epsilon)}$, where $\cond(\matA) = \lambda_{\max}(\matA)/\lambda_{\min}(\matA)$.
	\item \textbf{Condition-Free Rate:} For any $\epsilon > 0$, there exists an $\matA,\matb$ such that to obtain  $f_{\matA,\matb}(\xhat) - f_{\matA,\matb}(\matxst) \le \epsilon \cdot \lambda_{1}(\matA) \|\matxst\|^2$, one needs $\BigTheta{\sqrt{1/\epsilon}}$ queries.\footnote{Note that $\lambda_{1}(\matA)$ is precisely the Lipschitz constant of $\nabla f_{\matA,\matb}$, and $\|\matxst\|^2$ corresponds to the Euclidean radius of the domain over which one is minimizing; see Remark~\ref{rem:main_thm}.}
\end{itemize}
It has long been wondered whether the above, worst-case lower bounds are reflective of the ``average case'' difficulty of minimizing quadratic functions, or if they are mere artificacts of uniquely 
adversarial constructions. 
For example, one may hope that randomness may allow a first order algorithm to avoid querying in worst-case, uninformative directions, at least for the initial few iterations.
Furthermore, quadratic objectives have uniform curvature, and thus local gradient exploration can provide global information about the function.

In this work, we show that in fact randomness does not substantially improve the query complexity of first order algorithms. Specifically, we show that even for randomized algorithms, (a) to obtain a solution $\|\xhat - \matxst\|_2 \le \epsilon_0$ for a small but universal constant $\epsilon_0$, one needs $\BigOmega{\sqrt{\cond(\matA)}}$ gradient queries, and, as a consequence, (b) for any $\epsilon > 0$, the condition-free lower bound of $\BigOmega{\epsilon^{-1/2}}$ queries for an $\epsilon$-approximate solution holds as well. These lower bounds are attained by explicit constructions of distributions over parameters $\matA$ and $\matb$, which are derived from classical models in random matrix theory. Hence, not only do our lower bounds resolve the question of the complexity of quadratic minimization with randomized first-order queries; they also provide compelling evidence that the worst-case and ``average-case'' complexity of quadratic minimization coincide up to constant factors. 

%

\subsection{Proof Ideas and Organization\label{sec:proof_ideas}}
	 Our argument draws heavily upon a lower bound due to~\citet{simchowitz2018tight} for approximating the top eigenvector of a deformed Wigner model, $\matM := \matW + \lambda \matu \matu^\top$, given a matrix-vector multiplication queries of the form $\wi = \matM \vi$. 
	 Here, $\matW$ is drawn from a Gaussian Orthogonal Ensemble (see Section~\ref{sec:reduction}), $\matu\sim \calN(0,I/d)$
	 \footnote{In~\cite{simchowitz2018tight}, $\matu$ was taken to be uniform on the sphere.}, and $\lambda > 1$ is a parameter controlling $\gap(\matM) := 1 - \frac{\lambda_2(\matM)}{\lambda_1(\matM)}$. That work showed that eigenvector approximation implies estimation of the so-called ``plant'' $\matu$, and showed that one required $\Omega(\gap(\matM)^{-1/2}\log d)$ queries to perform the estimation appropriately. 

	In this work, we show an analogous reduction: one can estimate $\matu$ if one can minimize the function $f_{\matA,\matb}(x)$, where $\matA = \gamma I - \matM$ for an appropriate $\gamma$, and $\matb$ is a Gaussian vector that is slightly correlated with $\matu$. 
	We also consider matrix vector multiply queries $\wi = \matM \vi$; these are equivalent both to querying $\matA \vi$, and to querying $\nabla f(\vi)$ (see Remark~\ref{rem:query_model}).

	The intuition behind our reduction comes from the $\shiftinvert$ meta-algorithm introduced by~\citet{garber2016faster}. 
	For epochs $s \in [\itS-1]$ and $\yhat^{(0)}$ uniform on the sphere, $\shiftinvert$ calls a black-box quadratic solver to produce iterates $\yhat^{(s+1)} \approx \matA^{-1}\yhat^{(s)} = \arg\min_{y} f_{\matA,\yhat^{(s)}}$.
	If the errors $\|\yhat^{(s+1)} -\matA^{-1}\yhat^{(s)}\|$ are sufficiently small and  if $\gamma$ is tuned appropriately one can show that 
	(a) $\cond(\matA) \approx 1/\gap(\matM)$ and
	(b) denoting the top eigenvector of $\matM$ by $v_1(\matM)$, the iterate $\yhat^{(\itS)}$ satisfies 
	\begin{align*}
	\langle \yhat^{(\itS)}, v_1(\matM) \rangle^2 \ge 1 - \epsilon,\quad  \text{where}~\itS = \Theta \left(\log (d/\epsilon)\right) \text{ is independent of }\gap(\matM)~.
	\end{align*} 
	In other words, $\shiftinvert$ reduces approximating the eigenvector of $\matM$ to minimizing a sequence of $\widetilde{\mathcal{O}}(1)$ convex quadratic functions $\{f_{\matA,\yhat^{(s-1)}}\}_{s \in [\itS]}$ with condition number $\BigOh{\frac{1}{\gap(\matM)}}$. Applying the lower bound for estimating $\matu$ from~\cite{simchowitz2018tight}, one should expect $\widetilde{\Omega}(\frac{1}{\sqrt\gap(\matM)}) = \widetilde{\Omega}(\sqrt{\cond(\matA)})$ queries on average to minimize these functions. 

	Unfortunately, applying the reduction in a black-box fashion requires high accuracy approximations of $\arg\min_{y} f_{\matA,\yhat^{(s)}}$; this mean that this reduction cannot be used to lower bound the query complexity required for constant levels of error $\epsilon_0$, and thus cannot be used to deduce the minimax rate. Our analysis therefore departs from the black-box reduction in that (a) we warm start $\yhat^{(0)} \leftarrow \bnot$ near the plant $\matu$ rather than from an isotropic distribution, (b) we effectively consider only the first iteration of the $\shiftinvert$ scheme, corresponding to finding $\xhat \approx \matA^{-1} \bnot$, and (c) we directly analyze the overlap between $\xhat$ and the plant $\matu$, $\langle \xhat, \matu \rangle^2$; the reduction is sketched in Section~\ref{sec:reduction}. 
	Moreover, we modify information-theoretic lower bounds for the estimation of $\matu$ from queries of $\matM$ to account for the additional information conveyed by the linear term $\bnot$ (see Section~\ref{sec:plant_estimation_lower}). 
	Altogether, our reduction affords us simpler proofs and an explicit construction of a ``hard instance''. Most importantly, the reduction tolerates constants error between the approximate minimizer $\xhat$ and the optimum $\xst = \matA^{-1}\matb$, which enables us to establish a sharp lower bound.

	In particular, to obtain a lower bound which matches known upper bounds up to constants, it is necessary to establish that the error $\xhat - \xst$ cannot align to closely with $\matu$. Otherwise, one could obtain a good approximation of $\xst$, namely $\xhat$, which was not sufficiently aligned with $\matu$. Since $\xhat - \xst$ is independent of $\matu$ given $\matM$ and $\matb$, we can bound their cosine in terms of the quantity
	\begin{align*}
	\mathtt{ovlap} := \max_{\uhat = \uhat(\matM,\matb) \in \calS^{d-1}}\Exp_{\matM,\matb,\matu}[\langle \uhat, \matu \rangle^2]~,
	\end{align*}
	which correponds to the largest alignment between $\matu$, and any $(\matM,\matb)$-measurable estimator $\uhat$ of the direction of $\matu$. We can relate this quantity to the minimum mean-squared error of estimating the plant $\matu$ in a deformed Wigner model. This can in turn be controlled by recent a result due to~\cite{lelarge2016fundamental}, which rigorously establishes a general replica-symmetric formula for planted matrix models. With this tool in hand, we prove Proposition~\ref{prop:ovlap_prop}, which gives an order-optimal bound on $\mathtt{ovlap}$ in terms of relevant problem parameters, provided that the ambient dimension $d$ is sufficiently large. We remark that the result of~\cite{lelarge2016fundamental} had been proven under additional restrictions by~\cite{barbier2016mutual}; see Section~\ref{sec:related} for related work and additional discussion.

	We cannot simply apply the bounds of~\cite{lelarge2016fundamental} out of the box, because (a) the former result does not allow for side information $\matb$, and (b) the former work consider a slighlty different observation model where only the off diagonals of $\matM$ are observed. In Section~\ref{sec:thm_asmp_lvap_proof}, we show that we can effectively remove the side information and reduce to a case where where $\matu_i \iidsim \calN(\boldalpha\mu/\sqrt{d},1)$ for an appropriate mean $\mu$ and a random scaling $\boldalpha$. Then, in Appendix~\ref{sec:MSE_Limit_Proof}, we carry out a careful interpolation argument in the spirit of the Wasserstein continuity of mutual information (see, e.g. \cite{wu2012functional}) to transfer the results from \cite{lelarge2016fundamental} to our observation model. This interpolation argument also lets us establish a version of uniform convergence, which is necessary to account for the random scaling $\boldalpha$.

	\textbf{Organization: } In Section~\ref{sec:main_results}, we formally introduce our formal query model and state our results; Section~\ref{sec:related} discusses related work. In Section~\ref{sec:proof_roadmap}, we sketch the main components of the proof. Section~\ref{sec:reduction} formally introduces the distribution over $\matA,\matb$ which witnesses our lower bound; it also presents Proposition~\ref{prop:ovlap_prop}, which bounds the term $\mathtt{ovlap}$, and gives the redunction from estimating the plant $\matu$ to approximately minimizing $f_{\matA,\matb}$.  Section~\ref{sec:redux_depth} gives a more in-depth proof roadmap for the reduction from estimation to optimization, which relies on non-asymptotic computations of the Stieltjes transition of $\matW$ and its derivatives. Lastly, Section~\ref{sec:plant_estimation_lower} fleshes out the proof of the lower bound for estimating $\matu$, and Section~\ref{sec:ovlap_sec} provides background information and a proof sketch for our bounds on $\mathtt{ovlap}$.

\subsection{Notation}
We shall use bold upper case letters (e.g. $\matM,\matA,\matW$) to denote (typically random) matrices related to a given problem instance, bold lower cause letters (e.g. $\matb,\matu,\matz$) to denote (typically random) vectors related to a problem instance, and serif-font ($\vi,\wi,\Alg,\xhat$) to denote quantities related to a given algorithm.
We use the standard notation $\|\cdot\|_2$, $\|\cdot\|_{\op}$, $\|\cdot\|_{\F}$ for the Euclidean 2-norm, matrix $\ell_2 \to \ell_2$ operator norm, and matrix Frobenius norm, respectively. 
We let $e_1,\dots,e_d \in \R^d$ denote the cannonical basis vectors in $\R^d$, let $\sphered := \{x \in \R^d:\|x\|_2 = 1\}$ denote the unit sphere, $\Symd:= \{M \in \R^{d \times d}: M = M^\top\}$ the set of symmetric matrices, and $\PD := \{M \in \Symd: M \succ 0\}$ the set of positive definite matrices.
For a matrix $\matA \in \Symd$, let $\lambda_{\max}(\matA) := \lambda_1(\matA) \ge \lambda_2(\matA) \dots \ge \lambda_d(\matA) = \lambda_{\min}(\matA)$ denote its eigevalues. For $\matA \in \PD$ and $\matb \in \R^d$, we let $\cond(\matA) := \lambda_1(\matA)/\lambda_d(\matA)$, and $f_{\matA,\matb}(x) := \frac{1}{2} x^\top \matA x - \langle \matb, x \rangle$.
Given vectors $v_1,\dots,v_k \in \R^d$, we let $\Proj_{v_1,\dots,v_k}$ denote the orthogonal projection onto $\mathrm{span}(\{v_1,\dots,v_k\})$. Lastly, given $x \in \R^n$, we let $\unit x = x/\|x\|$ if $x \ne 0$, and $\unit 0 = 0$.

\section{Main Results\label{sec:main_results}}
We begin by presenting a formal definition of our query model.
\begin{defn}[Randomized Query Algorithm]\label{defn:query} A \emph{randomized query algorithm} (RQA) $\Alg$ with \emph{query complexity} $\itT \in \mathbb{N}$ is an algorithm which interacts with an instance $(\matA,\matb) \in \PD \times \R^d$ via the following query scheme:\vspace{-.1cm}
\begin{enumerate}
	\item The algorithm recieves an initial input $\matb \in \R^d$ from an oracle.\vspace{-.1cm}
	\item For rounds $i \in [\itT]$, queries an oracle with a vector $\vi$, and receives a noiseless response $\wi = \matA \vi$. \vspace{-.1cm}
	\item At the end of $\itT$ rounds, the algorithm returns an estimate $\xhat \in \R^d$ of $\xst := \arg\min_x \frac{1}{2}x^\top \matA x - \langle b, x \rangle = \matA^{-1}\matb$.\vspace{-.1cm}
\end{enumerate}
The queries $\vi$ and output $\xhat$ are allowed to be randomized and adaptive, in that there is a random seed $\matxi$ such that $\vi$ is a function of $\{(\vone,\wone),\dots,(\viminus,\wiminus),\matxi\}$, and $\xhat$ is a function of $\{\matb,(\vone,\wone),\dots,(\vT,\wT),\matxi\}$.
\end{defn}
\begin{rem}\label{rem:query_model}
We remark that the above query model is equivalent to a querying exact gradient of the objective $f_{\matA,\matb}(x) := \frac{1}{2}x^\top \matA x - \langle b, x \rangle$. Indeed, $\matb = \nabla f_{\matA,\matb}(\mathbf{0})$, and $\matA \vi = \nabla f_{\matA,\matb}(\vi) - \nabla f_{\matA,\matb}(\mathbf{0})$. Thus, our query model encapsulates gradient descent, accelerated gradient descent, heavy-ball, and conjugate graident methods. Crucially, our query model differs from existing lower bounds by allowing for randomized queries as in~\cite{agarwal2014lower}, and by not requiring iterates to lie in the Krylov space spanned by past queries as in~\cite{nemirovskii1983problem}.
\end{rem}

We now state our main result, which shows that there exists a distribution over instances $(\matA,\matb)$ which matches the lower bounds of~\cite{nemirovskii1983problem}:
\begin{thm}[Main Theorem: Minimax Rate with Conjectured Polynomial Dimension]\label{thm:main_theorem} There exists a functions $d_0 : \R \to \N$ and universal constants $c_1,c_2,c_3 > 0$ such that the following holds. For $\condnum \ge 52$ and $d \ge \max\{d_0(\condnum),d_1(\condnum)\}$, there exists a joint distribution over instances $(\matAtil,\matbtil) \in \PD \times \R^{d}$ such that (a) $\cond(\matAtil) \le \condnum$ and (b) for any $d \ge d_1(\condnum)$ and any RQA $\Alg$ with query complexity $\itT < c_1\sqrt{\condnum}$ and output $\xhat$, we have that for $\matxsttil := \matAtil^{-1}\matbtil$,
\begin{align*}
\Pr_{\matAtil,\matbtil,\Alg}\left[ \left\{\|\xhat - \matxsttil\|_{2}^2 \le c_2\|\matxsttil\|_2^2\right\} \vee \left\{f_{\matAtil,\matbtil}(\xhat) - f_{\matAtil,\matbtil}(\matxsttil) \le c_2 \cdot \frac{\lambda_{1}(\matAtil)\|\matxsttil\|_2^2}{\condnum} \right\} \right] \le e^{-d^{c_3}}~,
\end{align*}
Moreover, $d_0 = \BigOh{\poly(\condnum)}$, and under a plausible conjecture, Conjecture~\ref{conj:nonasymp}, $d_1(\condnum) = \BigOh{\poly(\condnum)}$ as well. Here, $\Pr_{\matAtil,\matbtil,\Alg}$ refers to probability taken with respect to the random instance $\matAtil,\matbtil$, and the random seed $\matxi$.
\end{thm}

\begin{rem}\label{rem:main_thm}
Typically, convex optimization lower bounds are stated in terms of a strong convexity $\alpha$, a smoothness parameter $\beta$, and the radius of the domain, or distance between the first iterate and a global minimizer, $R = \|\xhat - x^{(0)}\|_2$ (see e.g.~\cite{bubeck2015convex}). For quadratics, the strong convexity parameter is $\alpha = \lambda_{\min}(\matAtil)$ and the smoothness parameter is $\beta = \lambda_{\max}(\matAtil)$; one can show that both these quantities are concentrate sharply in our particular distribution over $(\matAtil,\matbtil)$, and that $\lambda_{\max}(\matAtil)$ is at most a universal constant. As we are considering unconstrained optimization, the radius of the domain corresponds to $R = \|\matxsttil\|_2$. Indeed, the distribution of $(\matAtil,\matbtil)$ is rotationally symmetric, so a priori, the best estimate of $\matxsttil$ (before observing $\matbtil$ or querying $\matAtil$) is $\xhat = 0$. Hence the event  $\left\{f_{\matAtil,\matbtil}(\xhat) - f_{\matAtil,\matbtil}(\matxsttil)  \le   \frac{c_2\lambda_{1}(\matA)\|\matxsttil\|_2^2}{\condnum} \right\}$ can be interpreted as $\left\{f_{\matAtil,\matbtil}(\xhat) - f_{\matAtil,\matbtil}(\xst) \le \frac{c_2\beta R^2}{\kappa} \right\}$. Since one needs to have $\itT \ge c_1\sqrt{\kappa}$, we have that, with high probability,
\begin{align}\label{eq:minimax_rate}
f_{\matAtil,\matbtil}(\xhat) - f_{\matAtil,\matbtil}(\matxsttil) \ge \frac{c_2}{c_1} \cdot \frac{\beta R^2}{\itT^2}~.
\end{align}
which is which is the standard presentation of lower bounds for convex optimization. Similarly, the complement of the event $\left\{\|\xhat - \matxsttil\|_{2}^2 \le c_2\|\matxsttil\|_2^2\right\}$ can be rendered as 
\begin{align*}
\|\xhat - \matxsttil\|_2 \ge c_2\|\matxsttil\|^2\left(1 - \sqrt{\frac{1}{\condnum}}\right)^{T} \text{ for } \itT = c_1\sqrt{\condnum}~,
\end{align*}
where $\condnum = \cond(\matAtil) \ge \beta/\alpha$ is an upper bound on condition number.
\end{rem}
\begin{rem}[Scalings of $d_0,d_1$] In Theorem~\ref{thm:main_theorem}, the dimension $ d_0(\condnum)$ corresponds to how large the ambient dimension $d$ needs to be in order for $\matAtil$ to have the appropriate condition number, for approximations of $\matAtil^{-1}\matbtil$ to have sufficient overlap with $\matu$, assuming a bound on $\mathtt{ovlap}$, and for the lower bounds on estimating $\matu$ to kick in. For the sake of brevity, we show that $d_0$ is an unspecified polynomial in $\condnum$; characterizing the explicit dependence is possible, but would require great care, lengthier proofs, and would distract from the major ideas of the work.

The dimension $d_1(\condnum)$ captures how large $d$ must be in order to obtain the neccessary bound on $\mathtt{ovlap}$. Though $d_1(\condnum)$ is finite, we are only able to guarantee that the dependence on $\condnum$ is polynomial under a plausible conjecture, Conjecture~\ref{conj:nonasymp}, which requires that either (a) minimum-mean squared error of the estimate of the planted solution in a deformed Wigner model, or (b) the mutual information between the deformed Wigner matrix and the planted solution, converge to their asymptotic values at a polynomial rate.
\end{rem}

If non-conjectural bounds are desired which still guarantee that the dimension need only be polynomial in the condition number, we instead have the following theorem:
\begin{thm}[Main Theorem: Weaker Rate with Guaranteed Polynomial Dimension]\label{thm:polynomial_theorem} Let $c_1,c_2,c_3$ be as in Theorem~\ref{thm:main_theorem}, and let $d_0(\condnum) = \BigOh{\poly(\condnum)}$. Then for every $\condnum \ge 52$, there exists a distribution $(\matA,\matb)$ such that
$(\matA,\matb) \in \PD \times \R^{d}$ such that $\Pr[\cond(\matA) \le \condnum] \ge 1 - e^{-d^{c_3}}$ and for any $d \ge d_0(\condnum)$ and any RQA $\Alg$ with query complexity $\itT < c_1\sqrt{\condnum}$, we have that
\begin{align*}
\Pr_{\matA,\matb,\Alg}\left[ \left\{\|\xhat - \matxst\|_{2}^2 \le \frac{c_2}{\sqrt{\condnum}}\right\} \vee \left\{f_{\matA,\matb}(\xhat) - f_{\matA,\matb}(\xst) \le \frac{c_2}{\condnum^{3/2}} \right\} \right] \le e^{-d^{c_3}}~,
\end{align*}
\end{thm}
Note that Theorem~\ref{thm:polynomial_theorem} does not imply the minimax lower bound~\eqref{eq:minimax_rate}; however, it does show that to get to a modest accuracy in either $\|\xhat - \matxst\|_{2}^2$ or $f_{\matA,\matb}(\xhat) - f_{\matA,\matb}(\xst)$, one needs $\Omega(\sqrt{\cond}(\matA))$ queries.
\begin{rem}[The distributions $(\matA,\matb)$ and $(\matAtil,\matbtil)$] The distributions over $(\matAtil,\matbtil)$ from Theorem~\ref{thm:main_theorem} and $(\matA,\matb)$ from Theorem~\ref{thm:polynomial_theorem} differ subtly. The form of the distribution over $(\matA,\matb)$ is given explicitly at the beginning of Section~\ref{sec:reduction}, and is specialized for Theorem~\ref{thm:polynomial_theorem} by appropriately tuning parameters $\lambda = 1 + \sqrt{\frac{20}{\condnum}}$ and $\taub = (\lambda - 1)^2$. The distribution over $(\matAtil,\matbtil)$ is obtained by conditioning $(\matA,\matb)$ on a constant-probability, $(\matA,\matb)$-measurable event $\calE$ (see remarks following Proposition~\ref{prop:minimax_red}).  If one prefers, one can express Theorem~\ref{thm:main_theorem} as saying that, for the distribution $(\matA,\matb)$ as in Section~\ref{sec:reduction} and Theorem~\ref{thm:polynomial_theorem}, any algorithm with $\itT \le c_1\sqrt{\condnum}$ has a large error with \emph{constant} probability. However, by distinguishing between $(\matAtil,\matbtil)$ and $(\matA,\matb)$, we ensure that any algorithm incurs error with \emph{overwhelming}, rather than just \emph{constant}, probability. 
\end{rem}

\subsection{Related Work\label{sec:related}}
It is hard to do justice to the vast body of work on quadratic minimization and first order methods for optimization. We shall restrict the present survey to the lower bounds literature.

\textbf{Lower Bounds for Convex Optimization: }
The seminal work of~\cite{nemirovskii1983problem} established tight lower bounds on the number of gradient queries required to minimize quadratic objectives, in a model where the algorithm was (a) required to be deterministic (and was analyzed for a worst-case initialization), and (b) the gradient queries were restricted to lie in the linear span of the previous queries, known as the \emph{Krylov} space.~\cite{agarwal2014lower} showed that deterministic algorithms can be assumed to query in the Krylov space without loss of generality, but did not extend their analysis to randomized methods. \cite{woodworth2016tight} proved truly lower bounds against randomized first-order algorithms for finite-sum optimization of convex functions, but their constructions require non-quadratic objectives. Subsequent works generalized these constructions to query models which allow for high-order derivatives~\citep{agarwal2017lower,arjevani2017oracle}; these lower bounds are only relvant for non-quadratic functions, since a second order method can, by definition, minimize a quadratic function in one iteration. 

All aforementioned lower bounds, as well as those presented in this paper, require the ambient problem dimension to be sufficiently large as a function of relevant problem parameters; another line of work due to~\cite{arjevani2016iteration} attains dimension-free lower bounds, but at the expense of restricting the query model. 

\textbf{Lower Bounds for Stochastic Optimization:}
Lower bounds have also been established in the stochastic convex optimization~\citep{agarwal2009information,jamieson2012query}, where each gradient- or function-value oracle query is corrupted with i.i.d.\ noise, and~\cite{allen2016first} prove analogues of these bounds for streaming PCA. Other works have considered lower bounds which hold when the optimization algorithm is subject to memory constraints~\citep{steinhardt2015memory,steinhardt2015minimax,shamir2014fundamental}. While these stochastic lower bounds are information-theoretic, and thus unconditional, they are incomparable to  the setting considered in this work, where we are allowed to make exact, noiseless queries. 

\textbf{Query Complexity:}
Our proof casts eigenvector computation as a sequential estimation problem. These have been studied at length in the context of sparse recovery and active adaptive compressed sensing~\citep{arias2013fundamental,price2013lower,castro2017adaptive,castro2014adaptive}. Due to the noiseless oracle model, our setting is most similar to that of Price and Woodruff~\citep{price2013lower}, whereas other works~\citep{arias2013fundamental,castro2017adaptive,castro2014adaptive} study measurements contaminated with noise. 
More broadly, query complexity has received much recent attention in the context of communication-complexity~\citep{anshu2017lifting,nelson2017optimal}, in which lower bounds on query complexity imply corresponding bounds against communication via lifting theorems.  

\textbf{Estimation in the Deformed Wigner Model:} 
As mentioned in Section~\ref{sec:proof_ideas}, we require a result due to~\cite{lelarge2016fundamental} regarding the minimum mean squared error of estimation in a deformed Wigner model; this is achieved by establishing that the \emph{replica-symmetric} formula for mutual information in the deformed Wigner model holds in broad generality. The replica-symmetric formula had been conjectured by the statistical physics community (see~\cite{lesieur2015mmse}), and ~\cite{barbier2016mutual} and~\cite{krzakala2016mutual} had rigorously proven this formula under the restriction that the entries of the plant $\matu$ have finite support. In our application, $\matu$ has Gaussian entries, which is why we need the slightly more general result of~\cite{lelarge2016fundamental}. Later,~\cite{alaoui2018estimation} give a concise proof of the replica-symmetric formula, again under the assumption that $\matu$ has finite support.


\section{Proof Roadmap\label{sec:proof_roadmap}}
\subsection{Reduction from Estimation in the Deformed Wigner Model\label{sec:reduction}}
Our random instances will be parameterized by the quantities $\lambda \in (1,2]$, $\taub > 0$, and $d \in \N$; typically, one should think of $\lambda - 1$ as being on the order of $1/\sqrt{\cond(\matA)}$, and of $\taub = (\lambda-1)^2$, which is on the order of $1/\cond(\matA)$.
We say $c$ is a universal constant if it does not depend on the triple $(\lambda,\taub,d)$, and write $f(\lambda,\taub,d) \lesssim g(\lambda,\taub,d)$ as short hand for $f(\lambda,\taub,d) \le c \cdot g(\lambda,\taub,d)$, for some unspecified universal constant $c$. We shall also let $\dellam$ denote a term which is at most $c_4 e^{-c_1 d^{-c_2}(\lambda - 1)^{c_3}}$ for universal constants $c_1,\dots,c_4 > 0$. Given an event $\calE$, we note that writing $\Pr[\calE] \le \dellam$ allows us to encode constraints of the form $d^{c_2} \ge c_1 (\lambda - 1)^{c_3} \ge c_1$ (recall $\lambda \le 2$), since otherwise $\dellam \ge 1$ and the probability statement is vacuously true. In particular, we shall assume $d$ is sufficiently large that $d^{-.9} \le (\lambda - 1)^2$.

 For each $\lambda \in (1,2]$ and $d \in \N$, consider the \emph{deformed Wigner model}
	\begin{eqnarray}
	\matM := \lambda \matu \matu^{\top} + \matW~,
	\end{eqnarray}
	where $\matu \sim \mathcal{N}(0,\iden/d)$ is called the \emph{plant}, and $\matW$ is a $\GOE$ matrix, with $\matW_{ii} \sim \mathcal{N}(0,2)$ for $i \in [d]$, $\matW_{ij} \sim \mathcal{N}(0,1)$ and $\matW_{ji} := \matW_{ij}$ for for $1 \le i < j \le d$. With $\matu$ and $\matM$ defined above, we define our random instance $(\matA,\bnot)$ as 
	\begin{eqnarray}
	\matA := (2(\lambda + \lambda^{-1}) - 2) \iden - \matM \quad \text{ and } \quad \bnot \big{|} \matW,\matu~\sim~\mathcal{N}(\sqrt{\taub}\matu, \iden/d)~,
	\end{eqnarray}
	and let $\matxst := \matA^{-1}\matb \in \R^d$ denote the vector which exists almost surely, and when $\matA \in \PD$, is the unique minimizer of the quadratic objective $f_{\matA,\matb}(x) := \frac{1}{2}x^\top \matA x - \langle \bnot, x \rangle$. In this section, we give a high level sketch of the major technical building blocks which underly our main results in Section~\ref{sec:main_results}.

	The first step is to provide a reduction from \emph{estimation} to \emph{optimization}. Specifically, we must show that if the the approximate minimizer $\xhat$ returned by any $\RQA$ is close to the true optimal $\matxst$,
	then $\xhat$ has a large inner product with $\matu$. We must also ensure that we retain control over the conditioning of $\matA$ .
	%
	To this end, the parameter $\lambda \in (1,2]$ gives us a knob to control the condition number of $\matA$, and $\taub \le (\lambda - 1)^{-2}$ gives us control over to what extent we ``warm-start'' the algorithm near the true planted solution $\matu$. Specially, Proposition~\ref{prop:EA} implies that $\cond(\matA)$ will concentrate below 
	\begin{align*}
	\conddet(\lambda) := \frac{2(\lambda^2 + 1)}{(\lambda - 1)^2}  = \BigTheta{(\lambda-1)^{-2}} \quad \text{ as } \lambda \to 1,
	\end{align*}
	and standard arguments imply that $\langle \matu, \bnot \rangle^2$ concentrates around $\taub$. In Proposition~\ref{prop:Ex}, we show that if $\taub$ is is in some desired range, then then $\xst$ satisfies
	\begin{align}\label{eq:xst_overlap_one}
	\left\langle \unit\xst ,\matu \right\rangle^2 \gtrsim \frac{\taub}{\lambda - 1}~\text{ with high probability}.
	\end{align}
	In other words, the solution $\xst$ is about $1/(\lambda - 1)$-times more correlated with the plant $\matu$ than is $\matb$. This allows us to show that if $\xhat$ approximates $\xst$ up to sufficiently high accuracy, then we show in Section~\ref{sec:redux_depth} that one can achieve a solution which is $\gtrsim \taub/(\lambda -1)$ correlated with $\matu$:
	\begin{prop}[Reduction from Optimization to Estimation; First Attempt] \label{prop:not_minimax_red} For all $\lambda \in (1,2]$ and $\taub \in [d^{-.9},(\lambda - 1)^2]$, then $\matA,\matb$ as defined above satisfy 
	\begin{align*}
	\Pr_{\matA,\bnot,\matu,\Alg}\left[\langle \unit \xhat, \matu \rangle^2\gtrsim (\lambda - 1)  \right] \ge 
	\Pr_{\matA,\matb,\Alg}\left[ \frac{\|\xhat - \matxst\|_2^2}{\|\matxst\|_2^2} \lesssim \frac{\tau_0}{(\lambda -1)}\right] -\dellam~,
	\end{align*}
	and $\Pr[\cond(\matA) \le 2\conddet(\lambda)]  \le \dellam$.
	\end{prop}
	Proposition~\ref{prop:not_minimax_red} allows the $\taub$, the parameter controlling the correlation between $\bnot$ and $\matu$, to be \emph{vanishingly small} in the dimension. In fact, the condition $ \taub \ge d^{.9}$ can be replaced by $\taub \ge d^{1-\epsilon}$ for any $\epsilon > 0$, provided that the constants $c_1,\dots,c_4$ are ammended accordingly. Thus, our $\Omega(\sqrt{\cond(\matA)}$ lower bounds hold even when the linear term $\matb$ and the plant $\matu$ have little correlation, provide the solution accuracy is sufficiently high.
	Unfortunately, Proposition~\ref{prop:not_minimax_red} also requires that $\|\xhat - \xst\|$ be small. In fact, we can only take $\taub$ to be at most $(\lambda-1)^2$, yielding the bound
	\begin{align}\label{eq:taub_loose_bound}
	\Pr_{\matA,\bnot,\matu,\Alg}\left[\langle \unit \xhat, \matu \rangle^2\gtrsim \frac{\taub}{\lambda - 1} \right] \ge 
	\Pr_{\matA,\matb,\Alg}\left[ \frac{\|\xhat - \matxst\|_2^2}{\|\matxst\|_2^2} \lesssim (\lambda - 1)\right] -\dellam~,
	\end{align}
	which only applies if $\Alg$ can ensure $\frac{\|\xhat - \matxst\|_2^2}{\|\matxst\|_2^2} \lesssim (\lambda -1) \approx (\cond(\matA))^{-1/2}$. The minimax lower bounds, on the other hand, must apply as soon as $\frac{\|\xhat - \matxst\|_2^2}{\|\matxst\|_2^2}$ is some (possibly small) constant.

	To sharpen Proposition~\ref{prop:not_minimax_red}, we make the following observation: 
	whereas~\eqref{eq:xst_overlap_one} controls the overlap between $\xst$ and $\matu$, we are more precisely interested in the overlap between $\xhat$ and $\matu$. If the error $\xhat - \xst$ could align arbitrarily well with $\matu$, then we would only be able to tolerate small errors $\xhat - \xst$ to ensure large correlations $\langle \unit\xhat,\matu\rangle^2$. However, we observe that both $\xst$ and $\xhat$ are conditionally independent of $\matu$, given $\matA,\bnot$. Since conditioning on $(\matA,\bnot)$ is equivalent to conditioning on $(\matM,\bnot)$, we can bound the alignment between $\xhat - \xst$ and $\matu$ by viewing $\unit{\xhat - \xst}$ as an estimator $\uhat \in \sphered$, and bounding the quantity
	\begin{align*}
	\ovlapd(\taub) := \Exp_{\matA,\bnot}\max_{\uhat \in \sphered}  \Exp_{\matu}[\langle \uhat, \matu \rangle^2 \big{|} \matM,\bnot]~.
	\end{align*}
	Here, $\ovlapd(\taub)$ corresponds the largest possible expected alignment between $\matu$ and any vector possible estimator $\uhat$ depending on a \emph{total observation} of $\matM,\matb$. In particular, if $\ovlapd(\taub)$ is small, then the overlap between $\unit{\xhat - \xst}$ and $\matu$ is small in expectation. This idea leads to the following refinement of~\eqref{eq:taub_loose_bound}:
	\begin{prop}[Reduction from Optimization to Estimation; Sharpened Version]\label{prop:minimax_red} Let $\lambda \in (1,2]$ and set $\taub = (\lambda - 1)^2$. Then, there exists a distribution $\calD$ of instances $(\matAtil,\matbtil)$ with $\Pr[\matAtil \succ 0 \cap \cond(\matAtil) \le 2\conddet(\lambda)]  = 1$ such that, for $\matxsttil = \matAtil^{-1}\matbtil$
	\begin{align*}
	\Pr_{\matA,\bnot,\matu,\Alg}\left[\langle \unit \xhat, \matu \rangle^2 \gtrsim  (\lambda - 1) \right] \ge 
	\frac{1}{4}\Pr_{\matAtil,\matbtil \sim \calD}\Pr_{\Alg}\left[ \frac{\|\xhat - \matxsttil\|_2^2}{\|\matxsttil\|_2^2 } \lesssim \frac{(\lambda - 1)}{\ovlapd(\taub)}\right] -\dellam.
	\end{align*}
	The distribution $\calD$ is obtained by conditioning the distribution over $(\matA,\matb)$ on a constant-probability event, described in Section~\ref{sec:redux_depth}.
	\end{prop}
	The proofs of Proposition~\ref{prop:minimax_red} and its coarser analouge~\ref{prop:not_minimax_red} are given in Section~\ref{sec:redux_depth}. The main idea is to relate quantities of interest to fundamental quantities in the study of deformed Wigner matrix, namely the Stieltjes transform and its derivatives. Leveraging the non-asymptotic convergence of the Stieltjes transform, we can establish non-asymptotic convegence of its derivatives via Lemma~\ref{lem:cvx_approximation} in the appendix, a quantitative analogue of a classical bound regarding the convergence of the derivatives of limits of convex functions. 

	Compared to~\eqref{eq:taub_loose_bound}, Proposition~\ref{prop:minimax_red}  \emph{increases} the error tolerance by a factor of $\frac{1}{\ovlapd(\taub)}$, up to multiplicative constants. In particular, if we can show $\ovlapd(\taub) \lesssim \lambda - 1$, then the $\RQA$ need only output a solution $\xhat$ satisfying $\frac{\|\xhat - \matxsttil\|_2^2}{\|\matxsttil\|_2^2 } \lesssim 1$. For $d$ sufficiently large, we can prove precisely this bound.
	\begin{prop}\label{prop:ovlap_prop} Suppose that $\taub  = (\lambda - 1)^2$. Then, there exists a $d_1 = d_1(\lambda)$ such for all $d \ge d_1$, $\ovlapd(\taub) \le 5(\lambda - 1)$. Moreover, under Conjecture~\ref{conj:nonasymp}, $d_1 \le \BigOh{\poly(\frac{1}{\lambda - 1})}$. 
	\end{prop}
	The above result leverages a recent result regarding the asymptotic error of plant estimation in a deformed Wigner model~\citep{lelarge2016fundamental}. The proof involves engaging with rather specialized material, and is deferred to Section~\ref{sec:ovlap_sec}. Specifically, the first statement is a consequence of Corollary~\ref{cor:taublambdaovlp_asym}, and the second statement follows from Corollary~\ref{cor:taublambdaovlp_nonasym}.

	Combining the bound on $\ovlapd(\taub)$ from Propostion~\ref{prop:ovlap_prop}, and the reduction from estimation in Proposition~\ref{prop:minimax_red}, we obtain
	\begin{align}\label{eq:taub_bound_tight}
	\forall d \ge d_1(\lambda),~\Pr_{\matA,\bnot,\matu,\Alg}\left[\langle \unit \xhat, \matu \rangle^2 \gtrsim  (\lambda - 1) \right] \ge 
	\frac{1}{4}\Pr_{\matAtil,\matbtil \sim \calD}\Pr_{\Alg}\left[ \frac{\|\xhat - \matxsttil\|_2^2}{\|\matxsttil\|_2^2 } \lesssim 1\right] -\dellam.
	\end{align}
	The last ingredient we need in our proof is to upper bound $\Pr_{\matA,\bnot,\matu,\Alg}\left[\langle \unit \xhat, \matu \rangle^2 \gtrsim  (\lambda - 1) \right]$
	\begin{thm}\label{thm:est_u_lb} Let $\lambda \in (1,\frac{3}{2}]$ and $\taub = (\lambda - 1)^2$, and let $\matu$, $\matM$ and $\matb$ be as in Section~\ref{sec:reduction}. Then for any $\RQA$ $\Alg$ interacting with the instances $(\matA,\matb)$, and any $\itT \le \frac{1}{5(\lambda - 1)}$,
	\begin{align*}
	\Pr_{\matu,\matA,\matb,\Alg}\left[ \langle \unit{\xhat}, \matu \rangle^2 >   2 e \cdot \taub \itT  \right] \nonumber  \le \dellam
	\end{align*}
	where the probability is taken over the randomness of the algorithm, and over $\matu,\matb,\matW$. 
	\end{thm} 
	We prove Theorem~\ref{thm:est_u_lb} by modifying the arguments from~\cite{simchowitz2018tight}; the proof is outlined in Section~\ref{sec:plant_estimation_lower}. The key intuition is to slightly modify $\Alg$'s queries so the innner product $\langle \unit{\xhat}, \matu \rangle^2$ by the norm of the projection of $\matu$ onto $\itT+1$-queries, and show that this projection grows at a rate thats bounded by a geometric series on the order of $\sum_{j=1}^{\itT} \lambda^{\BigOh{t}}$, which is $\lesssim \itT$ for $\itT \lesssim 1/(\lambda - 1)$. With the above results in place, we are now ready to prove our main theorems:
	\begin{proof}[Proof of Theorems~\ref{thm:main_theorem} and~\ref{thm:polynomial_theorem}]
	To prove Theorem~\ref{thm:main_theorem}, let $c_1$ denote the hidden universal constant on the left hand side of equation~\eqref{eq:taub_bound_tight}, and $c_2$ the universal constant on the right hand side. Then, for $\lambda \in (1,3/2)$, $\itT < \frac{c_1}{2e}\cdot \frac{1}{\lambda-1}$, and $d \ge d_1(\lambda)$, the distribution $\calD$ from the sharpened reduction in Propostion~\ref{prop:minimax_red} satisfies
	\begin{align*}
	\dellam &\overset{\text{Theorem}~\ref{thm:est_u_lb}}{\ge}  \Pr_{\matu,\matA,\matb,\Alg}\left[ \langle \unit{\xhat}, \matu \rangle^2 \ge c_1 (\lambda - 1) \right]\\
	&\overset{\text{Eq.}~\eqref{eq:taub_bound_tight}}{\ge}  \frac{1}{4}\Pr_{\matAtil,\matbtil \sim \calD}\Pr_{\Alg}\left[ \frac{\|\xhat - \matxsttil\|_2^2}{\|\matxsttil\|_2^2 } \le c_2\right] -\dellam,
	\end{align*}
	Rearranging, combining $\dellam$ terms, and absorbing constants, we have that 
	\begin{align*}
	\Pr_{\matAtil,\matbtil \sim \calD}\Pr_{\Alg}\left[ \frac{\|\xhat - \matxsttil\|_2^2}{\|\matxsttil\|_2^2 } \le c_2\right] \le \dellam  = c_3e^{-c_4d^{-c_5}(\lambda - 1)^{c_6}}\text{ for } \itT < \frac{c_1}{2e}\cdot \frac{1}{\lambda-1}, d \ge d_1(\lambda).
	\end{align*}
	Now recall that with probability one over $\calD$, $\cond(\matAtil) \le \conddet(\lambda) = 2(\lambda^2 + 1)/(\lambda - 1)^2$. We see that $\lambda \mapsto 2\conddet(\lambda)$ is a decreasing bijection from $(1,\tfrac{3}{2}]$ to $[\conddet(3/2),\infty)$, we may reparameterize both $d_1(\cdot)$ and the above result in terms of $\kappa := 2\conddet(\lambda)$. Recognizing that $(\lambda - 1)^{-2} \lesssim  \kappa \lesssim (\lambda - 1)^{-2}$, we see that for possibly modified constants $c_1,\dots,c_6$, it holds that for all $\kappa \ge \conddet(3/2) = 52$, we have
	\begin{align*}
	\Pr_{\matAtil,\matbtil \sim \calD}\Pr_{\Alg}\left[ \frac{\|\xhat - \matxsttil\|_2^2}{\|\matxsttil\|_2^2 } \le c_2\right] \le  c_3e^{-c_4d^{-c_5}\kappa^{-c_6}}\text{ for } \itT < c_1 \sqrt{\kappa}, d \ge d_1(\kappa),
	\end{align*}
	where with probability $1$, $\cond(\matAtil) \le \kappa$.
	We remark if if $d_1(\cdot)$ is polynomial in $1/(\lambda - 1)$, as in Conjecture~\ref{conj:nonasymp}, then $d_1 = \poly(\kappa)$ when parameterized in terms of $\kappa$. We also for some $d_0(\condnum) = \poly(\condnum)$, we can bound $ c_3e^{-c_4d^{-c_5}\kappa^{-c_6}} \le e^{-d^{-c_3'}}$ for a new universal constant $c_3'$. Lastly, we find $f_{\matAtil,\matbtil}(\xhat) - f_{\matAtil,\matbtil}(\matxsttil) \ge \lambda_{\min}(\matAtil)\|\xhat - \matxst\|^2 =\frac{\lambda_{\max}(\matAtil)}{\cond(\matAtil)}\|\xhat - \matxst\|^2 \ge \frac{\lambda_{\max}(\matAtil)}{\kappa}\|\xhat - \matxst\|^2$, and thus the event $\{\frac{\|\xhat - \matxsttil\|_2^2}{\|\matxsttil\|_2^2 } \le c_2\}$ entails $\{\|\xhat - \matxsttil\|_2^2 \le c_2 \frac{\lambda_1(\matA)\|\matxsttil\|_2^2}{\kappa}\}$. This concludes the proof of Theorem~\ref{thm:main_theorem}. The proof of Theorem~\ref{thm:polynomial_theorem} follows similarly by arguing from Equation~\eqref{eq:taub_loose_bound} instead of from~\eqref{eq:taub_bound_tight}; in this case, we no longer need the requirement $d \ge d_1(\lambda)$, and we work with the original distribution over $(\matA,\matb)$ instead of the conditional distribution $\calD$.
	\end{proof}

	\iftoggle{stocsub}
	{
	\newpage
	}

\section{Reduction from Estimation to Minimization: Proof of Propositions~\ref{prop:minimax_red} and~\ref{prop:not_minimax_red}\label{sec:redux_depth}}
	In this section, we shall focus on establishing Proposition \ref{prop:minimax_red}; the proof of Proposition~\ref{prop:not_minimax_red} uses strictly a simplified version of the same argument, and we defer its proof to the end of the section. In proving Proposition \ref{prop:minimax_red}, our goal will be to define an event $\calE_*$ such that the desired distribution of $(\matAtil,\matbtil)$ is just the conditional distribution $(\matA, \matb) | \calE_*$. We shall construct $\calE_*$ as the intersection of two events $\EA$ and $\Eovlap$, which ensure respectively that
	\begin{itemize}
		\item $\matA$ is well conditioned; specifically, $\cond(\matA) \gtrsim \conddet(\lambda)$.
		\item Any approximate minimizer $\xhat$ of $f_{\matA,\bnot}(\cdot)$ is well aligned with $\matu$ with constant probability.
	\end{itemize}
	Let's begin with $\EA$, which ensures the conditioning of $\matA$. In what follows, we let $\mult > 1$ denote a parameter representing a multiplicative error in our deviation bounds; we shall choose $\mult = \sqrt{2}$ without affecting the scaling of the results, but taking $\mult \to 1$ will recover known asymptotic scalings in many (but not all) of our bounds. 
	\begin{prop}\label{prop:EA} Let $\lambda \in (1,2]$. Then, for any fixed $\mult > 1$, the event
	\begin{eqnarray} 
	\EA(\mult) := \left\{ \mult^{-1}\cdot\frac{(\lambda - 1)^2}{\lambda} \le \lambda_{d}(\matA) \le  \lambda_{1}(\matA) \le \mult \cdot 2(\lambda + \lambda^{-1}) \right\}
	\end{eqnarray}
	occurs with probability at least $1 - \dellam$. 
	\end{prop}
	Proposition~\ref{prop:EA} is derived from a finite sample analogue of known asymptotic properties of the spectrum of deformed Wigner matrices; it's proof is explained further in Section~\ref{sec:redux_depth}. Note that on $\EA(\mult)$, we have $\cond(\matA) \le \mult^2 \conddet(\lambda)$. We shall consider the event $\EA = \EA(\sqrt{2})$, so $\mult^2 = 2$. 

	We now turn to the second bullet. By conditioning on the random seeds $\boldsymbol{\xi}$, we may assume without loss of generality that $\Alg$ is deterministic.  The main idea here is to express the overlap between $\xhat$ and $\matu$ in terms of the overlap between $\xst$ and $\matu$, and the overlap between the error $\xst - \xhat$ and $\matu$. Recall the notation $\unit{x} := x/\|x\|$ if $x \ne 0$, and $0$ otherwise. Let $\errhat := \unit{\unit{\xhat} - \unit{\matxst}}$ denote the unit vector pointing in the direction of $\unit\xhat - \unit\matxst$. We can lower bound the overlap between $\xhat$ and $\matu$ via
\begin{eqnarray}
\left|\langle \unit{\xhat}, \matu \rangle\right| &\ge&  \left|\langle \unit{\xst}, \matu \rangle\right| - \left\|\unit{\xhat} - \unit{\xst}\right\|_2 |\langle \errhat, \matu \rangle| \nonumber\\
&\overset{(i)}{\ge}&  \left|\langle \unit\matxst, \matu \rangle\right| - 2\frac{\|\matxst - \xhat\|_2}{\|\matxst\|_2}|\langle \errhat, \matu \rangle| \label{eq:distance_ineq}~,
\end{eqnarray}
where we verify $(i)$ in Section~\ref{sec:proof:distance_ineq}. We remark that both inequalities holds even if $\errhat = 0$. As a consequence, we have that for an $L > 0$ to be chosen at the end of the proof,
\begin{align*}
&\Pr\left[\left|\langle \unit{\xhat}, \matu \rangle\right| \ge  \frac{L}{2} \right] \\
&\overset{\eqref{eq:distance_ineq}}{\ge}\Pr\left[ 2\cdot \frac{\|\matxst - \xhat\|_2}{\|\matxst\|_2} \cdot \left|\langle \errhat, \matu \rangle\right| \le \frac{L}{2} \text{ and } \left|\langle \unit\matxst, \matu \rangle\right| \ge L\right] \\
&\ge \Pr\left[2\cdot \frac{\|\matxst - \xhat\|_2}{\|\matxst\|_2} \cdot \left|\langle \errhat, \matu \rangle\right| \le \frac{L}{2}\right] - \Pr\left[\left|\langle \unit\matxst, \matu \rangle\right| < L\right].
\end{align*}
Next, note that $\errhat$ is an $\matM,\matb$ measurable unit vector (or the zero vector), we know that $\Exp_{\matM,\matb}\Exp_{\matu}[\langle \errhat, \matu \rangle^2]] \le \ovlapd$, where we suppress dependence on $\taub$ to streamline notation. Hence, it make sense to introduce the low-overlap event $\{\langle \errhat, \matu \rangle^2 \le t\cdot \ovlapd \}$, where $t > 1$ is a parameter that introduces some slack. With some small rearrangements, we may therefore lower bound
\begin{multline}
\Pr\left[\left|\langle \unit{\xhat}, \matu \rangle\right| \ge  \frac{L}{2} \right] \ge \\
\Pr\left[\frac{\|\matxst - \xhat\|_2}{\|\matxst\|_2} \le \frac{L}{4\sqrt{t \cdot \ovlapd}} \text{ and } \langle \errhat, \matu \rangle^2 \le t\cdot \ovlapd  \right] - \Pr\left[\langle \unit\matxst, \matu \rangle^2 < L^2\right].
\end{multline}
At this stage, we have to show to lower bound the probability that $\frac{\|\matxst - \xhat\|_2}{\|\matxst\|_2}$ is small, when restricted to the event that the overlap between $\errhat$ and $\matu$ is also small, and we have to upper bound $\Pr\left[\langle \unit\matxst, \matu \rangle^2 < L^2\right]$. Let's start with the first term. The challenge here is that $\xhat$ and $\errhat$ are very correlated, but we can decouple them with the following strategy. We introduce an event which depends only on $\matM$ and $\bnot$, but not on the algorithm, under which the best possible overlap is at most $t \cdot \ovlapd$ with constant probability. Specifically,  
\begin{align*}
		\Eovlap(t) &:= \left\{\sup_{\uhat = \uhat(\matM,\bnot)}\Pr_{\matu}\left[\langle \uhat, \matu \rangle^2 > t \cdot \ovlapd \big{|}\matM,\bnot\right] \le t^{-1/2}\right\}.
\end{align*}
We shall choose $t = 4$ at the end of the proof, but for now it will be simpler to leave $t$ as a numerical parameter. We can use $\Eovlap$ to decouple our two events by conditioning on $(\matM,\matb)$ and making the following observation: $\Eovlap(t)$ is $(\matM,\matb)$-measurable, and recalling our assumption that $\Alg$ is deterministic, and noting that $\matA$ and $\matM$ are in one-to-one correspondence, we see that the events $\{\frac{\|\matxst - \xhat\|_2}{\|\matxst\|_2} \le \frac{L}{4\sqrt{t}}\}$ is $(\matM,\matb)$-measurable as well. Hence, 
\begin{align*}
&\Pr\left[\frac{\|\matxst - \xhat\|_2}{\|\matxst\|_2} \le \frac{L}{4\sqrt{t \cdot \ovlapd}} \text{ and } \langle \errhat, \matu \rangle^2 \le t\cdot \ovlapd  \right]  \\
&\ge \Pr\left[\frac{\|\matxst - \xhat\|_2}{\|\matxst\|_2} \le \frac{L}{4\sqrt{t \cdot \ovlapd}} \cap \langle \errhat, \matu \rangle^2 \le t\cdot \ovlapd  \cap \Eovlap(t) \right] \\
&\ge \Exp_{\matM,\matb}[\Pr_{\matu}\left[\frac{\|\matxst - \xhat\|_2}{\|\matxst\|_2} \le \frac{L}{4\sqrt{t \cdot \ovlapd}} \cap \langle \errhat, \matu \rangle^2 \le t\cdot \ovlapd  \cap \Eovlap(t) | \matM,\matb \right] ]\\
&\ge \Exp_{\matM,\matb}\left[\I\left(\frac{\|\matxst - \xhat\|_2}{\|\matxst\|_2} \le \frac{L}{4\sqrt{t \cdot \ovlapd}}  \cap \Eovlap(t) ]\right)\cdot\Pr_{\matu}\left[ \langle \errhat, \matu \rangle^2 \le t\cdot \ovlapd  | \matM,\matb \right] \right]\\
\end{align*}
The key observation here is that by definition of $\Eovlap(t)$, we see that whenever the above indicator function in nonzero, we must have $\Pr_{\matu}\left[ \langle \errhat, \matu \rangle^2 \le t\cdot \ovlapd  | \matM,\matb \right] \ge 1 - t^{-1/2}$. Hence, the above display as at least
\begin{multline*}
\Exp_{\matM,\matb}\left[\I(\frac{\|\matxst - \xhat\|_2}{\|\matxst\|_2} \le \frac{L}{4\sqrt{t \cdot \ovlapd}}  \cap \Eovlap(t) ])(1-\frac{1}{\sqrt{t}})\right] \\
= (1-\frac{1}{\sqrt{t}}) \Pr_{\matM,\matb}\left[\frac{\|\matxst - \xhat\|_2}{\|\matxst\|_2} \le \frac{L}{4\sqrt{t \cdot \ovlapd}}  \cap \Eovlap(t) \right].
\end{multline*}
We now define the event $\calE_*$ (the one for which $(\matAtil,\matbtil) \overset{d}{=} \matA,\matb| \calE_*$) to be $\EA \cap \Eovlap(t)$; we shall check that $\calE_*$ occurs with nonzero probability at the end of the proof. Since $\calE_* \subseteq \Eovlap$, we may lower bound $\Pr\left[\left|\langle \unit{\xhat}, \matu \rangle\right| \ge  \frac{L}{2} \right]$ by
\begin{align*}
&\Pr_{\matM,\matb}\left[\frac{\|\matxst - \xhat\|_2}{\|\matxst\|_2} \le \frac{L}{4\sqrt{t\cdot \ovlapd}}  \cap \Eovlap(t) \right] \\
&\ge \Pr_{\matM,\matb}\left[\frac{\|\matxst - \xhat\|_2}{\|\matxst\|_2} \le \frac{L}{4\sqrt{t\cdot \ovlapd}}  \cap \calE_* \right]\\
&\ge \Pr_{\matM,\matb}\left[\frac{\|\matxst - \xhat\|_2}{\|\matxst\|_2} \le \frac{L}{4\sqrt{t\cdot \ovlapd}} | \calE_* \right]\Pr[\calE_*] \\
&\ge \Pr_{\matM,\matb}\left[\frac{\|\matxst - \xhat\|_2}{\|\matxst\|_2} \le \frac{L}{4\sqrt{t\cdot \ovlapd}} | \calE_* \right](\Pr[\Eovlap(t)] - \Pr[\EA])\\
&\ge \Pr_{\matM,\matb}\left[\frac{\|\matxst - \xhat\|_2}{\|\matxst\|_2} \le \frac{L}{4\sqrt{t\cdot \ovlapd}} | \calE_* \right]\cdot \Pr[\Eovlap(t)]   - \Pr[\EA]\\
&= \Pr_{\matAtil,\matbtil}\left[\frac{\|\matxsttil - \xhat\|_2}{\|\matxsttil\|_2} \le \frac{L}{4\sqrt{t\cdot \ovlapd}} \right]\cdot \Pr[\Eovlap(t)]   - \Pr[\EA] ,
\end{align*}
where in the last line we note that $\matM$ and $\matA$ are in one-to-one correspondence, so the law induced by $(\matM,\matb)$ is the same as the one induced by $(\matA,\matb)$, which conditioned on $\calE_*$, is precisely that of $(\matAtil,\matbtil)$. Collecting what we have thus far, $\Pr\left[\left|\langle \unit{\xhat}, \matu \rangle\right| \ge  \frac{L}{2} \right]$ is at most
\begin{align*}
(1-\frac{1}{\sqrt{t}}) \Pr_{\matAtil,\matbtil}\left[\frac{\|\matxsttil - \xhat\|_2}{\|\matxsttil\|_2} \le \frac{L}{4\sqrt{t\cdot \ovlapd}} \right]\cdot \Pr[\Eovlap(t)]   - \Pr[\EA] - \Pr\left[\langle \unit\matxst, \matu \rangle^2 < L^2\right].
\end{align*}
To wrap up, it suffices to lower bound $\Pr[\Eovlap(t)]$, upper bound $\Pr[\EA]$,  and choose $L$ so as to upper bound $\Pr\left[\langle \unit\matxst, \matu \rangle^2 < L^2\right]$. We can lower bound $\Pr[\Eovlap(t)] \ge 1 - t^{-1/2} $ can be lower bound by two applications of Markov's inequality,
\begin{eqnarray*}
&&\Pr_{\matM,\matb}\left\{  \max_{\uhat = \uhat(\matM,\bnot)}\Pr_{\matu}[\langle \uhat, \matu \rangle^2 \ge t \ovlapd(\tau) \big{|}\matM,\bnot] \ge t^{-1/2}\right\} \\
&\le& \Pr_{\matM,\matb} \left\{\frac{1}{t \ovlapd(\tau)}\max_{\uhat = \uhat(\matM,\bnot)}\Exp_{\matu}[\langle \uhat, \matu \rangle^2  \big{|}\matM,\bnot] \ge t^{-1/2}  \right\} \\
&\le& \frac{1}{t \ovlapd(\tau) \cdot t^{-1/2}}\Exp_{\matM,\matb} \max_{\uhat = \uhat(\matM,\bnot)}\Exp_{\matu}[\langle \uhat, \matu \rangle^2  \big{|}\matM,\bnot]  = \frac{1}{t^{1/2}}~.
\end{eqnarray*}
Next, $ \Pr[\EA]  \le  \dellam$, as given by Proposition~\ref{prop:EA} (recall $\dellam$ is an exponentially small error term). Lastly,the following proposition shows how to choose the term $L$	\begin{prop}\label{prop:Ex} For any $\lambda \in (1,2]$, any $d^{-.9} \le \taub \le (\lambda - 1)^2$,and any fixed $\mult > 1$, the event
	\begin{eqnarray}
	\Ex(\mult) := \left\{\left\langle \frac{\xst}{\|\xst\|_2} ,\matu \right\rangle^2 \ge \frac{1}{\mult^2} \cdot \frac{\taub}{3(\lambda - 1)}\right\}
	\end{eqnarray}
	occurs with probability at least $1 - \delta_{\lambda}(d)$.
	\end{prop}

Proposition~\ref{prop:Ex} is quite technical, and we give a sketch of its proof in Section~\ref{sec:ovlap_explanation}. However, we have now collected all the proof ingredients we shall need. Specifically, choosing $L^2 = \frac{1}{2} \cdot \frac{\taub}{3(\lambda - 1)}$ (i.e. $\mult = \sqrt{2}$), we have that 
\begin{align*}
\Pr\left[\left|\langle \unit{\xhat}, \matu \rangle\right| \ge \sqrt{ \frac{\taub}{12(\lambda - 1)}} \right] \ge (1-\tfrac{1}{\sqrt{t}})^2 \Pr_{\matAtil,\matbtil}\left[\frac{\|\matxsttil - \xhat\|_2}{\|\matxsttil\|_2} \le \frac{1}{4}\sqrt{\frac{\taub}{6(\lambda - 1) \cdot t\cdot \ovlapd}} \right] - 2 \dellam.
\end{align*}
Noteing that $2\dellam \equiv \dellam$, absorbing universal constants, and letting $t = 4$, we find 
\begin{align*}
\Pr\left[\langle \unit{\xhat}, \matu \rangle^2 \ge c_1 \cdot \frac{\taub}{(\lambda - 1)} \right] \ge \frac{1}{4} \Pr_{\matAtil,\matbtil}\left[\frac{\|\matxsttil - \xhat\|_2^2}{\|\matxsttil\|_2^2} \le c_2\cdot \frac{\taub}{(\lambda - 1) \cdot \ovlapd} \right] - \dellam.
\end{align*}
Lastly, substituing in $\taub = (\lambda - 1)^2$ and $\ovlapd = K(\lambda - 1)$ concludes the proof of the proposition. We finally check that $\Pr[\calE_*]$ is bounded away from zero. Indeed, $\Pr[\calE_*] \ge \Pr[\Eovlap(t)] - \Pr[\EA] \ge (1 - t^{-1/2}) - \dellam \ge \frac{1}{2} - \dellam$, which is bounded away from zero provived that $\dellam$ is sufficiently small.

\subsection{Proof Sketch of Proposition~\ref{prop:EA}\label{sec:EA_sec}}

	To understand the proof of Proposition~\ref{prop:EA}, we remark that the spectrum of $\matM$ is well studied in random matrix theory~\cite{peche2006largest,feral2007largest,anderson2010introduction,benaych2011eigenvalues}. In particular, as $d \to \infty$, we have
	\begin{eqnarray*}
	\lambda_{1}(\matM) \pto  \lambda + \lambda^{-1} \quad \text{and} \quad \lambda_d(\matM) \pto -2~.
	\end{eqnarray*}
	Setting $\matA = (2(\lambda + \lambda^{-1}) - 2) I - \matM$ we have that 
	\begin{eqnarray*}
	\lambda_{1}(\matA) \pto 2(\lambda + \lambda^{-1})~\text{and }~ \lambda_{d}(\matA) \pto \lambda + \lambda^{-1} - 2 = \lambda^{-1}(\lambda - 1)^2~.
	\end{eqnarray*}
	To prove Propoposition~\ref{prop:EA}, we invoke non-asymptotic analogoues of the above asymptotic convergence results, derived in~\cite{simchowitz2018tight}. The details are carried out in Appendix~\ref{proof:prop:EA}.

\subsection{Proof Sketch Proposition~\ref{prop:Ex}: Bounding the overlap of $\xst$ and $\matu$\label{sec:ovlap_explanation}}
	The proof of proposition~\ref{prop:Ex} is quite technical, but we outline the main ideas here. Throughout, it will be convenient for us to render $\bnot = \sqrt{\taub} \matu + \matz$, where $\matz \sim \mathcal{N}(0,I/d)$ is independent of $\matW,\matu$. We shall introduce a more granular version of $\dellam$, $\deltaexp$, which is a term bounded by at most $c_1\exp(-c_2 d^{-c_3}(\lambda - 1)^{c_4}(\mult - 1)^{c_5})$ for universal constants $c_1,\dots,c_5$.
	We will also introduce the notation $\smallo$ to denote a term which satisfies $\Pr[\smallo \le \mult - 1] \le \deltaexp$, and let $\gamma := 2(\lambda + \lambda^{-1}) - 2$ denote the factor such that $\matA = \gamma I - \matM$. In the appendix, we show that
	\begin{eqnarray*} 
	\left\langle \frac{\xst}{\|\xst\|_2} ,\matu \right\rangle^2  = \frac{\taub (\matu^\top \matA^{-1} \matu)^2 - \smallo}{\taub \matu^\top\matA^{-2}\matu^\top + \matz^{\top}\matA^{-2}\matz + \smallo}.
	\end{eqnarray*}
	We then unpack $\matA^{-1}$ and $\matA^{-2}$ using the Sherman-Morrison-identity, and relate the above expression to terms depending on $\matz^{\top}(\gamma I - \matW)^{-1}\matz^{\top}$, $\matz^{\top}(\gamma I - \matW)^{-2}\matz^{\top}$, and analogous terms with $\matz$ replaced by $\matu$. Since $\matW$ is independent of $\matz$ and $\matu$, Hanson-Wright implies
	\begin{eqnarray*}
	\matz^{\top}(\gamma I - \matW)^{-1}\matz &=& \tr(\gamma I - \matW)^{-1} + \smallo \text{ and } \nonumber\\
	\matz^{\top}(\gamma I - \matW)^{-2}\matz &=& \tr(\gamma I - \matW)^{-2} + \smallo~,
	\end{eqnarray*}
	and similarly for terms involving $\matu$. Asymptotic expresions for $\tr(\gamma I - \matW)^{-1}$ and $\tr(\gamma I - \matW)^{-2}$ are well-studied in the literature \citep{anderson2010introduction,peche2006largest,feral2007largest,benaych2011eigenvalues}. In Appendix~\ref{proof:prop:trace_comps}, we prove quantitative convergence result:
	\begin{prop}\label{prop:trace_comps} The following bounds hold:
	\begin{align*}
	\tr(\gamma I - \matW)^{-1} &= \stielt(\gamma) + \smallo, \text{ where } \stielt(\gamma) := \frac{\gamma - \sqrt{\gamma^2 - 4}}{2},\\
	\tr(\gamma I - \matW)^{-2} &= \quielt(\gamma) + \smallo, \text{ where } \quielt(\gamma) := \frac{-d}{d\gamma} \stielt(\gamma).
	\end{align*}
	\end{prop} 
	The function $\stielt(\gamma)$ is known as the \emph{Stieljes transform} of the Wigner Semicircle law~\citep{anderson2010introduction}, and is a central object in the study of random matrices. The estimate $\tr(\gamma I - \matW)^{-1} = \stielt(\gamma) + \smallo$ is a direct consequence of a non-asymptotic convergence result from~\cite{simchowitz2018tight}; the estimate for $\tr(\gamma I - \matW)^{-2}$ follows from a quantitative version (Lemma~\ref{lem:cvx_approximation}) of a classical lemma regarding the convergence of derivatives of concave functions. Putting things together, we show in Appendix~\ref{proof:Ex} that 
	\begin{eqnarray} 
	\left\langle \frac{\xst}{\|\xst\|_2} ,\matu \right\rangle^2  = \taub \cdot \frac{1+\smallo}{\stielt(\gamma)^{-2} \cdot \quielt(\gamma) (\taub + (1 - \lambda \stielt(\gamma))^2) + \smallo }~,
	\end{eqnarray}
	Lastly, establish the deterministic bounds $\stielt(\gamma)^{-2} \quielt(\gamma) \le 3/2(\lambda - 1)$ and $1 - \lambda \stielt(\gamma)  \le \lambda - 1$ (Lemma~\ref{lem:stielt_comp}) which implies Proposition~\ref{prop:Ex}, after some elementary computations completed in Appendix~\ref{proof:Ex}.

\subsection{Proof of Proposition~\ref{prop:not_minimax_red}}
The argument is very similar to the proof of Proposition~\ref{prop:minimax_red}, except we need far less care in handing the overlap term. Recalling our steps from the proof of Proposition~\ref{prop:minimax_red}, we may bound
\begin{align*}
&\Pr\left[\left|\langle \unit{\xhat}, \matu \rangle\right| \ge  \frac{L}{2} \right] \ge \Pr\left[2\cdot \frac{\|\matxst - \xhat\|_2}{\|\matxst\|_2} \cdot \left|\langle \errhat, \matu \rangle\right| \le \frac{L}{2}\right] - \Pr\left[\left|\langle \unit\matxst, \matu \rangle\right| < L\right].
\end{align*}
However, for a constant $\mult > 1$, we now crudely bound
\begin{align*}
\Pr\left[2\cdot \frac{\|\matxst - \xhat\|_2}{\|\matxst\|_2} \cdot \left|\langle \errhat, \matu \rangle\right| \ge \frac{L}{2}\right] &\ge \Pr\left[2\cdot \frac{\|\matxst - \xhat\|_2}{\|\matxst\|_2} \cdot \|\matu\|_2 \le \frac{L}{2}\right]\\
&\ge \Pr\left[ \frac{\|\matxst - \xhat\|_2}{\|\matxst\|_2} \le \frac{L}{4 \mult}\right] - \Pr[\|\matu\|_2 \ge \mult].
\end{align*} 
Now for any fixed constant $\mult > $
$\Pr[\|\matu\|_2^2 \ge \mult] \le \dellam$ by standard $\chi^2$-concentration (see, e.g.~\citet[Lemma 1]{laurent2000adaptive}). Hence, selecting $L^2 = \frac{1}{\mult} \cdot \frac{\taub}{3(\lambda - 1)}$, invoking Propostion~\ref{prop:Ex},  choosing an arbitrary constant $\mult$ bounded away from $1$, and absorbing constants, we conclude that
\begin{align*}
\Pr\left[\langle \unit{\xhat}, \matu \rangle^2 \frac{\taub}{\lambda - 1} \gtrsim  \right]  \ge \Pr\left[ \frac{\|\matxst - \xhat\|_2^2}{\|\matxst\|_2^2} \lesssim \frac{\taub}{\lambda - 1}\right] - 2\dellam.
\end{align*}

\subsection{Proof of~\eqref{eq:distance_ineq}\label{sec:proof:distance_ineq}}
Note that with probability $1$, $\matxst \ne 0$. Moreover, if $\xhat = 0$, then~\eqref{eq:distance_ineq} follows immediately from the triangle inequality. Otherwise,
\begin{eqnarray*}
\left\|\frac{\matxst}{\|\matxst\|_2} - \frac{\xhat}{\|\xhat\|_2}\right\|_2 &\le& \frac{\|\matxst - \xhat\|_2}{\|\matxst\|_2} + \left|\frac{\|\xhat\|_2}{\|\matxst\|_2} - \frac{\|\xhat\|_2}{\|\xhat\|_2}\right|\\
 &=& \frac{\|\matxst - \xhat\|_2}{\|\matxst\|_2} + \left|\frac{\|\xhat\|\|\xhat\| - \|\matxst\|\|\xhat\|}{\|\matxst\|_2 \|\xhat\|} \right| \le 2\frac{\|\matxst - \xhat\|_2}{\|\matxst\|_2}~.
\end{eqnarray*}

\section{Lower Bound for Plant Estimation\label{sec:plant_estimation_lower}}
In this section, we prove Theorem~\ref{thm:est_u_lb}, which provides a lower bound for the alignment between $\unit{\xhat}$ and $\matu$, given given a sequence of $\itT$ queries, as well as the initial information $\bnot$. The idea here is to consider $\uhat := \unit{\xhat}$ as an adaptive estimator of $\matu$, and leverage the machinery developed in~\cite{simchowitz2018tight} to lower bound this adaptive estimation problem. We shall need to make modify the bounds in the previous work for our setting as follows: (1) we consider the case where $\matu \sim \calN(0,\frac{1}{d}I)$, whereas the past work considers $\matu$ drawn uniformly from $\sphered$, (2) we consider initial side information $\bnot$, whereas~\cite{simchowitz2018tight} does not, and (3) we consider a regime where we take relatively few iterations, in a sense described below. 

To adress $(1)$, we shall need to restrict to the event restrict to the event $\eventbound(\epsilon) := \{\|\matu\|_2^2 \le 1+\epsilon\}$, where we will ultimately choose $\epsilon = \lambda - 1$ at the end of the proof. This ensures that the plant is not too large, and therefore does not provide the learner too much information in any given query. For any bound $\tau > 0$, we have
	\begin{align*}
	\Pr[ \langle \matu, \uhat\rangle^2 \ge \tau] &\le  \Pr\left[\left\{ \langle \matu, \uhat\rangle^2 \ge \tau\right\} \cap \eventbound(\epsilon)\right] + \Pr [\eventbound(\epsilon)^c]\\
	&\le  \Pr\left[\left\{ \langle \matu, \uhat\rangle^2 \ge \tau\right\} \cap \eventbound(\epsilon)\right] + \dellam\numberthis\label{eq:eventbound_bound},
	\end{align*} 
	where we note that, for $\epsilon = \lambda - 1$, $\Pr [\eventbound(\epsilon)^c] = \Pr[\|\matu\|_2^2 \ge \lambda] \le \dellam$ by standard $\chi^2$ concentration (see, e.g.~\citet[Lemma 1]{laurent2000adaptive}).

Restricting to the event $\eventbound(\epsilon)$ for $\epsilon = \lambda - 1$, our proof boils down to establishing the following bound, which will require the majority of our technical effort: 
\begin{align}\label{eq:inner_prod_upper_bound_wts}
\Pr\left[ \left\{\langle \matu, \uhat\rangle^2 > 2\taub  \sum_{j=1}^{\itT} \lambda^{5j}\right\} \cap \eventbound(\epsilon) \right] ~\le~ (\itT+1)e^{ - d\lambda^2\taub (\lambda - 1)}.
\end{align}
The above inequality states that, with high probability angle between $\matu$ and $\uhat$, $\langle \matu, \uhat\rangle^2 $, grows as fast as a geometric sequence of length $\itT$, with base $\lambda^{\BigOh{1}}$. For the large $\itT \gg \frac{1}{\lambda - 1}$ considered in~\cite{simchowitz2018tight}, this quantity behaves roughly as $\lambda^{\BigOh{T}}$. 

To prove Theorem~\ref{thm:est_u_lb}, we shall instead only consider $\itT \le \frac{1}{5(\lambda-1)}$ iterations. For such $\itT$, $2\taub  \sum_{j=1}^{\itT} \lambda^{5j} \le 2\taub \itT(\lambda^{5\itT}) \le 2\taub\itT(1 + (\lambda-1))^{(\lambda - 1)} \le 2e \taub \itT$, an thus~\eqref{eq:inner_prod_upper_bound_wts} implies
\begin{align*}
\Pr\left[ \left\{\langle \matu, \uhat\rangle^2 > 2\taub e\itT\right\} \cap \eventbound(\epsilon) \right] ~\le~ (1 + \frac{1}{5(\lambda - 1)})e^{ - d\lambda^2\taub (\lambda - 1)} \le \dellam,
\end{align*}
where the last inequality follows from elementary algebra and  $d\taub \ge d^{.1}$, by assumption. Combining with~\eqref{eq:eventbound_bound}, we have $\Pr\left[ \left\{\langle \matu, \uhat\rangle^2 > 2\taub e\itT\right\} \right] \le \dellam + \dellam = \dellam$, thereby proving Theorem~\ref{thm:est_u_lb}.


%


\textbf{Proving: \eqref{eq:inner_prod_upper_bound_wts}:} We will begin by retracing the steps from~\cite{simchowitz2018tight}, clarifying where modifications are necessary. We begin with a couple simplifications:
\begin{itemize}
	\item As in~\cite{simchowitz2018tight}, we may assume that the queries $\vone,\dots,\vTplus$ form an orthonormal basis. This is without loss of generality because we may always simulate a query $\vi$ which is not orthonormal to prior queries $\vone,\dots,\viminus$ by querying the projection of $\vi$ onto the orthogonal complement of $\vone,\dots,\viminus$, and normalizing. Note that if $\vi \in \mathrm{span}(\vone,\dots,\viminus)$, then this query can be ignored.
	\item  Let $\itV_k \in \R^{d \times k}$ denote the matrix whose columns are $\vone,\dots,\vk$. By the above bullet, $\itV_k^\top \itV_k = I_k$. Defining the potential function
	\begin{align}
	\Phi(\itV_k;\matu) := \matu^\top \itV_k \itV_k^\top \matu =\| \Proj_{\vone,\dots,\vk}\matu\|_2^2~,
	\end{align} 
	we may assume without loss of generality that we make $\itT + 1$ queries, and that $\Phi(\itV_{\itT+1},\matu) \ge \langle \uhat, \matu \rangle^2$. The reason is that, given a putative estimate $\uhat$ at time $\itT$, we can always chose our $\itT+1$-st query to ensure that $\uhat \in \vspan(\vone,\dots,\vTplus)$.
	\item Because we are querying from a known distribution, we may assume that $\Alg$ is deterministic.
\end{itemize}
With these simplifications in hand, our strategy is to the argue about the rate at which the potential function $\Phi(\itV_{k},\matu)$ can grow in $k$, with some high probability. We do this by considering a sequence of thresholds $\{\tau_k\}$, and considering the probabiltity that there exists a $k$ for which $\Phi(\itV_{k},\matu)$ exceeds one $\tau_k$. Letting $\tau_0 = 0$, we see that
\begin{align}
\Pr\left[\left\{ \langle \matu, \uhat\rangle^2 \ge \tau_{\itT+1}\right\} \cap \eventbound(\epsilon)\right] 
&\overset{(i)}{\le} \Pr\left[\left\{\Phi(\itV_{\itT+1};\matu) > \tau_{\itT+1}\right\} \cap \eventbound(\epsilon) \right] \nonumber \\
&= \Pr\left[   \left\{\exists k \in [\itT+1] : \Phi(\itV_k;\matu) > \tau_k \right\} \cap \eventbound(\epsilon) \right] \nonumber\\
&\le \sum_{k=0}^{\itT} \Pr\left[\left\{\Phi(\itV_k;\matu) \le \tau_k\right\}  \cap \eventbound(\epsilon) \cap \left\{\Phi(\itV_{k+1};\matu) > \tau_{k+1}\right\} \right], \label{eq:phi_union_bound}
\end{align}
where $(i)$ holds by the simplification we made above that $\langle \matu, \uhat\rangle^2 \le \Phi(\itV_{\itT+1};\matu)$.  The bound on $\Pr[\{\Phi(\itV_{k+1};\matu) > \tau_{k+1}\} \cap \eventbound(\epsilon) \cap \{\Phi(\itV_k;\matu) \le \tau_k\}]$ requires carefully modifying techniques from the proof of Proposition 3.1 in~\cite{simchowitz2018tight} to account for the initial information $\bnot$, and the Gaussian, rather than spherical, distribution of $\matu$. In the appendix, we prove the following proposition:
\begin{prop}\label{prop:recur}  Under the randomness of $\matu$, $\matW$, one has the bound
\begin{multline}\label{eq:recur_bound_rkone}
\Pr[\{ \Phi(\itV_k;\matu) \le \tau_k\} \cap \eventbound(\epsilon) \cap \{ \Phi(\itV_{k+1};\matu)> \tau_{k+1} ] \\
\le \exp\left\{ \frac{\lambda - 1}{2\lambda}  \left(d(1+\epsilon)\boldlam^3 (\tau_k + \taub)  - \left(\sqrt{d\tau_{k+1}}-\sqrt{2k+2}\right)^2\right) \right\}
\end{multline}
\end{prop}
We are now ready to complete the proof of Theorem~\ref{thm:est_u_lb}. Fix $\delta := e^{ - \lambda^2\taub (\lambda - 1)/2}$, and set $\epsilon = \lambda - 1$. We now consider the sequence
\begin{align*}
 \tau_{k+1} :=   \frac{2\lambda^2}{d(\lambda - 1)}\log(1/\delta) + \boldlam^4 (1+\epsilon)(\tau_k + \taub) = \frac{2\lambda^2}{d(\lambda - 1)}\log(1/\delta) + \boldlam^5 (\tau_k + \taub)~.
\end{align*}
In Appendix~\ref{sec:asm_rec}, we verify that, for any choice of $\epsilon \ge 0$, 
\begin{eqnarray}\label{eq:asm_recur}
\left(\sqrt{d\tau_{k+1}}-\sqrt{2k+2}\right)^2 \ge d\tau_{k+1}/\lambda,
\end{eqnarray}
which, by Proposition~\ref{prop:recur}, implies that 
\begin{align*}
 \Pr[\{ \Phi(\itV_k;\matu) \le \tau_k\} \cap \eventbound(\epsilon) \cap \{ \Phi(\itV_{k+1};\matu)> \tau_{k+1} ] \le e^{ \frac{\lambda - 1}{2\lambda}  \left((1+\epsilon) d\boldlam^3 (\tau_k + \taub)  -  d\tau_{k+1}/\lambda\right)}  = \delta.
\end{align*}
Thus, by~\eqref{eq:phi_union_bound}, which bounds on $\Pr[ \langle \matu, \uhat\rangle^2 \ge \tau_{\itT+1}]$ in terms of the probabilities, $ \Pr[\{ \Phi(\itV_k;\matu) \le \tau_k\} \cap \{ \Phi(\itV_{k+1};\matu)> \tau_{k+1} ]$, we conclude
\begin{align*}
\Pr[ \langle \matu, \uhat\rangle^2 \ge \tau_{\itT+1}] \le (\itT + 1)\delta.
\end{align*}
To conclude, we upper bound the recursion for $\tau_{\itT+1}$: 
\begin{align*}
\tau_{\itT+1} &= \sum_{j=1}^{\itT} \lambda^{5(\itT -j)}\left(\boldlam^5 \taub + \frac{2\lambda^2}{d(\lambda - 1)}\log(1/\delta)\right)\\
&\le \left(\boldlam^5 \taub + \frac{2\lambda^2\log(1/\delta)}{d(\lambda - 1)}\right) \sum_{j=1}^{\itT} \lambda^{5(k-j)} = \left(\taub + \frac{2\log(1/\delta)}{d\lambda^2(\lambda - 1)}\right) \sum_{j=1}^{\itT} \lambda^{5j}~.
\end{align*}
Plugging in $\delta :=  e^{ - d\lambda^2\taub (\lambda - 1)}$, we have that 
\begin{align*}
\Pr\left[ \langle \matu, \uhat\rangle^2 > 2\taub  \sum_{j=1}^{\itT} \lambda^{5j} \right] ~\le~ (\itT+1)e^{ - d\lambda^2\taub (\lambda - 1)}, \text{ as needed}.
\end{align*}




\section{Upper Bound on $\ovlapd$~\label{sec:ovlap_sec}}
The goal of this section is to provide on an asymptotic upper bound on the expected overlap betwen the planted signal $\matu$, and any estimator $\uhat \in \sphered$ which has access to $\matM$ and $\bnot \sim \calN(\sqrt{\taub}\matu,\frac{1}{d}I)$. More precisely, we recall the definition
\begin{align*}
\ovlapd(\taub) := \Exp_{\matM,\bnot}\max_{\uhat \in \sphered}  \Exp_{\matu}[\langle \uhat, \matu \rangle^2 \big{|} \matM,\bnot].
\end{align*} 
In light of Proposition~\ref{prop:ovlap_prop}, we would like to show that for $d$ sufficiently large, $\ovlapd(\taub) \le K(\lambda - 1)$ for a universal constant $K$. This is accomplished by the following result, the main theorem of this section:
\begin{thm}[Asymptotic Bound on $\ovlapd(\taub)$]\label{thm:asmp_ovlap_thm} For $\matu,\bnot,\matM$ and $\ovlapd(\taub)$ defined in Section~\ref{sec:reduction}, we have for $\lambda \in (1,2]$ that 
\begin{align*}
\limsup_{d \to \infty} \ovlapd(\taub) &\le 1 - \frac{1}{\lambda^2} + \taub + \frac{\sqrt{\taub}}{\lambda} 
\end{align*} 
\end{thm}
In particular, if $\taub = (\lambda - 1)^2$, then the above reduces to 
\begin{align*}
\limsup_{d \to \infty} \ovlapd(\taub) &\le (\lambda - 1)\left\{\frac{\lambda + 1}{\lambda^2} + (\lambda-1) + \frac{1}{\lambda}\right\} \le \frac{9}{2}(\lambda - 1)
\end{align*}
This implies the following corollary, which proves the first part of Proposition~\ref{prop:ovlap_prop}:
\begin{cor}\label{cor:taublambdaovlp_asym} There exists a $d_0 = d_0(\lambda,\taub)$ such that for all $d \ge d_0$, $\ovlapd(\taub) \le 5(\lambda -1)$.
\end{cor}
Hence, for $d$ sufficiently large, we can take $K$ in Proposition~\ref{prop:minimax_red} to be a universal constant. 
For intuition about Theorem~\ref{thm:asmp_ovlap_thm}, consider the setting where we do not have access to side information $\bnot$, that is, $\taub = 0$.
Perhaps the most natural estimator of $\matu \sim \mathcal{N}(0,I/d)$ is the top eigenvector of $v_1(\matM)$, and it is known (see, e.g.~\cite{peche2006largest}) that, for any $\lambda  > 1$,
\begin{eqnarray*}
\lim_{d \to \infty} \langle v_1(\matM),\matu \rangle^2 = 1 - \lambda^{-2} \propto 1 - \lambda,
\end{eqnarray*}
where the proportionality holds when $\lambda$ is bounded above by a universal constant, as in this work. 
Nevertheless, one may still wonder if there exists a more sophisticated (maybe computationally infeasible!) estimator $\uhat$  has a larger expected overlap with $\matu$ than does $v_{1}(\matM)$. 

Beautiful recent results due to~\cite{barbier2016mutual} and~\cite{lelarge2016fundamental} show in fact that this is not the case. These works show an explicit and very general formula for the mutual information between $\matM$ and $\matu$.~\cite{barbier2016mutual} applies when the entries of $\matu$ have a finite (discrete) support, and \cite{lelarge2016fundamental}  when $\matu$ is drawn according to \emph{any} distribution with i.i.d. coordinates whose second moments are bounded. 
Due to a correspondence between mutual information and MMSE in a Gaussian channel~\citep{guo2005mutual}, these works use this formula to derive the following asymptotic expression for the minimum mean square error (MMSE) for estimating{} $\matu \matu^{\top}$ given $\matM := \matW + \lambda\matu\matu^\top$, defined as:
\begin{align}
\MMSE(\matu\matu^\top|\matM) := \Exp_{\matu}\left[ \|\matu \matu^\top -  \Exp[\matu\matu^{\top} | \matM] \|_{\F}^2 \big{|}\matM\right]~ \label{eq:MMSE_u_not_b}
\end{align}  
By relating the optimal overlap to the $\MMSE$,~\cite{lelarge2016fundamental} conclude that, in the special case that $\pnot = \mathcal{N}(0,1)$, $v_{1}(\matM)$ indeed attains the optimal asymptotic overlap of $1 - \lambda^{-2}$. 

Unlike the setting of~\cite{lelarge2016fundamental}, we need to account for the additional side information given in $\bnot$. This is achieved by noticing that, conditioning on $\bnot$ amounts to changing the conditional distribution of $\matu$; by conjugacy, $\matu | \bnot$ is still Gaussian, and its covariance is isotropic (Lemma~\ref{condition_dist_lem}). Lastly, by a symmetry argument, we show without loss of generality $\Exp[\matu| \bnot]$ is aligned with the all-ones vector. Thus, the coordinates of $\matu$ given $\bnot$ can be assumed to be i.i.d, returning us to the setting of~\cite{lelarge2016fundamental}. The proof of Theorem~\ref{thm:asmp_ovlap_thm} is formally given in Section~\ref{sec:thm_asmp_lvap_proof} below.

\subsection{Conjectures for Non-Asymptotic Bound on $\ovlapd$\label{sec:non_asymp_ovlap}}
We now introduce a conjecture under which we can bound $d$ by being polynomially large in relevant problem parameters. 
\begin{conj}[Non-Asymptotic Convergence]\label{conj:nonasymp}
There exists universal constants $c_0,\dots,c_3$ such that, for all $\lambda \in (1,2]$, all $\mu  \in (0,1)$,  $d \ge d_0$, and $\matu \sim \mathcal{N}(\mu/\sqrt{d},1/d)$, either (a)
\begin{align*}
\MMSE(\matu\matu^\top|\matM) \ge \lim_{d \to \infty} \MMSE(\matu\matu^\top|\matM) - c_0d^{-c_1}\cdot (\lambda - 1)^{-c_2} \cdot(1+\mu^{-c_3})~,
\end{align*}
where $\MMSE(\matu\matu^\top|\matM)$ is as defined in~\eqref{eq:MMSE_u_not_b}, or (b), the mutual information $\info(\matW + \lambda \matu \matu^{\top}, \matu\matu^\top)$ between $\matW + \lambda \matu \matu^{\top}$ and $\matu \matu^\top$ satisfies
\begin{align*}
|\info(\matW + \lambda \matu \matu^{\top};\matu \matu^{\top}) - \lim_{d \to \infty} \info(\matW + \lambda \matu \matu^{\top};\matu \matu^{\top})| \le c_0d^{-c_1}\cdot(\lambda - 1)^{-c_2} \cdot\mu^{-c_3}~.
\end{align*}
\end{conj}
The above conjecture simply says that the relevant information-theoretic quantities converge to their asymptotic values at polynomial rates in relevant problem conjectures.
The author believes that the dependence on $\mu \in (0,1)$ is not needed, but we accomodate this dependence in the conjecture because it does not affect what follows. 
In Section~\ref{sec:conj_rates}, we show that the above conjecture implies the desired bound non-asymptotic on $\ovlapd$:
\begin{prop}\label{prop:conj_rates}Conjecture~\ref{conj:nonasymp} part (b) implies Conjecture~\ref{conj:nonasymp} part (a), and Conjecture~\ref{conj:nonasymp} part (a) implies that there exists constants $c_1,c_2,c_3,c_0 > 0$ for which 
\begin{align}\label{eq:ovlapdnonasym}
\ovlapd(\taub) \le 1 - \frac{1}{\lambda^2} + \taub + \frac{\sqrt{\taub}}{\lambda}  +  c_0 d^{-c_1}(\lambda - 1)^{-c_2} \taub^{-c_3}~.
\end{align}
\end{prop}
In particular, if $\taub = (\lambda - 1)^2$, we get the following analogue of Corollary~\ref{cor:taublambdaovlp_asym}, which proves the second part of Proposition~\ref{prop:ovlap_prop}:
\begin{cor}\label{cor:taublambdaovlp_nonasym} If either Part (a) or (b) of Conjecture~\ref{conj:nonasymp} hold, then there exists universal constants $c_0,c_1 > 0$, $d \ge c_0 (\lambda - 1)^{-c_1}$, $\ovlapd(\taub) \le 5(\lambda -1)$.
\end{cor}

\subsection{Proof of Theorem~\ref{thm:asmp_ovlap_thm}\label{sec:thm_asmp_lvap_proof}}
Fix $\lambda \in (1,2]$ and $\taub \le (\lambda - 1)^2$. To prove Theorem~\ref{thm:asmp_ovlap_thm}, we relate $\ovlapd(\tau)$ to the Minimum Mean Squared Error of estimating $\matu\matu^{\top}$ given $\matM$ and $\matb$. Define the conditional MMSE
\begin{align}
\MMSE(\matu\matu^\top\mid \matM,\matb) :=  \Exp_{\matu}\left[ \|\matu \matu^\top -  \Exp[\matu\matu^{\top} | \matM,\matb] \|_{\F}^2  | \matM,\matb\right]~,
\end{align}
which is the minimum mean squared error attainable by any estimate of $\matu\matu^{\top}$ given access to $\matM$ and $\matb$. As described above, the $\MMSE$ is intimately connected to the mutual information between $\matM,\matb$ and $\matu\matu^\top$, and we shall be able to use this fact below to control this term. Moreover, $\ovlapd(\taub)$ can be be bounded by $\MMSE(\matu\matu^\top;\matM,\matb) $ via the following esimate (proved in Section~\ref{sec:lem_ovlap_to_MMSE})
\begin{lem}\label{lem:ovlap_to_MMSE} For $\lambda, \taub \le 2$,  exists universal constants $c_1,c_2$ such that for any estimator $\uhat = \uhat(\matM,\bnot) \in \sphere$,
\begin{align*}
\Exp_{\matM,\matb}\Exp[\langle \uhat, \matutil \rangle^2 \big{|} \matM,\bnot] \le \sqrt{\Exp[\|\matu\|_2^2]^2 - \Exp_{\matM,\matb}[\MMSE(\matu\matu^\top\mid \matM,\bnot) ]} + c_1d^{-c_2}~.\end{align*}
\end{lem}
By Jensen's inequality, we upper bound the above display by the minimum mean-squared error, conditioned on $\bnot$
\begin{align*}
\Exp[\|\matu\|_2^2]^2 - \Exp_{\matM,\matb}[\MMSE(\matu\matu^\top|\matM,\bnot) ] \le \Exp_{\matb}[ \Exp_{\matM,\matu}[\|\matu\|_2^2 \mid \bnot]^2 - \Exp_{\matM}[\MMSE(\matu\matu^\top\bnot,\matM) ] 
\end{align*}
Our next step is to reduce the computation of the above MMSE to a setting in which the results of~\cite{lelarge2016fundamental} hold. It will be convenient to define the conditional cross term:
\begin{align*}
\cross(\matu\matu^\top\mid\matb) :=  \Exp_{\matu,\matM}\left[  \|\Exp[\matu\matu^{\top} | \matM,\matb] \|_{\F}^2  \right]~.
\end{align*}
A standard computation reveals that, conditioned on $\bnot$
\begin{align*}
\Exp_{\matM,\matu}[\|\matu\|_2^2 \mid \bnot]^2 - \Exp_{\matM}[\MMSE(\matu\matu^\top\mid \bnot,\matM)] = \cross(\matu\matu^\top\mid\matb), 
\end{align*}
and thus 
\begin{align}
\ovlapd(\taub) = \sup_{\uhat} \Exp_{\matM,\matb}\Exp[\langle \uhat, \matutil \rangle^2 \big{|} \matM,\bnot] \le \sqrt{\Exp_{\bnot}[ \cross(\matu\matu^\top\mid\matb)} + c_1d^{-c_2}~.\label{eq:cross_bound}
\end{align}

Next, via Lemma~\ref{condition_dist_lem}, we check the conditional distribution $\matu | \bnot \sim \mathcal{N}\left(\frac{\sqrt{\taub} \bnot}{1+\taub} ,\frac{1}{1+\taub} \cdot \frac{I}{d}\right)$. By rotation invariance, we argue that we may assume that $\bnot$ is alinged with the all ones vector. This, combined with some truncation, lets us bound $\ovlapd(\taub)$ in terms of a cross term parameterized by the conditioned mean of $\matu$, which we denote $\mu$. For consistency with~\cite{lelarge2016fundamental}, we also reparameterize the deformation parameter with $\lambda \leftarrow \sqrt{\rho}$: Precisely, Appendix~\ref{sec:prop_ovlap_conversion_proof} proves the following estimate

\begin{prop}\label{prop:ovlap_conversion} Define the mean parametrized-minimum mean squared error:
\begin{align}
\crossch_{d}(\rho;\mu) &:= \|\Exp[\uch\uch^{\top} | \Mch] \|_{\F}^2~ \nonumber\\
\text{where } \Mch &:= \matW + \sqrt{\rho} \uch\uch^\top, \uch_i \overset{(i.i.d)}{\sim} \mathcal{N}(\mu/\sqrt{d},1/d)~. \label{eq:ovlap_mu}
\end{align} 
Then, letting $\boldalpha$ have the distribution of $\|\matx\|$ for $\matx \sim \calN(0,I/d)$, and letting $\rho_{\taub} = (\lambda/(1+\taub))^2$ and $\mu_{\taub} = \sqrt{\taub}$,  we have
\begin{align*}
\ovlapd(\taub) \le \frac{ \sqrt{\Exp_{\boldalpha}\I(|\boldalpha - 1| \le d^{-\frac{1}{4}})\crossch_{d}(\rho_{\taub};\boldalpha \mu_{\taub}) }}{1+\taub}  + c_1d^{-c_2}~.
\end{align*}
for universal constants $c_1,c_2$.
\end{prop}

The upshot of using the mean-paramterized term $\crossch_{d}(\rho;\mu)$ is that it is defined in terms of the random vector $\matutil_i \iidsim \mathcal{N}(\mu/\sqrt{d},1/d)$, which has independent and identically distributed coordinates. This allows us to use Theorem 1 in~\cite{lelarge2016fundamental}, which gives an exact expression for the asymptotic value for this term. We also have to have to deal with the wrinkle that we are considering an expectation of these terms, $\Exp_{\boldalpha}\I(|\boldalpha - 1| \le d^{-\frac{1}{4}})\crossch_{d}(\rho_{\taub};\boldalpha \mu_{\taub}) $. Moreover, our observation model is slightly different than the one considered in~\cite{lelarge2016fundamental}. Hence, we shall have to careful modify the guarantees from the past work to establish the following theorem, whose proof we defer to Appendix~\ref{sec:MSE_Limit_Proof}:

\begin{thm}\label{thm:main_asymptotic_comp} Fix a $\rho \ge 1$ and $\mu > 0$. Then, 
\begin{align*}
 \lim_{d \to \infty} \Exp_{\boldalpha}\left[\I(|\boldalpha - 1|\le d^{-1/4}) \crossch_{d}(\rho;\mu\boldalpha)\right] \le \left(1+\mu^2 - \frac{1}{\rho} + \frac{|\mu|}{\sqrt{\rho}}\right)
\end{align*}
\end{thm}

We may now conclude the proof of Theorem~\ref{thm:asmp_ovlap_thm}. Plugging in $\rho_{\taub}$ and $\mu_{\taub}$ into Theorem~\ref{thm:main_asymptotic_comp}, 
\begin{align*}
\lim_{d \to \infty}\ovlapd(\taub) &\le \frac{ \sqrt{\lim_{d \to \infty} \Exp_{\boldalpha}\I(|\boldalpha - 1| \le d^{-\frac{1}{4}})\crossch_{d}(\rho_{\taub};\boldalpha \mu_{\taub}) }}{1+\taub} \\
&\le \frac{1+\mu_{\taub}^2 - \frac{1}{\rho_{\taub}} + \frac{|\mu_{\taub}|}{\sqrt{\rho}}}{1+\taub}   \\
&= \frac{1+\taub- \frac{(1+\taub)^2}{\lambda^2} + \frac{\sqrt{\taub}(1+\taub)}{\lambda}}{1+\taub}   \\
&= 1 - \frac{(1+\taub)}{\lambda^2}- \frac{\sqrt{\taub}}{\lambda} \le 1 - \frac{1}{\lambda^2} +\taub - \frac{\sqrt{\taub}}{\lambda},
\end{align*}
where in the last line we used $\lambda \ge 1$.

\clearpage
\bibliographystyle{plainnat}
\bibliography{main_quad_lb}
\clearpage

\clearpage
\appendix
\tableofcontents
\newcommand{\Mcross}{M_{\mathrm{cross}}}
\section{Proof of Proposition~\ref{prop:Ex}\label{proof:Ex}}
	\textbf{Notation: } Throughout, we assume $\lambda \in (1,\frac{3}{2}]$. Let $\gamma := 2(\lambda + \lambda^{-1}) - 2$.   Recall the notation that $Z = \smallo$ if $\Pr[\smallo \ge \mult - 1] \le \deltaexp$, or equivalently, for any $\epsilon > 0$,
	\begin{align}\label{small_o_esp}
	\Pr[|Z| \ge \epsilon] \le \exp( - d^{c_1}\epsilon^{c_2}(\lambda - 1)^{c_3})
	\end{align}
	for constants $c_0,c_1, c_2,c_3 > 0$. We will also use the notation $\dellam$ to denote a term which is at most $\exp( - c_0 d^{c_1}(\lambda - 1)^{c_2})$. Finally, we say $W = \baro$ if there is are constants $c_0,\dots,c_4 > 0$ such that $\Pr[|W| \ge c_4(\lambda - 1)^{c_3}] \le \exp( - c_0d^{c_1}(\lambda - 1)^{c_2})$. We shall use the following observation throughout:
	\begin{fact}\label{fact:baro} If $W = \baro$ and $Z = \smallo$, then $WZ = \smallo$, and $W + Z = \baro$. Moreover, $|Z|^{p} = \smallo$ for any fixed constant $p > 0$, and if $Z' = \smallo$, $ZZ' = \smallo$.
	\end{fact}

	\begin{proof}[Proof of Proposition~\ref{prop:Ex}] We begin by writing out
	\begin{align}
	\langle \matxst, \matu \rangle = \langle \matA^{-1}\matb , \matu \rangle = \sqrt{\taub} \matu^{\top}\matA^{-1}\matu + \matz^{\top}\matA^{-1}\matu
	\end{align}
	and 
	\begin{align}\label{eq:matxsq}
	\|\matx\|_2^2 = \matb^{\top} \matA^{-2}\matb  = \taub\matu^{\top}\matA^{-2}\matu + \matz^{\top}\matA^{-2}\matz + 2\sqrt{\taub}\matu^{\top}\matA^{-2}\matz
	\end{align}
	The following lemma (proof in Section~\ref{sec:smallo_o_lem}) shows that $\matz^{\top}\matA^{-1}\matu$ and neglible $\matz^{\top}\matA^{-2}\matu$:

	\begin{lem}\label{lem:smallo} $\matz^{\top}\matA^{-1}\matu = \smallo$ and $\matu^{\top}\matA^{-2}\matz = \smallo$. More precisely, there is a term $\Mcross = \baro$ such that event $\Ecross(\delta) := \{ \matz^{\top}\matA^{-1}\matu \le \Mcross \cdot (d\log(1/\delta))^{-1/2}\}$ occurs with probability at least $1 - \delta$. 
	\end{lem}

	Throughout, we shall fix $\delta = e^{-d^{.05}}$. Next, we unpack our terms via the Sherman-Morrison idenity, which states that any invertible $A \in \R^{d \times d}$,  and $x, y \in \R^{d}$, one has
	\begin{align*}
	(A + xy^\top) = A^{-1} - \frac{A^{-1}xy^\top A^{-1}}{1 + y^\top A^{-1}x}
	\end{align*}
	In particular, define the denominator term 
	\begin{align*}\denom := 1 - \lambda \matu^{\top}(\gamma I - \matW)\matu,
	\end{align*} 
	we have
	\begin{align}\label{eq:matAinv}
	\matA^{-1} = (\gamma I - \matW - \lambda \matu\matu^\top) = (\gamma I - \matW)^{-1} + \frac{\lambda (\gamma I - \matW )^{-1}\matu\matu^{\top}(\gamma I - \matW )^{-1}}{\denom} 
	\end{align}
	and thus, with probability at least $1-\delta = 1 - \exp(-d^{.05})$, 
	\begin{align}
	\langle \matxst, \matu \rangle &\overset{\text{Lem.}~\ref{lem:smallo}}{=} \sqrt{\taub} \matu^{\top}\matA^{-1}\matu + \Mcross \cdot (d\log(1/\delta))^{-1/2}\nonumber\\
	&= \sqrt{\taub} \left\{\matu^{\top}(\gamma I - \matW)^{-1}\matu +  \frac{\lambda(\matu^{\top}(\gamma I - \matW )^{-1}\matu)^2}{\denom}\right\} + \Mcross \cdot (d\log(1/\delta))^{-1/2} \nonumber\\
	&=\sqrt{\taub}\matu^{\top}(\gamma I - \matW)^{-1}\matu \cdot \left\{ 1 +  \frac{\lambda\matu^{\top}(\gamma I - \matW )^{-1}\matu}{\denom}\right\} + \Mcross \cdot (d\log(1/\delta))^{-1/2} \nonumber\\
	&\overset{(i)}{=}  \sqrt{\taub} \frac{\matu^{\top}(\gamma I - \matW)^{-1}\matu}{\denom} + \baro \cdot (d\log(1/\delta))^{-1/2}\label{eq:matxinner}~,
	\end{align}
	where $(i)$ uses $\lambda\matu^{\top}(\gamma I - \matW )\matu = 1 - \denom$, and that $\Mcross = \baro$. To bound~\eqref{eq:matxsq}, we need to control $\matu^\top \matA^{-2}\matu$ and $\matz^\top \matA^{-2}\matz$. This is achieved by the following lemma, proved in Section~\ref{sec:lem:second_order}. 
	\begin{lem}\label{lem:second_order} The following estimates hold:
	\begin{align*}
	\matu^\top \matA^{-2}\matu &= \matu^\top(\gamma I - \matW)^{-2}\matu \cdot  \frac{1}{\denom^2} \\
	 \matz^\top \matA^{-2}\matz &= \matz^\top (\gamma I - \matW)^{-2}\matz + \smallo \left\{\frac{|\denom| +\matu^\top (\gamma I - \matW)^{-2} \matu }{\denom^2} + \right\}
	\end{align*}
	\end{lem}
	Inspecting Lemma~\ref{lem:second_order} and~\eqref{lem:smallo}, we see that the terms we must control are $\matu^{\top}(\gamma I - \matW)^{-1} \matu)$, $\matz^{\top}(\gamma I - \matW)^{-1} \matz)$, and $\matz^{\top}(\gamma I - \matW)^{-1} \matz$. Our first step is to invoke the Hanson-Wright inequality (see Section~\ref{sec:lem:HansonWright} for proof):
	\begin{lem}\label{lem:HansonWright} $\matz^{\top}(\gamma I - \matW)^{-1} \matz = \frac{1}{d}\tr(\gamma I - \matW)^{-1} + \smallo$, $\matu^{\top}(\gamma I - \matW)^{-1} \matu = \frac{1}{d}\tr(\gamma I - \matW)^{-1} + \smallo$, and $\matu^{\top}(\gamma I - \matW)^{-2} \matu = \frac{1}{d}\tr(\gamma I - \matW)^{-1} + \smallo$
	\end{lem}
	Using the bounds $\tr(\gamma I - \matW)^{-1} = \stielt(\gamma) + \smallo$ and $\tr(\gamma I - \matW)^{-2} = \quielt(\gamma) + \smallo$ from Proposition~\ref{prop:trace_comps}, we have the following estimates:
	\begin{align*}
	\langle \matxst, \matu \rangle &= \sqrt{\taub} \frac{\stielt(\gamma) + \smallo}{\denom} + \baro \cdot (d\log(1/\delta))^{-1/2}~. \\
	\matu^{\top}\matA^{-1}\matu &= (\quielt(\gamma) + \smallo) \cdot  \frac{1}{\denom^2}~. \\ 
	\matz^{\top}\matA^{-1}\matz &= \quielt(\gamma) + \smallo + \smallo \cdot\frac{|\denom| + (\quielt(\gamma) + \smallo)}{\denom^2}~.
	\end{align*}
	We can see that $\denom = 1 - \lambda \stielt(\gamma) + \smallo$, and using the fact that $\stielt(\gamma)$, $1/\stielt(\gamma)$ $1-\lambda \stielt(\gamma) $ and $\quielt(\gamma)$ are all $\baro$  (deterministically!). Hence, invoking Fact~\ref{fact:baro} to simplify terms in the denominator, we have
	\begin{align*}
	\frac{\langle \matxst, \matu \rangle^2}{\|\matxst\|^2} &= \frac{\denom^2\langle \matxst, \matu \rangle^2}{\denom^2\|\matxst\|^2}\\
	&= \frac{\left\{\sqrt{\taub} \cdot(\stielt(\gamma) + \smallo) \pm \Mcross \cdot (d\log(1/\delta))^{-1/2}\right\}^2 }{(\quielt(\gamma) + \smallo)(\taub + \denom^2) + \smallo  \cdot \{(1 + |\stielt(\gamma)|) + \quielt(\gamma) + \smallo\})}~\\
	&= \frac{\left\{\sqrt{\taub} \cdot(\stielt(\gamma) + \smallo) \pm \baro \cdot (d\log(1/\delta))^{-1/2}\right\}^2}{\quielt(\gamma) (\taub + (1 - \lambda \stielt(\gamma))^2) + \smallo }~.\\
	&= \taub \cdot \frac{\left\{(1 + \smallo) \pm \baro \cdot (\taub d\log(1/\delta))^{-1/2}\right\}^2}{\stielt(\gamma)^{-2} \cdot \quielt(\gamma) (\taub + (1 - \lambda \stielt(\gamma))^2) + \smallo }~,
	\end{align*}
	where in the last line, we divided the numerator and denominator both by $\stielt(\gamma)$, unsing the fact that $1/\stielt(\gamma) = \baro$ (see~\eqref{eq:stieltLB}), and simplifying with Fact~\ref{fact:baro}. Let's simplied the numerator a bit. As long as $\taub \ge d^{-.9}$, we can see that with probability $1-\delta = 1 - \exp(d^{-.05})$,
	\begin{align*} 
	1 + \smallo \pm \baro \cdot (\taub d\log(1/\delta))^{-1/2} &=  1 + \smallo \pm \baro \sqrt{d^{-.9}\cdot d \cdot d^{-.05} } \\
	&=  1 + \smallo \pm \baro/d^{.025}.
	\end{align*} 

	We now introduce a lemma which allows us to 
	\begin{lem}\label{lem:stielt_comp}
	$\quielt(\gamma) \stielt(\gamma)^{-2} \le \frac{3}{2(\lambda - 1)}$ and $1 - \lambda \stielt(\gamma)  \le (\lambda - 1)$.
	\end{lem}
	Moreover, since $\taub \le (\lambda  -1 )$, we conclude that $\stielt(\gamma)^{-2} \cdot \quielt(\gamma) (\taub + (1 - \lambda \stielt(\gamma))^2) \le 3(\lambda - 1)$, so that 
	\begin{align*}
	\frac{\langle \matxst, \matu \rangle^2}{\|\matxst\|^2} &= \taub \cdot \frac{1+\smallo}{3(\lambda - 1)+\smallo}~,
	\end{align*}
	which implies the proposition.
	\end{proof}

\subsection{Supporting Concentration Proofs}

\subsubsection{Proof of Lemma~\ref{lem:smallo}\label{sec:smallo_o_lem}} Consider the terms $\matz^{\top}\matA^{-\ell} \matu $ for $\ell \in \{1,2\}$. By standard Gaussian concentration, and the fact that $\matA^{\ell}\matu$ and $ \matz$ are independent, we have that for any $\delta > 0$,
\begin{align*}
\langle \matA^{\ell}\matu, \matz \rangle \le \|\matA^{\ell}\matu\| \cdot \sqrt{2 \log(1/\delta)/d},\quad \text{with probability } 1 - \delta.
\end{align*}
It now suffices to show that $\Mcross := \max_{\ell \in \{1,2\}}\|\matA^{-\ell}\matu\|_2 = \baro$. Indeed, this will directly prove the second statement of the lemma, and the first statement will follows since the above display implies that $\langle \matA^{\ell}\matu, \matz \rangle \le \|\matA^{\ell}\matu\| \smallo \le \Mcross \smallo$, and if $\Mcross = \baro$, then $\Mcross \cdot \smallo = \baro\cdot \smallo = \smallo$ by Fact~\ref{fact:baro}. To this end, we bound
\begin{align*}
\Pr[\|\matA^{-\ell}\matu\|_2 \ge 2(\frac{\sqrt{2}\lambda}{(\lambda - 1)^2})^{\ell}] &\le \Pr[\|\matu\|_2 \ge 2] + \Pr[\|\matA\|_{2} \ge (\frac{\sqrt{2}\lambda}{(\lambda - 1)^2})^{-1}] \le \dellam~,
\end{align*}
wher the last inequality is standard gaussian concentration for $\|\matu\|_2$, and Proposition~\ref{prop:EA} for bounding $\Pr[\|\matA\|_{2} \ge (\frac{\sqrt{2}\lambda}{(\lambda - 1)^2})^{-1}]$.

\subsubsection{Proof of Lemma~\ref{lem:HansonWright}\label{sec:lem:HansonWright}}
By Theorem~\ref{thm:main_spec_thm} (which bounds $\|\matW\| \le 2 + d^{-\Omega(1)}$ with high probability), we see that $\gamma I - \matW \succsim (\lambda - 1)^2$ with probability $1-\dellam$. The bounds now follow from a routine application of the Hanson-Wright inequality (see, e.g.~\cite{rudelson2013hanson}) on the event $\{\gamma I - \matW \succsim (\lambda - 1)^2\}$, and noting that $\matu$ and $\matz$ are both independent of $\matW$.

\subsection{Proof of Lemma~\ref{lem:second_order}\label{sec:lem:second_order}}
	In light of~\eqref{eq:matAinv}, we have that 
	\begin{align*}
	\matA^{-1} &=\left( (\gamma I - \matW)^{-1} + \frac{(\gamma I - \matW)^{-1}(\lambda \matu \matu^\top)(\gamma I - \matW)^{-1}}{\denom}\right)^2\\
	&= (\gamma I - \matW)^{-2} + 2\lambda \Symm{\frac{(\gamma I - \matW)^{-2} (\matu \matu^\top)(\gamma I - \matW)^{-1}}{\denom}}\\
	&+ \lambda^2 \matu^\top (\gamma I - \matW)^{-2} \matu \frac{ (\gamma I - \matW)^{-1}  \matu \matu^\top(\gamma I - \matW)^{-1}}{\denom^2}
	\end{align*}
	\textbf{A. Computing $\matu^\top\matB^{-2}\matu$}. Using the above, we have that $\matu^\top\matB^{-2}\matu$
	\begin{align*} 
		& \matu^\top \matA^{-2}\matu \\
		 &= \matu^\top (\gamma I - \matW)^{-2}\matu + 2\lambda \frac{\matu^\top (\gamma I - \matW)^{-2} (\matu \matu^\top)(\gamma I - \matW)^{-1}\matu }{\denom}\\
		 &+ \lambda^2 \matu^\top (\gamma I - \matW)^{-2} \matu \frac{ \matu^\top(\gamma I - \matW)^{-1}  \matu \matu^\top(\gamma I - \matW)^{-1}\matu }{\denom^2}\\
		 &= \matu^\top (\gamma I - \matW)^{-2}\matu \cdot \left\{1+ 2\lambda \frac{\matu^\top(\gamma I - \matW)^{-1}\matu }{\denom} + \lambda^2\left(\frac{\matu^\top(\gamma I - \matW)^{-1}\matu }{\denom}\right)^2 \right\}\\
		 &= \matu^\top (\gamma I - \matW)^{-2}\matu \cdot \left\{1+ \frac{\lambda \matu^\top(\gamma I - \matW)^{-1}\matu }{\denom} \right\}^2\\
		 &= \matu^\top(\gamma I - \matW)^{-2}\matu \cdot \left( \frac{1}{\denom}\right)^2~\quad (\lambda\matu^{\top}(\gamma I - \matW )\matu = 1 - \denom) .
		\end{align*}
		\noindent \textbf{B. Computing $\matz^\top\matB^{-2}\matz$.} We now compute
		\begin{align*} 
		& \matz^\top (\gamma I - \matW - \lambda \matu \matu^\top)\matz \\
		 &= \matz^\top (\gamma I - \matW)^{-2}\matz + 2\lambda \frac{\matz^\top (\gamma I - \matW)^{-2} (\matu \matu^\top)(\gamma I - \matW)^{-1}\matz  }{\denom}\\
		 &+ \lambda^2 \matu^\top (\gamma I - \matW)^{-2} \matu \frac{ \matz^\top(\gamma I - \matW)^{-1}  \matu \matu^\top(\gamma I - \matW)^{-1}\matz }{\denom^2}~.
		\end{align*}
		By Lemma~\ref{lem:smallo}, $\matz^\top (\gamma I - \matW)^{-2} \matu = \smallo$ and $\matu^\top(\gamma I - \matW)^{-1}\matz  = \smallo$. Thus, 
		\begin{align*} 
		& \matz^\top (\gamma I - \matW - \lambda \matu \matu^\top)\matz \\
		 &= \matz^\top (\gamma I - \matW)^{-2}\matz +  \frac{2\lambda \smallo\cdot \smallo}{\denom} + \lambda^2 \matu^\top (\gamma I - \matW)^{-2} \matu \frac{ \smallo \cdot \smallo}{\denom^2}\\
		 &= \matz^\top (\gamma I - \matW)^{-2}\matz +  \frac{\smallo}{\denom} + \frac{\smallo \cdot \matu^\top (\gamma I - \matW)^{-2} \matu   }{\denom^2}~,
		\end{align*}
		where the last step uses $\smallo \cdot \smallo= \smallo$ by Fact~\ref{lem:smallo}, and the fact that $\lambda \le 2$. Factoring out the $\smallo$ term yields 
		\begin{align*}
		 \matz^\top (\gamma I - \matW - \lambda \matu \matu^\top)\matz = \matz^\top (\gamma I - \matW)^{-2}\matz +  \smallo \cdot \{ \frac{|\denom| + \matu^\top (\gamma I - \matW)^{-2} \matu  }{\denom^2}\}~.
		\end{align*}
		
\subsection{Proof of Lemma~\ref{lem:stielt_comp}\label{sec:lem:stielt_comp}}

We shall begin with an explicit expression for $\quielt(a)$:

	\begin{lem}\label{lem:stieltdef} $\quielt(a) :=  -\frac{\rmd}{\rmd a} \stielt(a) = \frac{\stielt(a)}{\sqrt{a - 4}}$.
	\end{lem}
	\begin{proof} Recalling $\stielt(a) = \frac{a - \sqrt{a^2 - 4}}{2}$, we have $\quielt(a) = -\rmda\stielt(a) = \frac{-1}{2}\left(1 -  \frac{a}{\sqrt{a^2 - 4}}\right)$.  Rearranging, we find $\frac{1}{\sqrt{a^2 - 4}}(\frac{\sqrt{a^2 - 4} - 1}{2})$, and we recognize $\frac{\sqrt{a^2 - 4} - 1}{2} := \stielt(a)$.
	\end{proof}

\textbf{Upper bound on $\quielt(\gamma)$:} By Lemma~\ref{lem:stieltdef}, we have  $\quielt(\gamma) = \frac{\stielt(\gamma)}{\sqrt{\gamma^2 - 4}}$. Hence, 
\begin{align*}
\quielt(\gamma) \stielt(\gamma)^{-2} = \frac{1}{\stielt(\gamma)\sqrt{\gamma^2 -4}}
\end{align*}
Moreover, noting that $\gamma = 2(\lambda + \lambda^{-1}) -2 \in [2,3]$ for $\lambda \in (1,2]$,
\begin{align}
\stielt(\gamma) &= \frac{\gamma - \sqrt{\gamma^2 - 4}}{2} ~=~ \frac{\gamma^2 - (\gamma^2- 4)}{2(\gamma + \sqrt{\gamma^2 - 4})}\nonumber\\
&= \frac{4}{2(\gamma + \sqrt{\gamma^2 - 4})} \ge \frac{1}{\gamma} \ge \frac{1}{3}~. \label{eq:stieltLB}
\end{align}

 Letting $\boldgap = \lambda + \lambda^{-1} - 2$, we have 
\begin{eqnarray}
\sqrt{\gamma^2 - 4} &=& \sqrt{(\lambda + \lambda^{-1} + \boldgap)^2 - 4 } \nonumber\\
&=& \sqrt{(\lambda + \lambda^{-1})^2 - 4 +   2\boldgap(\lambda + \lambda^{-1}) + \boldgap^2} \nonumber\\
&=& \sqrt{(\lambda + \lambda^{-1} - 2)(\lambda + \lambda^{-1} + 2) +   2\boldgap(\lambda + \lambda^{-1}) + \boldgap^2} \nonumber\\
&=& \sqrt{\boldgap(\lambda + \lambda^{-1} + 2) +   2\boldgap(\lambda + \lambda^{-1}) + \boldgap^2} \nonumber\\
&=& \sqrt{\boldgap \cdot (3(\lambda + \lambda^{-1}) + 2 + \boldgap)} \nonumber\\
&=& \sqrt{\boldgap\cdot 4(\lambda + \lambda^{-1}) } \nonumber\\
&=& 2(\lambda - 1)\sqrt{1 + \lambda^{-2}}  \label{quielt:lastline}
\end{eqnarray}
Hence, we conclude 
\begin{align*}
\quielt(\gamma) \stielt(\gamma)^{-2} = \frac{1}{\stielt(\gamma)\sqrt{\gamma^2 -4}} \le \frac{3}{2(\lambda - 1)}
\end{align*}

\textbf{Upper Bound for $1 - \lambda \stielt(\gamma)$.} 
We begin by upper bound $\stielt(\gamma)$ via
\begin{align*} 
1 - \lambda\stielt(\gamma) &=~  1 - \lambda \cdot \frac{\gamma -  \sqrt{\gamma^2 - 4})}{2} \quad=\quad 1 - \frac{\lambda (\lambda + \lambda^{-1} + \boldgap) -  \lambda \sqrt{\gamma^2 - 4})}{2} \\
&=~ 1 - \frac{\lambda^2 + 1 + \lambda(\boldgap - \sqrt{\gamma^2 - 4})}{2}
 \quad=\quad  1 - \frac{\lambda^2 + 1 + \lambda(\boldgap - \sqrt{\gamma^2 - 4})}{2} \\
 &=~  \frac{ \lambda \sqrt{\gamma^2 - 4} - \lambda\boldgap  - (\lambda^2 - 1)}{2}
 \quad=\quad   \frac{ \lambda \sqrt{\gamma^2 - 4} - (\lambda-1)^2  - (\lambda^2 - 1)}{2} \\
 &=~  \frac{ \lambda \sqrt{\gamma^2 - 4} - \lambda^2 + 2\lambda - 1  - \lambda^2 + 1)}{2}
  \quad=\quad   \frac{ \lambda \sqrt{\gamma^2 - 4} - 2\lambda(\lambda-1)}{2} \\
  &\overset{(i)}{=}~  \frac{ 2\lambda(\lambda - 1)\sqrt{1 + \lambda^{-2}} - 2\lambda(\lambda-1)}{2} \quad=\quad  \lambda(\lambda - 1) \cdot (\sqrt{1 + \lambda^{-2}} - 1 ) \\
  &=~ (\lambda - 1) (\sqrt{\lambda^2 + 1} - \lambda) \le (\lambda - 1)~.
\end{align*}
where $(i)$ uses~\eqref{quielt:lastline}.

\section{Random Matrix Theory: Proof of Propositions~\ref{prop:EA} and~\ref{prop:trace_comps}}

\subsection{Proof of Proposition~\ref{prop:EA}\label{proof:prop:EA}}
Recall the $\smallo$-notation from~\eqref{small_o_esp}, that $Z = \smallo$ if $\Pr[|Z| \ge \epsilon] \le \exp( - d^{c_1}\epsilon^{c_2}(\lambda - 1)^{c_3})$. Moreover, observe the equivalence that if $W$ is a random quantity, and $W_0$ is deterministic, and if, $W_0 \ge (\lambda - 1)^{c}$ for some constant $c$, then $W - W_0 = \smallo$ implies $\Pr[ \mult^{-1} W \le W_0 \le \mult W] =\deltaexp$ for any $\mult > 1$. Thus, to prove Proposition~\ref{prop:EA}, it suffices to show
\begin{align*}
\lambda_{1}(\matA)  \le 2(\lambda + \lambda^{-1}) + \smallo \quad \text{ and }\quad  \lambda_{d}(\matA) \ge (\lambda-1)^2/\lambda + \smallo
\end{align*}
Further, we observe that 
\begin{align*}
\lambda_{1}(\matA) &= 2(\lambda + \lambda^{-1}) -2  - \lambda_d(\matW + \lambda \matu \matu^{\top}) \\
&\overset{(i)}{\le}~ 2(\lambda + \lambda^{-1}) -2 - \lambda_d(\matW) \le 2(\lambda + \lambda^{-1}) + (\|\matW\|_{\op} - 2)~,
\end{align*}
where $(i)$ is by eigenvalue interlacing. Moreover, we have that 
\begin{eqnarray*}
\lambda_{d}(\matA) = (\lambda + \lambda^{-1} - 2) + \lambda + \lambda^{-1} - \lambda_1(\matM)~.
\end{eqnarray*}
Hence, to conclude, it suffices to verify that $\|\matW\|_{\op} - 2 = \smallo$ and $\lambda + \lambda^{-1} - \lambda_1(\matM) = \smallo $. This is a direct consequence of the following finite sample convergence bound from~\cite{simchowitz2018tight}: 
\begin{thm}[Rank-1 Specialization of Theorem 6.1 in~\cite{simchowitz2018tight}]\label{thm:main_spec_thm} There exists a universal constant $C \ge 0$ such that the following holds. Let $\matM = \matW + \boldlam \matu\matu^\top $, and let $\boldgap := \frac{(\boldlam-1)^2}{\boldlam^2 + 1}$. Let $\kappa \le 1/2$, $\epsilon \le \boldgap \cdot \min\{\frac12, \frac{1}{\boldlam^2 - 1}\}$, and $\delta >0$. Then for 
\begin{eqnarray}
d \ge  C \left(\frac{(q+\log(1/\delta)) }{ \boldgap \epsilon^2} + (\kappa\boldgap)^{-3} \log(1/\kappa\boldgap) \right),
 \end{eqnarray} 
 the event the event $\calE_{\matM}$ defined below holds with probability at least $1 - 9\delta$:
 \begin{eqnarray*}
	\calE_{\matM} := \left\{ \|\matW\|_{\op}  \le 2 + \kappa (\lambda + \lambdainv - 2) \right\} \bigcap \left\{\lambda_1(\matM) \in (\boldlam + \lambdainv)[1-\epsilon,1+\epsilon]\right\}.
\end{eqnarray*}

	
\end{thm}

\subsection{Proof of Proposition~\ref{prop:trace_comps}\label{proof:prop:trace_comps}}

Before showing proving Proposition~\ref{prop:trace_comps}, we will reducing bounding $|\quielt(a) - \tr(aI - \matW)^{-2}|$ to bound $|\quielt(a) - \tr(aI - \matW)^{-1}|$. Throughout, we shall take $\lambda \in (1,2]$, $\gamma = 2(\lambda + \lambda^{-1}) - 2$, The reduction if facilliated by the following proposition:
\begin{prop}\label{prop:stiel_to_quielt} Let $C \ge 8$ denote a universal constant,  and fix $\epsilon \le (\lambda - 1)$. Then then, there exists a (deterministic) $t = t(\lambda,\epsilon)$ such that (a) $t \le \frac{\gamma -2}{2}$ and (b) on the event 
\begin{align*}
\left\{\|\matW\|_{\op} \ge \gamma - t\right\} \cap \left\{\max_{a \in \{\gamma - t, \gamma , \gamma + t\}} |\tr(aI - \matW)^{-1} - \stielt(a)| \le \epsilon\right\}
\end{align*}
it holds that  $|\tr(aI - \matW)^{-2} - \quielt(a)| \le 2\sqrt{2C(\lambda - 1)^{-3}\epsilon}$.
\end{prop}
\begin{proof}[Proof of Proposition~\ref{prop:stiel_to_quielt}]  Let $C \ge 2$ be a constant defined in Lemma~\ref{lem:quielt_derv} below, let $L:= C(\lambda - 1)^{-3}$, and let $t := \sqrt{2\epsilon/L} = \sqrt{2\epsilon(\lambda - 1)^{-3}/C}$. 
Observe that, since $\epsilon \le \lambda - 1$ and $C \ge 2$, we have that
\begin{eqnarray}
t \le \lambda + \lambda^{-1} - 2 = \frac{\gamma - 2}{2}~.
\end{eqnarray}
We now assume that the following event holds:
\begin{align*}
\left\{\|\matW\|_{\op} \ge \gamma - t\right\} \cap \left\{\max_{a \in \{\gamma - t, \gamma , \gamma + t\}} |\tr(aI - \matW)^{-1} - \stielt(a)| \le \epsilon\right\}
\end{align*}
If we define the maps
\begin{align*}
f(a) := -\stielt(a) \quad \text{and} \quad g(a) := -\tr(aI - \matW)^{-1} = -\sum_{i=1}^d \frac{1}{a - \lambda_{i}(\matW)}~,
\end{align*}
we observe that on the event $\{\|\matW\|_{\op} < \gamma - t\}$, $g(a)$ is concave and differentiable on $[\gamma - t,\infty)$, with $g'(a) = \tr(aI - \matW)^{-2} $, and $f(a)$ is differentiable on $(2,\infty)$, with $f'(a) = \quielt(a)$. The following lemma shows in addition that $f'(a)$ is $L$ Lipschitz for $a \in [\gamma - t,\gamma + t]$:
\begin{lem}\label{lem:quielt_derv} Let $\lambda \le 2$,  $\gamma = 2(\lambda + \lambda^{-1}) - 2$, and $t \le (\gamma - 2)/2$. Then there is a universal constant $C \ge 8$ for which 
\begin{align*}
\max_{a \in [\gamma - t,\gamma + t]}|\quielt'(a)| \le C(\lambda - 1)^{-3}~.
\end{align*}
\end{lem}             
To conclude, we invoke the following approximation bound for concave functions, proved in Section~\ref{lem:cvx_proof} below:
\begin{lem}\label{lem:cvx_approximation} Let $L > 0$ and $\epsilon > 0$, and set $t=\sqrt{2\epsilon/L}$. Then if $g,f:[x-t,x+t] \to \R$ are such that (a) $g$ be a concave, differentiable function on $[x-t,x+t]$, (b) $f' (x)$ exists and is $L$-Lipschitz $[x-t,x+t]$, and $(c)$ for all $a \in \{x-t,x,x+t\}$, $|f(a) - g(a)| \le \epsilon$, then $|f'(x) - g'(x)| \le 2\sqrt{2L\epsilon}$.
\end{lem}
\end{proof}

\begin{proof}[Proof of Proposition~\ref{prop:trace_comps}] The estimate $\tr(\gamma I - \matW)^{-1} = \stielt(\gamma) + \smallo $ follows immediately from the following finite sample bound:
\begin{thm}[Specialization of Proposition 6.5 in~\cite{simchowitz2018tight}]\label{thm:stiel_thm} Fix $\delta \in( 0,1)$, let $p = e^{-d^{1/3}}$, and let $z^*:= 23 d^{-1/3}\log^{2/3}(d)$. Fix an $a \in (2 + \frac{1}{31}(z^*-2),d)$, and assume that $\boldeps := (d(a-z^*)^2)^{-1/2}$ satisfies $\boldeps^2 < \min\{\frac{1}{16\sqrt{2}},\frac{a-2}{32}\}$, and $p^{1/3} < \boldeps/8$. Then with probability at least $1 - \delta - p$, 
	\begin{eqnarray*}
	\left|\tr(aI - \matW)^{-1} - \stielt(a)\right| \le  c_{\delta}\boldeps^2 + 8d^{3/2} p^{1/6} , \text{ where } c_{\delta} := 4\sqrt{2} + 2\sqrt{\log(2/\delta)}.
	\end{eqnarray*}
\end{thm}

\noindent For the estimate $\tr(\gamma I - \matW)^{-2} = \quielt(\gamma) + \smallo $, note that for $\lambda \in (1,2]$ and $t \le \frac{\gamma - 2}{2}$ as in Proposition~\ref{prop:stiel_to_quielt}, we have that 
\begin{eqnarray}
[\gamma - t, \gamma + t] \subset [2 + \frac{(\lambda - 1)^2}{2},5]
\end{eqnarray}
Hence, we have that for any $a \in \{\gamma - t, \gamma , \gamma + t\}$, $\tr(\gamma I - \matW)^{-1}  = \stielt(a) + \smallo$. By Proposition~\ref{prop:stiel_to_quielt} and some algebraic manipulations, we see that the equality $(i)$ in 
\begin{align*}
\tr(\gamma I - \matW)^{-2} ~\overset{(i)}{=}~ \quielt(\gamma) + 2\sqrt{2|\smallo| C(\lambda - 1)^{-3}} ~\overset{\mathrm{Fact}~\ref{fact:baro}}{=}~ \quielt(\gamma) + \smallo
\end{align*} 
will follow as soon as we can bound $\Pr[\|\matW\|_{\op} \ge \gamma -t ] \le \exp( - c_0 d^{c_1} (\lambda - 1)^{c_2})$. Since $\gamma - t \ge 2 + {(\lambda - 1)^2}{2}$, it suffices only to show that, for universal constants $c_0,c_1,c_2 > 0$,
\begin{align*}
\Pr\left[|\matW\|_{\op} <  2 + \frac{(\lambda - 1)^2}{\lambda}\right] \ge 1 - \exp( - c_0 d^{c_1} (\lambda - 1)^{c_2})
\end{align*}
The above display is direct consequence of the following proposition:
\begin{prop}[Specialization of Proposition 6.3 in~\cite{simchowitz2018tight}]\label{prop:norm_upper_bound} Let $d \ge 250$, and fix a $p \in (0,1)$. Then,
	$\Pr[\|\matW\|_\op > z^*] \le e^{-d^{1/3}}$, where $z^* = 23 d^{-1/3}\log^{2/3}(d)$.
\end{prop}
\end{proof}

\subsubsection{Proof of Lemma~\ref{lem:quielt_derv}\label{sec:quielt_derv_proof}}
 We see that for all $a \ge 2$
\begin{align*}
 |\quielt'(a)| &~= \left|\frac{\rmd}{\rmd x} \frac{\stielt(x)}{\sqrt{x^2- 4}}\right|\\
&=  \left|\frac{-\quielt(a)}{\sqrt{a^2- 4}} - \frac{a\stielt(a)}{(x^2 - 4)^{3/2}}\right|~=~  \left|\frac{-\stielt(a)}{a^2 - 4} + \frac{2x\stielt(x)}{(x^2 - 4)^{3/2}}\right|\\
&\le  \stielt(a)\cdot \{1+a\} \cdot (\min \{a^2 - 4,1\})^{-3/2}\\
&\le  \stielt(a)\cdot (a+2)^2 \cdot (\min \{a - 2,1\})^{-3/2}\\
&\le (a+2)^2 \cdot (\min \{a - 2,1\})^{-3/2}~,
\end{align*}
where the last line uses that $\stielt(a)$ is decreasing (as $- \rmda \quielt(a) > 0$) for $a \in (0,2]$, so $\stielt(a) \le \stielt(2) = 1$.
In particular, suppose $\lambda \le 2$, so that $\gamma := 2(\lambda + \lambda^{-1}) - 2 \le 3$, and choose $t \le (\gamma - 2)/2$ and $\gamma \le 3$. Then,
\begin{align*}
\max_{a \in [\gamma - t,\gamma + t]} |\rmda \quielt(a)| &\le (\gamma + 2 + t)^2(\min \{\gamma - t - 2 ,1\})^{-3/2} ~\le C\min\left\{1,\frac{\gamma - 2}{2}\right\}^{-3/2} \\
&= C \min\{1,(\lambda - 1)^2/\lambda\}^{-3/2} \le C'(\lambda - 1)^{-3}~,
\end{align*}
where $C,C'$ are universal constants

\subsubsection{Proof of Lemma~\ref{lem:cvx_approximation}\label{lem:cvx_proof}}
Let $t=\sqrt{2\epsilon/L}$. Since $g$ is concave and differentiable on $[x - 2\sqrt{2\epsilon/L}, \infty)$, we have that
\begin{align*}
\frac{g(x) - g(x-t)}{t} \ge g'(x) \ge \frac{g(x+t) - g(x)}{t}
\end{align*}
Moreover, if $f'$ is $L$-Lipschitz on $[x-t,x+t]$, then 
\begin{align*}
f'(x) + tL \ge \frac{f(x) - f(x-t)}{t} ~\text{ and }~ f'(x) - tL \le \frac{f(x +t ) - f(x)}{t}
\end{align*}
Hence, 
\begin{align*}
g'(x) &\le f'(x) + tL + \frac{(f-g)(x) - (f-g)(x-t)}{t} \\
g'(x) &\ge f'(x) - tL + \frac{(f-g)(x + t) - (f-g)(x)}{t} ~.
\end{align*}
Thus, as $\|g(u) - f(u)\| \le \epsilon$ for all $u \in \{x-t,x,x+t\}$, then by the choice of $t = \sqrt{2\epsilon/L}$, we have 
\begin{eqnarray}
f'(x) - 2\sqrt{2L\epsilon} = f'(x) - tL - \frac{2\epsilon}{t} \le g'(x) \le  f'(x) + tL + \frac{2\epsilon}{t} = f'(x) + 2\sqrt{2L\epsilon}~,
\end{eqnarray}
whence $|g'(x) - f'(x)| \le 2\sqrt{2L\epsilon}$.


\section{Estimation Lower Bound: Supplement for Theorem~\ref{thm:est_u_lb}}

\subsection{Verifying~\eqref{eq:asm_recur}\label{sec:asm_rec}}
It is easy to see that $\tau_{k+1}$ can be lower bounded as $\tau_{k+1} \ge \lambda^{5k}\taub$.
\begin{eqnarray*}
\left(\sqrt{d\tau_{k+1}}-\sqrt{2k+2}\right)^2 \ge d\tau_{k+1}/\lambda &\iff& (1 - 1/\lambda)d\tau_{k+1} - \sqrt{2(k+1)d\tau_{k+1}} + 2k+2 \ge 0\\
&\impliedby& (1 - 1/\lambda)d\tau_{k+1} - \sqrt{2(k+1)d\tau_{k+1}} \ge 0\\
&\impliedby& \sqrt{d\tau_{k+1}} \ge \sqrt{2(k+1)}/(1 - 1/\lambda) \ge 0\\
&\impliedby& d\tau_k \ge \lambda^2(2k+1)/(\lambda - 1)^2 \ge 0\\
&\impliedby& d\lambda^{5k}\taub \ge \lambda^2(2k+1)/(\lambda - 1)^2 \ge 0\\
&\impliedby& d\taub \ge \frac{2\lambda^2}{(\lambda - 1)^2} \cdot \max_{k\ge 0} \lambda^{-5k}(k+1)\ge 0~.
\end{eqnarray*}

Now, we can compute that $\max_{k\ge 0} \lambda^{-5k}(k+1) = \max_{k\ge 0}\exp( -5k \log \lambda + \log (k+1))$. The function $x \mapsto -5x\log \lambda + \log(1+x)$ is concave, and maximized at $x_*$ when $1/(1+x_*) = 5\log \lambda$, that is, $x_* = \frac{1}{5\log \lambda} - 1$, when $\frac{1}{5\log \lambda} - 1 \ge 0$, and at $x_* = 0$ otherwise. If the maximum is at $x_* = 0$, it suffices to take $d\taub \ge \frac{2\lambda^2}{(\lambda - 1)^2}$. Otherwise, we still have $x_* \ge 0$, and hence
\begin{align*}
\max_{k\ge 0} \lambda^{-5k}(k+1) \le \exp( -5x_*\log \lambda + \log (1+x_*)) \le \exp( \log (1+x_*)) = 1+x_* = \frac{1}{5\log \lambda}.
\end{align*} 
Thus, it suffices that
\begin{align*}
\taub \ge \frac{\lambda^2}{2d(\lambda - 1)^2 \min\{1,5\log \lambda\}}~.
\end{align*}
Moreover, we have that bound that, for $\lambda \in (1,2]$, $5\log \lambda > \log \lambda \ge \frac{\lambda - 1}{2}$ as well as $1 \ge \frac{\lambda - 1}{2}$, so in fact, its enough to take $\taub \ge \frac{\lambda^2}{d(\lambda - 1)^3 }~.$

\subsection{Proof of Proposition~\ref{prop:recur}\label{sec:prop_recur}}
To begin, we can assume without loss of generality that $\Alg$ is deterministic. We let $\itZ_k := \{\bnot,\vone,\wone,\dots,\vk,\wk\}$ denote the information collected by $\Alg$ up to round $k$. Moreover, we let $\Prit_{u}$ denote the distribution of $\itZ_k$ given $\matu = u$.

We start by stating the following analogue of the data-processing inequality~\citet[Proposition 3.2]{simchowitz2018tight}:
\begin{prop}\label{prop:info_th_rkone}  Then for any $\tau_{k} \le \tau_{k+1}$, $\epsilon > 0$, and $\eta > 0$,
\begin{align*}
&\Exp_{\matu \sim \calN(0,\frac{1}{d})}\Prit_{\matu}\left[\{ \Phi(\itV_k;\matu) \le \tau_k \cap \eventbound(\epsilon)\} \cap \{ \Phi(\itV_{k+1};\matu)> \tau_{k+1}\} \right] \le \\
&\Bigg(\Exp_{ \matu \sim \calN(0,\frac{1}{d})}\Exp_{\itZ_{k} \sim \Prit_0}\left[\left(\frac{\rmd\Prit_{\matu}(\itZ_{k})}{\rmd \Prit_0(\itZ_{k})}\right)^{1+\eta} \I(\{ \Phi(\itV_k;\matu) \le \tau_k \cap \eventbound(\epsilon)\})\right] \\
&\quad \cdot \sup_{V \in \calO(d,k+1)} \Pr_{\matu \sim \calN(0,\frac{1}{d})}[\Phi(V;\matu)> \tau_{k+1}]^\eta   \Bigg)^{\frac{1}{1+\eta}}
\end{align*}
\end{prop}
The proof of the above proposition is essentially identical to that of~\citet[Proposition 3.2]{simchowitz2018tight}, and is ommitted for the sake of brevity. The main difference is that we modify the distribution of $\matu$ (which does not alter the proof), and that that we replace $\I(\{ \Phi(\itV_k;\matu) \le \tau_k\})$ with $\I(\{ \Phi(\itV_k;\matu) \le \tau_k \cap \eventbound(\epsilon)\})$, thereby restricting to the event $\eventbound(\epsilon)$. 

Proposition~\ref{prop:info_th_rkone} recursively controls the probability that the $\Phi(\itV_{k+1},\matu)$ is above the threshold $\tau_{k+1}$, on the ``good event'' that $\Phi(\itV_{k},\matu) \le \tau_k$, in terms of two quantities: (a) an information-theoretic term that depends on the likelihood ratios and (b) a ``best-guess'' probability which upper bounds the largest vallue of $\Phi(\itV_{k+1},\matu)$ if $\itV_{k+1}$ were selected only according to the prior on $\matu$, \emph{without any posterior knowledge} of $\itZ_k$.

The best-guess probability is bounded with the following lemma, which is the Gaussian analogue to~\cite{simchowitz2018tight}:
\begin{lem}\label{lem:small_ball_prob} For any $V \in \calO(d,k+1)$ and $d\tau_{k+1} \ge \sqrt{2(k+1)}$, we have 
\begin{eqnarray}
\Pr_{\matu \sim \calN(0,I/d)}[ \matu^\top V^\top V \matu \ge \tau_{k+1}] \le \exp\left\{-\frac{1}{2}\left(\sqrt{d\tau_{k+1}}-\sqrt{2(k+1)}\right)^2\right\}
\end{eqnarray}
\end{lem}
\begin{proof}  Let  $\matZ = \matu^\top V^\top V \matu $. Then $d\matZ$ is $\chi^2$ random variable of degree $k+1$, Lemma 1 in~\cite{laurent2000adaptive} implies that 
\begin{align*}
\Pr[d\matZ \ge (k+1) + 2\sqrt{u (k+1)} + 2u] \le \exp(-u).
\end{align*}
For $u \ge 0$, this implies the cruder bound $\Pr[dX \ge 2( (k+1) + 2\sqrt{u (k+1)} + u)] = \Pr[d\matZ \ge (\sqrt{2(k+1)} + \sqrt{2u})^2]\le \exp(-u)$. Setting $u = \frac{1}{2}\left(\sqrt{d\tau_{k+1}}-\sqrt{2(k+1)}\right)^2$, we can verify that 
\begin{align*}
(\sqrt{2(k+1)} + \sqrt{2u})^2 = (\sqrt{2(k+1)} + \sqrt{d\tau_{k+1}} - \sqrt{2(k+1)})^2 = d\tau_{k+1}.
\end{align*}
Thus, $\Pr[\matu^\top V^\top V \matu  \ge \tau_{k+1}] = \Pr[d\matZ \ge d\tau_{k+1}]\le \exp( - \frac{1}{2}\left(\sqrt{d\tau_{k+1}}-\sqrt{2(k+1)}\right)^2)$, as needed. 
\end{proof}

The likelihood term is a bit more effort to control. The following bound mirrors Proposition 3.4 in~\cite{simchowitz2018tight}, but with the additional subtlety of taking the dependence on $\taub$ into account; the proof is in Section~\ref{sec:likelihood_proof}.
\begin{prop}\label{prop:likelihood_info}
For any $\tau_k \ge 0$ and any $u \in \R^d$ with $\|u\|_2^2 \le 1+\epsilon$, we have
\begin{eqnarray}\label{Chi_Plus_1_Eq}
\Exp_{\Prit_0}\left[\left(\frac{\rmd\Prit_u(\itZ_k)}{\rmd\Prit_0(\itZ_k)}\I(\eventbound(\epsilon) \cap \Phi(\itV_k;u)) \le \tau_k\right)^{1+\eta} \right] \le \exp\left(\frac{(1+\epsilon)\eta(1 + \eta)}{2}\boldlam^2 (\tau_k+ \taub)\right)
\end{eqnarray}
\end{prop}

Putting the pieces together, we have
\begin{eqnarray*}
&&\Exp_{u \sim \calD}\Prit_{u}\left[\{ \Phi(\itV_k;u) \le \tau_k\} \cap \eventbound(\epsilon) \cap \{ \Phi(\itV_{k+1};u)> \tau_{k+1} \cap \eventbound(\epsilon)\} \right] \le \\
&&\left(\exp\left(\frac{\eta(1 + \eta)}{2}\boldlam^2 (\tau_k+ \taub)\right) \cdot \exp\left\{-\frac{\eta}{2}\left(\sqrt{d\tau_{k+1}}-\sqrt{2(k+1)}\right)^2\right\}  \right)^{\frac{1}{1+\eta}} = \\
&& \exp\left(\frac{\eta}{2(1+\eta)} \left(\left((1+\epsilon)\boldlam^2 (1+\eta)(\tau_k+ \taub)\right) -\left(\sqrt{d\tau_{k+1}}-\sqrt{2(k+1)}\right)^2\right)\right)~.
\end{eqnarray*}
Choosing $\eta = \lambda - 1$ concludes the proof of Proposition~\ref{prop:recur}

\subsection{Proof of Proposition~\ref{prop:likelihood_info}\label{sec:likelihood_proof}}
The proof of Proposition~\ref{prop:likelihood_info} mirrors the proof of Proposition 3.4 in~\cite{simchowitz2018tight}, with minor modifications to take into account the additional side information $\matb$. The next subsection first collects necessarily preliminary results, and the second concludes the proof.

\subsubsection{Preliminary Results for Proposition~\ref{prop:likelihood_info}}
We need to start by describing the likelihood ratios associated with the algorithm history $\itZ_k$:
	\begin{lem}[Conditional Likelihoods] \label{ConditionalLemma} Let $\itP_{i} := I - \itV_i\itV_i^\top$ denote the orthogonal projection onto the orthogonal complement of $\mathrm{span}(\vone,\dots,\vi)$. Under $\Prit_{u}$ ( the joint law of $\matM,\bnot$ and $\itZ_T$ on $\{\matu = u\}$), we have
	\begin{multline}
	(\itP_{i-1})\matM\vi \big{|} \itZ_{i-1},\matu = u \sim \mathcal{N}\left(\boldlam (u^\top \vi) \itP_{i-1} u,\frac{1}{d}\itSigma_i \right)\\
	  \text{where} \quad \itSigma_{i} := \itP_{i-1}\left(I_d+ \vi\viT \right)\itP_{i-1}.
	\end{multline}
	In particular, $\wi$ is conditionally independent of $\wzero,\wone,\dots,\wiminus$ given $\vone,\dots,\viminus$ and $\matu = u$.
	\end{lem}
	\begin{proof} The lemma was proven in Lemma 2.4~\cite{simchowitz2018tight} in the case where there was no initial side-information $\wzero$. 
	When there is side information, we just need to argue that $\itZ_i | \itZ_{i-1}$ is independent of $\wzero$, conditioned on $\matu$. 
	Since $\Alg$ is deterministic by assumption, $\vi$ is a measurable funciton of $\itZ$, and thus $\itZ | \itZ_{i-1}$ is a measurable function of $\wi = (\lambda \matu \matu^\top + \matW)\vi$. 
	Hence, conditioned on $\itZ_{i-1}$ and $\matu$, $\wi$ is measurable function of $\matW$, which is independent of $\wzero$. 
	\end{proof}

The next proposition is copied verbatim from Proposition 3.5 in~\cite{simchowitz2018tight}, with the exceptions that the indices $i$ are allowed to range from $0$ to $k$ (rather than $1$ to $k$) to account for an initial round of side information. It's proof is identical: 
\begin{prop}[Generic Upper Bound on Likelihood Ratios] \label{Generic_UB_LL}Fix an $u,s\in \calS^{d-1}$, and fix $r_u, r_s,r_0 \ge 0$. For $i \ge 0$ and $\itVtil_i \in \calO(d,i)$, define the likelihood.
	\begin{eqnarray}\label{g_func_def}
	g_i(\itVtil_i) &:=& \Exp_{\Prit_0}\left[\frac{\rmd\Prit_u(\itZ_i |\itZ_{i-1})^{r_u}\rmd\Prit_s(\itZ_i |\itZ_{i-1})^{r_s}}{\rmd\Prit_0(\itZ_i | \itZ_{i-1})^{r_0}} \big{|} \itV_i = \itVtil_i \right].
	\end{eqnarray}
	Then for any $\calV_k \subset \calO(d,k)$, we have
	\begin{eqnarray}\label{product_eq}
	\Exp_{\Prit_0}\left[\frac{\rmd\Prit_u(\itZ_k)^{r_u}\rmd\Prit_s(\itZ_k)^{r_s}\I(\itV_k \in \calV_k)}{ \rmd \Prit_0(\itZ_k)^{r_0} } \right] &\le& \sup_{\itVtil_{k} \in \calV_k}\prod_{i=0}^kg_i(\itVtil_{1:i})~,
	\end{eqnarray}
	where $\itVtil_{1:i}$ denotes the first $i$ columns of $\itVtil_k$.
	\end{prop}
	Lastly, we recall the following elemntary computation, stated as Lemma 3.6 in~\cite{simchowitz2018tight}:
	\begin{lem}\label{lem:power_divergence_comp} Let $\Pr$ denote the distribution $\calN(\mu_1,\Sigma)$ and $\Q$ denote $\calN(\mu_2,\Sigma)$, where $\mu_1,\mu_2 \in (\ker \Sigma)^{\perp}$. Then 
	\begin{eqnarray}
	\Exp_{\Q}\left[\left(\frac{\rmd \Pr}{\rmd \Q}\right)^{1+\eta}\right] =  \exp\left(\frac{\eta(1 + \eta)}{2}(\mu_1-\mu_2)^{\top} \Sigma^{\dagger} (\mu_1-\mu_2)\right)
	\end{eqnarray}
	\end{lem}

	\subsubsection{Concluding the proof of Proposition~\ref{prop:likelihood_info}}
	Fix a $u \in \calS^{d-1}$, and we shall and apply Proposition~\ref{Generic_UB_LL} with $r_u = r_0 = 1+\eta$ and $r_s = 0$. In the language of Proposition~\ref{Generic_UB_LL} , we have
	\begin{eqnarray*}
	g_i(\itV_i ) &=& \Exp_{\Prit_0}\left[\left(\frac{\rmd\Prit_u(\itZ_i |\itZ_{i-1})}{\rmd\Prit_0(\itZ_i | \itZ_{i-1})}\right)^{1+\eta} \big{|} \itV_i  \right] \\
	&=& \Exp_{\Prit_0}\left[\left(\frac{\rmd\Prit_u(\wi |\itZ_{i-1})}{\rmd\Prit_0(\wi | \itZ_{i-1})}\right)^{1+\eta} \big{|} \itV_i  \right] 
	\end{eqnarray*}
	Now, observe that, $\rmd\Prit_u(\wi |\itZ_{i-1})$ is the density of $\mathcal{N}(\boldlam \langle u, \vi \rangle \cdot \itP_{i-1}u, \frac{1}{d}\itSigma_i)$ and $\rmd\Prit_0(\wi |\itZ_{i-1})$ is the density of $\mathcal{N}(0, \frac{1}{d}\itSigma_i)$. Since $\itSigma_i = \itP_{i-1}\left(I_d+ \vi\viT \right)\itP_{i-1}$, we have $\itP_{i-1}\itSigma_i^{\dagger}\itP_{i-1} = \itP_{i-1} \preceq I$. Thus,
	\begin{eqnarray}\label{eq:itSigma_eq}
	u^{\top}\itP_{i-1}(\itSigma_i/d)^{\dagger}\itP_{i-1}u \le d\|u\|^2\le (1+\epsilon) d\quad \forall u: \|u\|_2^2 \le 1+\epsilon.
	\end{eqnarray}
	Hence, by Lemma~\ref{lem:power_divergence_comp}, we have for all $i \in [k]$ that
	\begin{eqnarray*}
	g_i(\itV_i )  &\overset{\text{Lemma~\ref{lem:power_divergence_comp}} }{=}& \exp\left( \frac{\eta(1+\eta)\boldlam^2\langle u, \vi \rangle^2}{2}u^{\top}\itP_{i-1}(\itSigma_i/d)^{\dagger}\itP_{i-1}u\right)\\
	&\overset{\text{Eq.~\eqref{eq:itSigma_eq}}}{\le}& \exp\left( \frac{\eta(1+\eta)\boldlam^2\cdot (1+\epsilon) d\langle u, \vi \rangle^2}{2}\right)
	\end{eqnarray*}
	For $i = 0$, we have that $\matw \sim \mathcal{N}(\sqrt{\taub} \matu,I/d)$. Thus, 
	\begin{eqnarray}
	g_0(\{\}) &\overset{\text{Lemma~\ref{lem:power_divergence_comp}} }{=}& \exp\left(\frac{\eta(1+\eta)}{2}(\sqrt{\taub} u)^{\top} (I/d)^{-1}(\sqrt{\taub} u)\right) = \exp\left(\frac{d\eta(1+\eta)\taub}{2}\right) 
	\end{eqnarray}
	Hence, if $\calV_k := \{\itVtil_k \in \calO(d;k): \Phi(\itVtil_k;u) \le \tau_k\}$, then Proposition~\ref{Generic_UB_LL} implies
	\begin{align*}
	&\Exp_{\Prit_0}\left[\left(\frac{\rmd\Prit_u(Z_k}{ \rmd \Prit_0(Z_k) }\right)^{1+\eta}I(\itV_k \in \calV_k) \right] \\
	&\le \exp\left(\frac{d\eta(1+\eta)\taub}{2}\right) \cdot \sup_{\itVtil_k \in \calV_k}\prod_{i=1}^k \exp( \frac{\eta(1+\eta)\boldlam^2\cdot d(1+\epsilon)\langle u, \itVtil_k[i] \rangle^2}{2})\\
	 &= \exp\left(\frac{d\eta(1+\eta)\taub}{2}\right) \sup_{\itVtil_k \in \calV_k} \exp( \frac{d(1+\epsilon)\eta(1+\eta)\boldlam^2\Phi(\itVtil_k;u)}{2})\\
	&\le  \exp\left(\frac{d\eta(1+\eta)\taub}{2}\right) \exp( \frac{d(1+\epsilon)\eta(1+\eta)\boldlam^2 \tau_k}{2})\\
	&\le \exp( \frac{d(1+\epsilon)\eta(1+\eta)\boldlam^2 (\tau_k+\taub)}{2}) \quad \text{ since } \boldlam \ge 1~.
	\end{align*}

\section{Appendix for Proof of Theorem~\ref{thm:asmp_ovlap_thm}}


\subsection{Proof of Lemma~\ref{lem:ovlap_to_MMSE}\label{sec:lem_ovlap_to_MMSE}}

To turn an upper bound on $\MMSE$ into a lower bound into inner product upper bounds, observe that for $(\matM,\bnot)$-measurable $\xhat$ of the form $\xhat = \|\xhat\|\uhat$ and $\uhat \in \calS^{d-1}$, one has (conditioning on $\matM$ and $\bnot$)
\begin{eqnarray*}
\MMSE(\matu\matu^\top|\matM,\bnot) &\le& \Exp_{\matu}\left[ \|\matu \matu^\top -  \xhat\xhat^{\top}\|_{F}^2 \big{|}\matM,\bnot\right] \\
&=& \Exp_{\matu}\left[ \|\matu\|_{2}^4 \big{|} \matM,\bnot\right]  - 2\Exp_{\matu}[\langle \xhat, \matu \rangle^2 \big{|} \matM,\bnot] + \Exp[\|\xhat\|_2^4] \\
&=& \Exp_{\matu}\left[ \|\matu\|_{2}^4 \big{|} \matM,\bnot\right]  - 2\|\xhat\|^2 \Exp_{\matu}[\langle \uhat, \matu \rangle^2 \big{|} \matM,\bnot] + \|\xhat\|^4~ 
\end{eqnarray*}
In particular, setting 
\begin{align*}
\|\xhat\|^2 := \sqrt{\Exp[\|\matutil\|_2^4 \big{|} \matM,\bnot] - \MMSE(\matu\matu^\top|\matM,\bnot)}~,
\end{align*}
we have 
\begin{align}\label{eq:conditional}
\Exp_{\matutil}[\langle \uhat, \matutil \rangle^2 \big{|} \matM,\bnot] &\le \sqrt{\Exp[\|\matutil\|_2^4 \big{|} \matM,\bnot] - \MMSE(\matu\matu^\top|\matM,\bnot)}~.
\end{align} 
Hence, we can bound
\begin{eqnarray*}
\Exp_{\matM,\matb}\Exp[\langle \uhat, \matutil \rangle^2 \big{|} \matM,\bnot] &\overset{\eqref{eq:conditional}}{\le}& \Exp_{\matM,\matb}\sqrt{\Exp[\|\matutil\|_2^4 \big{|} \matM,\bnot] - \MMSE(\matu\matu^\top|\matM,\bnot)}\\
&\overset{(i)}{\le}& \sqrt{\Exp_{\matM,\matb}\left[\Exp[\|\matutil\|_2^4 \big{|} \matM,\bnot] - \MMSE(\matu\matu^\top|\matM,\bnot) \right]}\\
&\overset{(ii)}{\le}& \sqrt{\Exp\|\matu\|_2^4 - \Exp_{\matM,\matb}[\MMSE(\matu\matu^\top|\matM,\bnot) ]} \\
&=& \sqrt{[\Exp\|\matu\|_2^2]^2 + \Var[\|\matu\|_2^2] - \Exp_{\matM,\matb}[\MMSE(\matu\matu^\top|\matM,\bnot) ]} \\
&\overset{(iii)}{\le}& \sqrt{\Exp[\|\matu\|_2]^2 - \Exp_{\matM,\matb}[\MMSE(\matu\matu^\top|\matM,\bnot) ]} + \sqrt{\Var[\|\matu\|_2^2]} \\
&\overset{(iv)}{\le}& \sqrt{\Exp[\|\matu\|_2]^2 - \Exp_{\matM,\matb}[\MMSE(\matu\matu^\top|\matM,\bnot) ]} + c_1d^{-c_2}~,
\end{eqnarray*}
where $(i)$ and $(ii)$ are Cauchy Schwartz, $(iii)$ is the inequality $\sqrt{a+b} \le \sqrt{a} + \sqrt{b}$ for $a,b \ge 0$, and $(iv)$ uses standard Guassian moment bounds to bound $\Var[\|\matu\|_2^2]$.

\subsection{Proof of Proposition~\ref{prop:ovlap_conversion}\label{sec:prop_ovlap_conversion_proof}}

	Recall that by Equation~\ref{eq:cross_bound}, we have that 
	\begin{align*}
	\ovlapd(\taub) \le \sqrt{\Exp_{\bnot}[ \cross(\matu\matu^\top\mid\matb)} + c_1d^{-c_2}~
	\end{align*}
	 We now define 
	\begin{eqnarray}
	\boldalpha := \frac{\|\bnot\|}{\sqrt{1+\taub}} \quad \text{and} \quad \Ealpha := \{ 1 - d^{-1/4} \le \boldalpha \le 1 + d^{-1/4}\}~,
	\end{eqnarray}
	The following lemma characterizes the distribution of $\matutil $ 
	 \begin{lem}\label{condition_dist_lem} Conditioned on $\bnot$, $\matu$ has the distribution
	$\matutil \sim \mathcal{N}\left(\frac{\sqrt{\taub} \bnot}{1+\taub} ,\frac{1}{1+\taub} \cdot \frac{I}{d}\right)$. 
	\end{lem}
	In particular, $\Exp[\|\matu\|_2^2 \big{|}\bnot] = \frac{1 + \boldalpha \taub}{1+\taub}$. Hence, by standard Gaussian concentration, we can truncate
	\begin{align*}
	\ovlapd(\taub)  \le \sqrt{\Exp_{\bnot}[ \I(|\boldalpha - 1| \le d^{-1/4}) \cross(\matu\matu^\top\mid\matb)} + c_1d^{-c_2}~,
	\end{align*}
	for possibly different constants $c_1,c_2$.

	Next observe that conditioned on any $\bnot$, the term $\cross(\matu\matu^\top\mid\matb)$ and the noise $\matW$ is invariant to orthogonal change of basis; hence, we may assume without loss of generality that $\bnot$ is aligned with the ones unit vector $\mathbf{1}/\sqrt{d}$. Moreover, precisely, we may assume without loss of generality that $\bnot/\sqrt{1+\taub} = \boldalpha \mathbf{1}/\sqrt{d}$, in which case 
	\begin{align*}
	\matu \overset{d}{=} \frac{1}{\sqrt{1+\taub}}\matubar(\boldalpha), \text{ where } \uch_i(\alpha) \sim  \mathcal{N}\left( \alpha \sqrt{\taub/d}, 1/d\right)~.
	\end{align*}
	Therefore, setting $\rho_{\taub} = (\frac{\lambda}{1+\taub})^2$ and $\mu_{\taub} = \sqrt{\taub}$, 
	\begin{align*}
	\cross(\matu\matu^\top\mid\matb) &:=  \Exp_{\matu,\matM}\left[  \|\Exp[\matu\matu^{\top} | \matM,\matb] \|_{\F}^2  \right]\\ &\overset{d}{=} \frac{1}{(1+\taub)^2}\Exp_{\uch(\boldalpha),\Mch}\left[  \|\Exp[\uch(\boldalpha)\uch(\boldalpha)^{\top} | \Mch] \|_{\F}^2  \right], \text{ where } \Mch = \matW = \frac{\lambda}{1+\taub}\uch(\boldalpha)\uch(\boldalpha)^\top\\
	&\overset{d}{=} \frac{1}{(1+\taub)^2}\crossch_d(\rho_{\taub};(\boldalpha)\mu_{\taub}).
	\end{align*}
	Therefore, we arrive at the desired bound:
	\begin{align*}
	\ovlapd(\taub)  \le \frac{\sqrt{\Exp_{\bnot}[ \I(|\boldalpha - 1| \le d^{-1/4}) \crossch_d(\rho_{\taub};\boldalpha\mu_{\tau_0})}}{1+\taub} + c_1d^{-c_2}~,
	\end{align*}

\subsection{Proof of Lemma~\ref{condition_dist_lem}\label{sec:conditional_dist_lem_proof}}
We observe that the posterior distribution of $\matu | \bnot$ is equivalent to the posterior distribution of $\matu | \bnot/\sqrt{\taub}$, which is 
\begin{eqnarray*}
\matu | \bnot &\sim& \mathcal{N}\left(\left((I/d)^{-1} + \left(\frac{I}{d\taub}\right)^{-1}\right)^{-1}\left(\frac{I}{d\taub}\right)^{-1}\frac{\bnot}{\sqrt{\taub}}, \left((I/d)^{-1} + \left(\frac{I}{d\taub}\right)^{-1}\right)^{-1}\right) \\
&=& \mathcal{N}\left(\frac{\sqrt{\taub} \bnot}{1+\taub} ,\frac{1}{d(1+\taub)}\right) \\
\end{eqnarray*}

\newpage

\subsection{Proof of Theorem~\ref{thm:main_asymptotic_comp}\label{sec:MSE_Limit_Proof}}

To ensure consistency with the results from~\cite{lelarge2016fundamental}, we shall begin with a reparametrization of $\cross$. We shall begin by parameterizing quantities in terms of arbitary scalar distribution $\calD$; as above, we will use a real scalar $\mu \in \R$ as the $\calD$-argument to denote the setting where $\calD = \calN(1,\mu)$.I
\begin{defn}[Full-Observation Model]\label{def: full_observation_model}
	Given $d \ge 2$, $\rho > 0$, and a distribution $\calD$ on $\R$ with finite fourth moment, let $\Ach_{ij} = 1 - \frac{1}{2}\I(i \ne j)$. 
	We define $\Psfch(d,\rho,\calD)$ as the law of $(\matXch,\Ych,\Zch)$, where 
	\begin{align}
		\forall 1 \le i,j \le d, \quad \Ych_{ij} = \sqrt{\frac{\rho\Ach_{ij}}{d}}\matXch_i\matXch_j + \Zch_{ij},
	\end{align}
	where $\matXch_i \iidsim \calD$, and where $\Zch_{ij} \iidsim \calN(0,1)$ for $1 \le i \le j \le d$, with $\Zch_{ji} = \matZbar_{ij}$ for $j \ge i$. 
	We define define the associated cross term. 
	\begin{align*}
	\crossch_{d}(\rho;\calD) := \frac{1}{d^2}\sum_{i,j =1}^d \Exp\left[\Exp[ \matXch_i\matXch_j | \Ych]^2 \right],
	\end{align*}
	where the expectation is taken with respect to $\Psfch_d(\rho;\calD)$. 
\end{defn}

We verify that the definition of $\crossch_{d}(\rho;\calD)$ is consistent with the definition given in Proposition~\ref{prop:ovlap_conversion}, when $\calD$ is taken to be $\calN(\mu,1)$:
\begin{lem}\label{lem:cross_eq} For all $\rho$ and $\mu \in \R$, $\crossch_{d}(\rho;\calN(\mu,1)) = \crossch_{d}(\sqrt{\rho};\mu)$. 
\end{lem}
\begin{proof} Let $\calD = \cal(\mu,1)$, and consider the marginal $(\matXch,\matY)$ under $\Psfch_d(\rho,\calD)$, and the marginal distribution of $(\matM, \matu)$, where $\uch_i \sim \frac{1}{\sqrt{d}}\calD$ and $\Mch = \matW + \sqrt{\rho} \uch\uch^\top $, $\matW \sim \GOE(d)$. Then, we can see that  $(\Ych,\matXch) \overset{d}{=} (\sqrt{d}\Ach^{-1/2}\matM,\sqrt{d}\matu)$. Hence, $\sqrt{d}\matu \mid  \matM \overset{d}{=} \sqrt{d}\matu \mid  \sqrt{d}\Ach^{-1/2}\matM \overset{d}{=} \matXch | \matY$. Writing out the definitions of $\crossch_{d}(\rho;\calD)$ and $\cross_{d}(\rho;\mu)$ concludes.
\end{proof}

Recall that our goal is to control a term of the form 
\begin{align*}
	\Exp_{\boldalpha}\left[\I(|\boldalpha - 1|\le d^{-1/4}) \crossch_{d}(\rho_{\taub};\mu\boldalpha)\right],
\end{align*}
where $\rho_{\taub} = (\lambda/(1+\taub))^2 \ge 1$, and $\mu = \sqrt{\taub} \in (0,1)$.

In order to directly use the bound from \cite{lelarge2016fundamental}, we shall need to show that the above expression can be approximate by a related quantity, depending on only off-diagonal measurements:
\begin{defn}[The Off-Diagonal Model]\label{defn:off_diagonal_model}
	Given $d \ge 2$, $\rho > 0$, and a distribution $\calD$ on $\R$ with finite fourth moment, and  define $\Psfoff_d(\rho;\calD)$ as the law of $(\matX,\matY,\matZ)$, where 
	\begin{align}
		\forall 1 \le i < j \le d, \quad \matY_{ij} = \sqrt{\frac{\rho}{d}}\matX_i\matX_j + \matZ_{ij},
	\end{align}
	where $\matX_i \iidsim \calD$, and where $\matZ_{ij} \iidsim \calN(0,1)$ for $1 \le i < j \le d$. We define
	\begin{align*}
		\crossoff_{d}(\rho;\calD) := \frac{2}{d^2}\sum_{1 \le i < j \le d}\Exp\left[\Exp[ \matX_i\matX_j | (\matY_{ij})_{1 \le i < j \le d}]^2 \right],
	\end{align*}
	where the expectation is taken with respect to the law $\Psfoff$. When $\calD = \calN(\mu,1)$, 
	we will overload notation and write $\crossoff_{d}(\rho;\mu)$ and $\Psfoff_d(\rho;\mu)$.  
\end{defn}

We see that the full observation model of Definition~\ref{def: full_observation_model} and the off-diagonal model of Definition~\ref{defn:off_diagonal_model} differ in two respects: in the full obseration model, one is allowed to see all entries of $\matY$, or equivalently, due to the symetry of $\Zch$, the entries $\matY_{ij}$ for which $1 \le i \le j \le d$. Moreover, the cross term $\crossch_d(\rho;\mu)$ is defined as an average over all entries $i,j$. On the other hand, in the off-diagonal model, the leaner only observes the above-diagonal entries $\matY_{ij}$ for $i < j$, and $\crossoff_d(\rho;\mu)$ depends only on these entries. 

To analyze compare terms $\crossch_d$ and $\crossoff_d$ and analyze their asymtotics, we shall establish that these quantities are proportional to the derivatives of convex functions, called the free energies, defined below:
\begin{defn}[Hamiltonians and Free Energy]\label{def:free_energy} 
	Given $(\matXch,\Ych,\Zch)\sim \Psfch_d(\rho;\calD)$, define the full-obervation Hamiltonian $\Hch_d: \R^d \to \R$ as the random function
	\begin{align*}
	\Hch_{d,\rho}(X) := \sum_{i \le j} \sqrt{\frac{\rho \Ach_{ij}}{d}}X_i\Zch_{ij} + \frac{\Ach_{ij}\rho}{d}\matXch_i\matXch_jX_iX_j - \frac{\Ach_{ij}\rho}{2d}(X_iX_j)^2,
	\end{align*}
	and similarly, for given $(\matX,\matY,\matZ)\sim \Psfoff_d(\rho;\calD)$, define the off-diagonall Hamiltonian $\Hoff_d: \R^d \to \R$ as the random function
	\begin{align*}
	\Hoff_{d,\rho}(X) := \sum_{i < j } \sqrt{\frac{\rho}{d}}X_i\matZ_{ij} + \frac{\rho}{d}\matX_i\matX_jX_iX_j - \frac{\rho}{2d}(X_iX_j)^2.
	\end{align*}
	We define the associated \emph{free energies}
	\begin{align*}
	\Freech_d(\rho;\calD) := \frac{1}{d}\Exp\left[\log \left(\Exp_{X_i \iidsim \calD} e^{\Hch_{d,\rho}(X)}\right)\right]\quad \text{ and }\quad F_d(\rho;\calD) := \frac{1}{d}\Exp\left[\log \left(\Exp_{X_i \iidsim \calD} e^{\Hoff_{d,\rho}(X)}\right)\right],
	\end{align*}
	where the expectations are taken with respect to $\Psfch_d(\rho;\calD)$ and $\Psfoff_d(\rho;\calD)$, respectively. When $\calD = \calN(\mu,1)$, we will abuse notation and write $\Freech_d(\rho;\mu)$ and $F_d(\rho;\mu)$. Lastly, we let $\Freech_d'(\rho;\calD) := \frac{\partial}{\partial \rho}\Freech_d(\rho;\calD)$, and similary for $F_d$ and $F_d'$.
\end{defn}
As show in~\cite{lelarge2016fundamental}, the off-diagonal free energies is closely related the mutual information between $\matX$ and $\matY$, via the equality
\begin{align*}
\info(\matX,\matY) = \frac{\rho(d-1)}{4d^2}\Exp_{X_0 \sim \calD}[X_0^2]^2 - \Freech_{d}(\rho;\calD).
\end{align*}
More importantly, for our purposes, the derivatives of the free energies directly correspond to the $\cross$-terms. We make this precise in the following lemma:
\begin{lem}[Correspondence of Free Energy and '$\cross$']\label{lem:free_cross} Recall the notation $\Freech_d'(\rho;\calD) := \frac{\partial}{\partial \rho}\Freech_d(\rho;\calD)$ and $F_d'(\rho;\calD) := \frac{\partial}{\partial \rho}F_d(\rho;\calD)$. Then, the free energies and cross terms are related as follows:
\begin{align*}
\Freech_d'(\rho;\calD) = \frac{1}{4}\crossch_d(\rho;\calD) \quad \text{and} \quad F_d'(\rho;\calD) = \frac{1}{4}\crossoff_d(\rho;\calD)
\end{align*}
Moreover, $\Freech_d(\rho;\calD)$ and $F_d(\rho;\calD)$ are convex in $\rho$. 
\end{lem}
\begin{proof}
The equality $F_d'(\rho;\calD) = \frac{1}{4}\crossoff_d(\rho;\calD)$ and convexity is established in the proof of \citet[Corollary 17]{lelarge2016fundamental}. By the same argument, one can verify that
\begin{align*}
\Freech_d'(\rho;\calD) = \frac{1}{2d^2}\sum_{1 \le i \le j \le d}\Ach_{ij}\Exp[\Exp[\matXch_i\matXch_j \mid \Ych]^2].
\end{align*}
Since $\Ach_{ij} = 1 - \frac{1}{2}\I(i \ne j)$, we see that 
\begin{align*}
\frac{1}{2d^2}\sum_{1 \le i \le j \le d}\Ach_{ij}\Exp[\Exp[\matXch_i\matXch_j \mid \Ych]^2] = \frac{1}{4d^2} \sum_{1 \le i, j \le d} \Exp[\Exp[\matXch_i\matXch_j \mid \Ych]^2] := \frac{1}{4}\crossoff_d(\rho;\calD).
\end{align*} 
To see that $\Freech_d(\rho;\calD)$ is convex in $\rho$, it suffices to check that $\crossch_d(\rho;\calD)$ is non-decreasing in $\rho$. Observe that 
\begin{align*}
\MMSEch(\rho;\calD) =  \frac{1}{d^2} \sum_{1 \le i, j \le d} \Exp[\Exp[(\matXch_i\matXch_j)^2] - \crossoff_d(\rho;\calD)
\end{align*}
corresponds to the scaled $\rmMMSE$ of $\matXch\matXch^\top$ given $\matY$, which is non-increasing in $\rho$ due a standard fact about Gaussian Channels (see, e.g. \cite{guo2005mutual}). Since the term $\frac{1}{d^2} \sum_{1 \le i, j \le d} \Exp[\Exp[(\matXch_i\matXch_j)^2]$ does not depend on $\rho$,  $\crossoff_d(\rho;\calD)$ must be non-decreasing, as needed.
\end{proof}

We now cite the main result of Lelarge and Miolane, which holds for the off-diagonal model:
\begin{thm}[Restatement of Theorem 1 and Proposition 15 in~\cite{lelarge2016fundamental} ]\label{thm:LelargeMain}
	Let $\rho  > 0$, and $\calD$ be a distribution with finite fourth moment. Finally, define the function
	\begin{align*}
	\calF(q;\rho):=\frac{\rho q}{2}\left(\Exp[X_0^2] - \frac{q}{2}\right) - \info( X_0,  \sqrt{\rho q}X_0 + Z_0) ~,
	\end{align*}
	where where $X_0 \sim \calD$, $Z_0 \sim \mathcal{N}(0,1)$, $X_0 \perp Z_0$ and $\info( \cdot,\cdot)$ denotes the mutual information between the first and second argument. Then, $\lim_{d \to \infty} F_d(\rho;\calD) = \frac{1}{4}\sup_{q \ge 0}\calF(q;\rho^2),$ and, whenever $\argmax_{q \ge 0}\calF(q;\lambda)$ is unique, 
	\begin{align*}
	\lim_{d \to \infty} \crossoff_d(\rho;\calD)=  \lim_{d \to \infty} 4F_d'(\rho;\calD) &=  \left(\argmax_{q \ge 0 } \calF(q;\rho)\right)^2 = \frac{\partial }{\partial \rho}\left(\argmax_{q \ge 0 } \calF(q; \rho)\right).
	\end{align*}
\end{thm}
In other words, the above theorem gives an explicit formula to compute $\lim_{d \to \infty} \crossoff_d(\rho;\calD)$. When $\calD$ is of the form $\calN(\mu,1)$, this limiting expression can be expressed as follows:

\begin{lem}\label{lem:Gauss_comp} 
	In the setting where $\calD = \calN(1,\mu)$, we have for all $\rho > 1$, 
	\begin{align*}
	q_{\mu}(\rho) := \argmax_{q \ge 0 } \calF(q;\sqrt{\rho})   =  \frac{1+\mu^2 - \frac{1}{\rho} + \sqrt{( 1+\mu^2 - \frac{1}{\rho})^2 + \frac{4\mu^2}{\rho}}  }{2}.
	\end{align*}
	Moreover, $q_{\mu}(\rho) \le 1+\mu^2 - \frac{1}{\rho} + \frac{|\mu|}{\sqrt{\rho}}$, and moreover $L(\mu) := |\sup_{\rho \ge 1} \frac{\partial}{\partial \rho}(q_{\mu}(\rho))^2| \lesssim 1 + \frac{1}{\mu^2} + \mu^2$. 
	\end{lem}

	Combining the results we have established thus far, we can bound our quantity of interest in terms of an asymptotic express, and the asymptotic difference $\Freech_{d}\left(\rho;\boldalpha \mu\right) - F_d(\rho;\mu)$:

\begin{cor} For any $\rho \ge 1$ and $\mu > 0$, we have 
\begin{align*}
&\limsup_{d \to \infty} \Exp_{\boldalpha}\left[\I(|\boldalpha - 1|\le d^{-1/4}) \cross_{d}\left(\rho^{1/2};\boldalpha \mu\right)\right] \\
&\qquad\le \left(1+\mu^2 - \frac{1}{\rho} + \frac{|\mu|}{\sqrt{\rho}}\right)^2 + 4\limsup_{d \to \infty} \sup_{\boldalpha:|\boldalpha - 1|\le d^{-1/4}}\left| \Freech_{d}\left(\rho;\boldalpha \mu\right) - F_d(\rho;\mu)\right|.
\end{align*}
\end{cor}
\begin{proof} We have that
\begin{align*}
& \Exp_{\boldalpha}\left[\I(|\boldalpha - 1|\le d^{-1/4}) \crossch_{d}\left(\rho;\boldalpha \mu\right)\right] \tag*{(Lemma~\ref{lem:cross_eq})}\\
&= 4\Exp_{\boldalpha}\left[\I(|\boldalpha - 1|\le d^{-1/4}) \Freech_d'(\rho;\boldalpha\mu) \right]\tag*{(Lemma~\ref{lem:free_cross})}\\
&= 4F'_d(\rho;\mu) + 4\left| \Exp_{\boldalpha}\left[\I(|\boldalpha - 1|\le d^{-1/4}) \Freech_d'(\rho;\boldalpha\mu) \right] - F'_d(\rho;\mu) \right|\\
&\le 4F'_d(\rho;\mu) + 4\sup_{\boldalpha: |\boldalpha - 1|\le d^{-1/4}} \left|  \Freech_d'(\rho;\boldalpha\mu) - F'_d(\rho;\mu) \right|.
\end{align*}
Taking the limit $d \to \infty$, we have by Theorem~\ref{thm:LelargeMain} and Lemma~\ref{lem:Gauss_comp} that $\lim_{d \to \infty} 4F'_d(\rho;\mu)=  (q_{\mu}(\rho))^2 \le (1+\mu^2 - \frac{1}{\rho} + \frac{|\mu|}{\sqrt{\rho}})^2$, as needed.
\end{proof}

To conclude our demonstration, it remains to show that
\begin{align}\label{eq:unif_conv}
\limsup_{d \to \infty} \sup_{\boldalpha:|\boldalpha - 1|\le d^{-1/4}}\left| \Freech_{d}'\left(\rho;\boldalpha \mu\right) - F_d'(\rho;\mu)\right| = 0.
\end{align}
Because comparing the derivatives $\Freech'$ and $F'$ directly is quite challenging, we will approach the bound indirectly by first showing that $\left| \Freech_{d}\left(\rho;\boldalpha \mu\right) - F_d'(\rho;\mu)\right|$ is small, and then using convexity to conclude convergence of the derivatives. Specifically, we will adopt the following strategy:
\begin{enumerate}
	\item We show that for every sufficiently small $\eta > 0$, there exists a $d_0(\eta,\mu)$ sufficiently large such that for all $d \ge d_0(\eta,\mu)$, 
	\begin{align*}
	\sup_{\rho \in [1,2] } \sup_{\boldalpha:|\boldalpha - 1|\le d^{-1/4}}\left| \Freech_{d}\left(\rho;\boldalpha \mu\right) - F_d(\rho;\mu)\right| \le \eta.
	\end{align*}
	This is a direct consequence of the following estimate, which is in the spirit of the Wasserstein continuity of the Mutual Information~\citep{wu2012functional}:
	\begin{lem}\label{lem:free_energy_bound}Fix $\mu,\much,\epsilon$ such that $|\much - \mu| \le \epsilon$ and $\epsilon \le 1$. Then, there is a universal constant $C$ such that
	\begin{align*}
	\left|\Freech_{d}\left(\rho;\much\right) - F_d(\rho;\mu)\right| \le C\rho\left(\frac{(1+\mu)^4}{d} + \frac{(1+\mu)^3}{\epsilon}\right)
	\end{align*}
	\end{lem}

	\item Next, set $L(\mu) := |\sup_{\rho \ge \rho_0} \frac{\partial}{\partial \rho}q_{\mu}^2(\rho)| \lesssim 1 + \frac{1}{\mu^2}$, and let $t(\eta,\mu) := \sqrt{2\eta/L(\mu)}$. We then show that if $\eta$ is small enough that $1 \le \rho - t(\eta,\mu) \le \rho + t(\eta,\mu) \le 2$, then there exists a $d_1(\eta,\mu)$ such that for all $d \ge d_1(\eta,\mu,\rho)$, it holds that
	\begin{align}
	F_d'(\rho;\mu) + t(\eta,\mu)\cdot 2L(\mu) &\ge \frac{F_d(\rho) - F_d(\rho - t(\eta,\mu))}{t(\eta,\mu)} \nonumber\\
	F_d'(\rho) - t(\rho,\mu) \cdot (2L(\mu)) &\le \frac{F_d(\rho + t(\eta,\mu)) - F_d(\rho)}{t(\eta,\mu)},\label{eq:lipschitz_like}
	\end{align}
	We can verify this equation as follows. 
	\begin{proof} We shall show that $F_d'(\rho;\mu) + t(\eta,\mu)\cdot 2L(\mu) - \frac{F_d(\rho) - F_d(\rho - t(\eta,\mu))}{t(\eta,\mu)} \ge 0$; the other inequality follows similiarly. By Theorem~\ref{thm:LelargeMain} and Lemma~\ref{lem:Gauss_comp}, we know that there is a function $F_{\infty}(\rho)$ with $F_{\infty}'(\rho) = \frac{1}{4} (q_{\mu}(\rho))^2$ such that $F_d(\rho;\mu) \overset{d\to \infty}{\to} F_{\infty}(\rho)$ and $F_d'(\rho;\mu)\overset{d\to \infty}{\to} F_{\infty}(\rho)$. 

	Hence, for all $d \ge d_1(\eta,\mu,\rho)$, we can ensure that 
	\begin{align*}
	&F_d'(\rho;\mu) + t(\eta,\mu)\cdot 2L(\mu) - \frac{F_d(\rho) - F_d(\rho - t(\eta,\mu))}{t(\eta,\mu)} \\
	&\ge F_{\infty}'(\rho;\mu) + t(\eta,\mu)\cdot L(\mu) - \frac{F_{\infty}(\rho) - F_{\infty}(\rho - t(\eta,\mu))}{t(\eta,\mu)},
	\end{align*}
	Since $F_{\infty}'(\rho) = \frac{1}{4} q_{\mu}(\rho)$,  $\sup_{\rho \ge 1} \frac{\partial}{\partial \rho}(q_{\mu}(\rho))^2 \le L(\mu)$, $F_{\infty}'(\rho)$ is $L(\mu)$ Lipschitz for $\rho \ge 1$. By the intermediate value theorem, this implies that $F_{\infty}'(\rho;\mu) + t(\eta,\mu)\cdot L(\mu) - \frac{F_{\infty}(\rho) - F_{\infty}(\rho - t(\eta,\mu))}{t(\eta,\mu)} \ge 0$, as needed.
	\end{proof}

	\item To conlude, we invoke a following lemma which gives quantitative bound on the difference between the derivatives of two functions provided their maximal distance on a small interval is small. 
	\begin{lem}\label{lem:cvx_approx_refined} Let $\epsilon, L > 0$, and let $t = \sqrt{2\epsilon/L}$. Let $f$ and $g$ be differentiable functions on an interval $[x-t,x+t]$. Suppose further that $g$ is convex, and that $f$ satisifies
	\begin{align}\label{eq:lipschitz_like}
	f'(x) + tL \ge \frac{f(x) - f(x-t)}{t} \quad \text{and} \quad f'(x) - tL \le \frac{f(x+t) - f(x)}{t},
	\end{align}
	and that $\sup_{x' \in [x-t,x+t]} |f(x') - g(x')| \le \epsilon$. Then, $|f'(x) - g'(x)| \le 2\sqrt{2\epsilon L}$.
	\end{lem}
	\begin{proof} The proof is identical to the convex approximation of Lemma~\ref{lem:cvx_approximation}. 
	\end{proof}
	This concludes the proof of~\eqref{eq:unif_conv}: indeed, for $d \ge d_0(\eta,\mu) \vee d_1(\eta,\mu,\rho)$, let $f(\rho) = F_d(\rho;\mu)$ and $g(\rho) = \Freech_d(\rho;\boldalpha \mu)$, where $|\boldalpha - 1|\le d^{-1/4}$. Then $g(\rho)$ is convex by Lemma~\ref{lem:free_cross}. Since $\eta$ is small enough that $[\rho - t(\eta,\mu), \rho - t(\eta,\mu)] \subset [1,2]$, Part (1) implies
	\begin{align*}
	\sup\{|f(\rho') - g(\rho')|: \rho' \in [\rho - t(\eta,\mu), \rho - t(\eta,\mu)]\} \le \sup\{|f(\rho') - g(\rho')|: \rho' \in [1,2]\} \le \eta,
	\end{align*}
	while~\eqref{eq:lipschitz_like} holds when setting $L \leftarrow 2L(\mu)$. Hence, Lemma~\ref{eq:lipschitz_like} implies $\left| \Freech_{d}'\left(\rho;\boldalpha \mu\right) - F_d'(\rho;\mu)\right| \le 4\sqrt{L(\mu)\eta}$. Since this holds for any $d \ge d_0(\eta,\mu) \vee d_1(\eta,\mu)$ and $|\boldalpha - 1|\le d^{-1/4}$, we have
	\begin{align*}
	\sup_{d \ge d_0(\eta,\mu) \vee d_1(\eta,\mu)} \sup_{\boldalpha:|\boldalpha - 1|\le d^{-1/4}}\left| \Freech_{d}'\left(\rho;\boldalpha \mu\right) - F_d'(\rho;\mu)\right| \le 4\sqrt{L(\mu)\eta}.
	\end{align*}
	Taking $\eta \to 0$ concludes the proof. 
\end{enumerate}

\subsection{Proof of Corollary~\ref{lem:Gauss_comp}\label{sec:lem_Gauss_comp}} 

	To compute $\arg\max_{q \ge 0} \calF(q;\rho)$, observe that for any $\gamma$, we have
	\begin{align*}
	\info(X_0, \sqrt{\gamma}X_0 + Z_0) &= \info(X_0 - \Exp[X_0], \sqrt{\gamma}(X_0 - \Exp[X_0]) + Z_0)\\
	&= \info( X_0', \sqrt{\gamma}X_0' + Z_0) \text{ where } X_0' \sim \mathcal{N}(0,1)~\\
	&= \frac{1}{2}\log(1 + q\gamma),
	\end{align*}
	where the last line is a standard identity (see e.g. Equation 11 in~\cite{guo2005mutual}). 
	We may then compute 
	\begin{eqnarray*}
	F(q;\rho) &:=& \frac{\rho q}{2}\left(\Exp[X_0^2] - \frac{q}{2}\right) - \frac{1}{2}\log \left( 1 +  q \rho\right)\\
	F'(q;\lambda) &=& \frac{\lambda^2}{2}\left(\Exp[X_0^2]  - q\right) - \frac{\rho}{2 (1 + q\rho)} \\
	&=& \frac{\rho}{2}\left( \left(\Exp[X_0^2]  - q\right) - \frac{1}{1 + q\rho} \right)~.
	\end{eqnarray*}
	Setting $F'(q;\rho) = 0$, we see that 
	\begin{eqnarray*}
	0 &=& (1 + q\rho) \left(\left(\Exp[X_0^2]  - q\right)\right) - 1 \\
	&=&  \left(\frac{1}{\rho} + q\right) \left( \Exp[X_0^2]  - q\right) - \frac{1}{\rho}\\
	&=&  -\left\{ q^2 - q\left(\Exp[X_0^2] - \frac{1}{\rho}\right) + \frac{1}{\rho}(\Exp[X_0^2] - 1)\right)~.
	\end{eqnarray*}
	Since $\Exp[X_0^2] - 1 = \mu^2 \ge 0$, we see that the discriminant of the above quadratic is nonnegative and thus its roots are
	\begin{align*}
	\frac{\Exp[X_0^2] - \frac{1}{\rho} \pm \sqrt{( \Exp[X_0^2] - \frac{1}{\rho})^2 + \frac{4}{\rho}(\Exp[X_0^2]- 1)}  }{2}~.
	\end{align*}
	\begin{claim} For $\rho > 1$, maximizer is obtained by the root corresponding to the $+$-sign.
	\end{claim}
	\begin{proof}
	Because $\Exp[X_0^2] - \frac{1}{\rho} > 0$ for $\rho > 1$, the root corresponding to the '$+$'-sign is nonnegative. If $\Exp[X_0^2]- 1  = \mu^2 > 0$, then the rooting corresponding to '$-$' is negative, and thus the $'+'$ root is the unique maximizer. In the edge-case where $\mu^2 = 0$, then the $-$-root is at $q = 0$, the $+$ root is $\Exp[X_0^2] - \frac{1}{\rho}$. In~\cite{lelarge2016fundamental}, it is verified that the latter value of $q$ corresponds to the maximizer. 
	\end{proof}
	We therefore conclude:
	\begin{align}
	\arg\max F(q;\rho) &= \frac{\Exp[X_0^2] - \frac{1}{\rho} + \sqrt{( \Exp[X_0^2] - \frac{1}{\rho})^2 + \frac{4}{\rho}(\Exp[X_0^2]- 1)}  }{2} \nonumber\\
	&= \frac{1+\mu^2 - \frac{1}{\rho} + \sqrt{( 1+\mu^2 - \frac{1}{\rho})^2 + \frac{4\mu^2}{\rho}}  }{2} \label{eq:mmse_intermediate}\\
	&\le 1+ \mu^2 - \frac{1}{\rho} + \frac{|\mu|}{\lambda} \nonumber~.
	\end{align}
	Finally, the bound $L(\mu) := |\sup_{\rho \ge 1} \frac{\partial}{\partial \rho}(q_{\mu}(\rho))^2| \lesssim 1 + \frac{1}{\mu^2}$ follows from standard calculus.

\subsection{Proof of Lemma~\ref{lem:free_energy_bound}}

We now wish to show that that $F_d(\rho;\mu)$ and $\Freech_d(\rho;\much)$ are close. Let $(\matX,\matY,\matZ,\matXch,\Ych,\Zch)$ have the joint distribution:
\begin{align*}
\matXch_i &= \matX_i + (\much - \mu),\\
\matY_{ij} &= \matZ_{ij} + \sqrt{\frac{\rho}{d}}\matX_i\matX_j,\quad i < j\\
\Ych_{ij} &= \Zch_{ij} + \sqrt{\Ach_{ij}\frac{\rho}{d}}\matXch_i\matXch_j, \quad i \le j\\
\Zch_{ij},\matZ_{ij} &\iidsim \calN(0,1), \quad \matX &\iidsim \calN(\mu,1).
\end{align*}
Then, we see that the marginals satisify $(\matXch,\Ych,\Zch) \sim \Psfch_d(\rho;\much)$ and $(\matX,\matY,\matZ) \sim \Psfoff_d(\rho;\mu)$. Recalling the definition of the Hamiltonians
\begin{align*}
	\Hch_{d,\rho}(X) := \sum_{i \le j} \sqrt{\frac{\rho \Ach_{ij}}{d}}X_i\Zch_{ij} + \frac{\Ach_{ij}\rho}{d}\matXch_i\matXch_jX_iX_j - \frac{\Ach_{ij}\rho}{2d}(X_iX_j)^2,
\end{align*}
and similarly, for $(\matX,\matY,\matZ)\sim \Psfoff_d(\rho;\calD)$,
\begin{align*}
	\Hoff_{d,\rho}(X) := \sum_{i < j } \sqrt{\frac{\rho}{d}}X_i\matZ_{ij} + \frac{\rho}{d}\matX_i\matX_jX_iX_j - \frac{\rho}{2d}(X_iX_j)^2,
\end{align*}
We now introduce the notation $\Xch_i := X_i + (\much - \mu)$, and  define the interpolated Hamiltonian
\begin{align*}
H_{\rho, t}(X) := \Hoff_{d,\rho t}(X)  + \Hch_{d,\rho(1-t)}(\Xch), 
\end{align*}
so that $H_{\rho, 0}(X) = \Hch_{d,\rho}(\Xch)$ and $H_{\rho,1}(X) = \Hoff_{d,\rho}(X)$. Defining the interpolation function
\begin{align*}
\phi(t) := \frac{1}{d}\Exp\left[\log \left(\Exp_{X  \sim \calN(\mu,1)}e^{H_{\rho,t}(X)}\right) \right],
\end{align*}
we therefore see that $\phi(0) = \Freech_d(\rho;\much) $ and $\phi(1) = F(\rho)$. We now compute that
\begin{align*}
\phi'(t)  = \frac{1}{d}\Exp\left[ \frac{\Exp_{X_i \iidsim \calN(\mu,1)}e^{H_{\rho,t}(X)} \frac{\partial}{\partial t} H_{\rho,t}(X)  }{\Exp_{X_i \iidsim \calN(\mu,1)}e^{H_{\rho,t}(X)}}  \right] := \frac{1}{d}\Exp[\langle H_{\rho,t}(X)  \rangle_t]
\end{align*}
where we let $\langle \cdot \rangle_t$ denote the (random) measure where 
\begin{align*}\langle f(X) \rangle_t := \frac{\Exp_{X  \sim \calN(\mu,1)}f(X) e^{H_{\rho,t}(X)} }{\Exp_{X_i  \iidsim \calN(\mu,1)}e^{H_{\rho,t}(X)}}.
\end{align*} 
Introduce the indicator $A_{ij} = \I( i < j)$, and the notation $\Xbar$

we can compute 
\begin{align*}
\frac{\partial}{\partial t} H_{\rho,t}(X) &= \sum_{i \le j} \frac{1}{2}\sqrt{\frac{ \rho A_{ij}}{t d}} \matZ_{ij}X_iX_j + \frac{ \rho A_{ij}}{d}X_iX_j\matX_i\matX_j - \frac{ \rho A_{ij}}{2d}(X_iX_j)^2 \\
&\qquad -  \frac{1}{2}\sqrt{\frac{\rho \Ach_{ij}}{(1-t) d}} \matZ_{ij}\Xch_i\Xch_j - \frac{\rho \Ach_{ij}}{d}\Xch_i\Xch_j\matXch_i\matXch_j + \frac{\rho \Ach_{ij}}{2d}(\Xch_i\Xch_j)^2\\
&= \sum_{i \le j} \frac{1}{2}\sqrt{\frac{ \rho }{ d}} (\sqrt{\frac{A_{ij}}{t}} \matZ_{ij} X_iX_j - \sqrt{\frac{\Ach_{ij}}{1-t}}\matZbar_{ij}\Xch_i\Xch_j)\\
&\qquad+ \frac{ \rho }{d}(A_{ij}X_iX_j\matX_i\matX_j - \Ach_{ij}\Xch_i\Xch_j\matXch_i\matXch_j) - \frac{\rho }{2d}\left(A_{ij}(X_iX_j)^2 -\Ach_{ij}(\Xch_i\Xch_j)^2\right).
\end{align*}
Next, define the shorthand random variable $W_{ij} := \frac{1}{2}\sqrt{\frac{\rho }{ d}}\cdot\sqrt{\frac{A_{ij}}{t}} X_iX_j$ and $\Wch_{ij} := \frac{1}{2}\sqrt{\frac{\rho }{ d}}\cdot\sqrt{\frac{A_{ij}}{1-t}} X_iX_j$. Below, we verify the following computation, proved in Section~\ref{sec:lem_Gauss_comp}:
\begin{lem}\label{lem:Gaussian_Integnration}
\begin{align*}
\Exp[\langle \matZ_{ij}W_{ij} \rangle_t ] &= \frac{\rho A_{ij}}{2d} \Exp\left[\langle X_i^2 + X_j^2 - X_{i}X_j\matX_i\matX_j \rangle_{t}\right],
\end{align*}
and the analogue holds for $\Exp[\langle \Zch_{ij}\Wch_{ij} \rangle_t ] $. 
\end{lem}
It then follows that 
\begin{align*}
\phi'(t) = \frac{1}{d}\Exp[\langle H_{\rho,t}(X)  \rangle_t] = \sum_{i \le j} \frac{\rho }{ 2 d}\Exp\left[\langle A_{ij}X_iX_j\matX_i\matX_j - \Ach_{ij}\Xch_i\Xch_j\matXch_i\matXch_j \rangle_t \right ].
\end{align*}
Finally, since $|\matXch_i - \matX_i| = |\Xch_i - X_i| = |\mu - \much| \le \epsilon$ for all $i \in [d]$, repeated applications of the triangle inequality yield:
\begin{align*}
 |A_{ij}X_iX_j\matX_i\matX_j - \Ach_{ij}\Xch_i\Xch_j\matXch_i\matXch_j| &\le |A_{ij}-\Ach_{ij}|X_iX_j\matX_i\matX_j + \epsilon \left( |X_j\matX_i\matX_j| + |\Xch_i \Xch_j \matX_j| + |\Xch_i \Xch_j \matXch_j| \right)\\
 &\lesssim |A_{ij}-\Ach_{ij}|X_iX_j\matX_i\matX_j + \epsilon (\epsilon + |X_i| + |X_j| + |\matX_i| + |\matX_j|)^3\\
 &\lesssim |A_{ij}-\Ach_{ij}|X_iX_j\matX_i\matX_j| + \epsilon (\epsilon^3 + |X_i|^3 + |X_j|^3 + |\matX_i|^3 + |\matX_j|^3)\\
 &\lesssim |A_{ij}-\Ach_{ij}(|X_iX_j|^2 + |\matX_i\matX_j|^2) + \epsilon (\epsilon^3 + |X_i|^3 + |X_j|^3 + |\matX_i|^3 + |\matX_j|^3).
\end{align*}
Hence, we can bound 
\begin{align*}
|\phi'(t)| &\lesssim \frac{\rho}{d^2}\sum_{i \le j}|A_{ij}-\Ach_{ij}|\left(\Exp[\langle  |X_iX_j|^2\rangle_t] + \Exp[\langle|\matX_i\matX_j|^2\rangle_t]\right) + \epsilon (\epsilon^3 + \Exp[\langle|X_i|^3 + |X_j|^3\rangle_t] + \Exp[\langle |\matX_i|^3 + |\matX_j|^3 \rangle_t)\\
&\overset{(i)}{=} \frac{\rho}{d^2}\sum_{i \le j}|A_{ij}-\Ach_{ij}|2\Exp[|\matX_i\matX_j|^2 + \epsilon (\epsilon^3 + 2\Exp[|\matX_i|^3 + |\matX_j|^3]) \\
&\overset{(ii)}{=} \frac{\rho}{d^2}\sum_{i \le j}|A_{ij}-\Ach_{ij}|(1+\mu)^4 + \epsilon ( \epsilon^3 + (1+\mu)^3)\\
&\overset{(iii)}{\lesssim} \frac{1}{d}(1+\mu)^4 + \epsilon(1 + \mu^3) 
\end{align*}
where $(i)$ uses the Nishimori Identity (See, e.g. Proposition 16 in\cite{lelarge2016fundamental}),  $(ii)$ uses standard formulae for Gaussian moments, and $(iii)$ uses $\epsilon \le 1$ and the bound $\sum_{i \le j}|A_{ij}-\Ach_{ij}| = \sum_{i }|A_{ii}-\Ach_{ii}|= \frac{d}{2}$, since $A_{ij} = \I(i < j)$and $\Ach_{ij} = 1 - \frac{1}{2}\I(i = j)$.

Integrating, it follows that 
\begin{align*}
|F(\rho) - \Fbar(\rho)| = |\int_{0}^1 \phi'(t)dt | \lesssim  \frac{1}{d}(1+\mu)^4 + \epsilon(1 + \mu^3),
\end{align*}
as nededed.

\begin{proof}[Proof of Lemma~\ref{lem:Gaussian_Integnration}] We compute $\Exp[\langle \matZ_{ij}W_{ij} \rangle_t ]$; the computation of $\Exp[\langle \Zch_{ij}\Wch_{ij} \rangle_t ]$ is analogous. We use Gaussian integration by parts to compute
\begin{align*}
\Exp[\langle \matZ_{ij}W_{ij} \rangle_t ] &= \Exp[ \frac{\partial}{\partial_{\matZ_{ij}}} \langle W_{ij} \rangle_t ] \\
&\overset{(i)}{=} \Exp\left[\langle W_{ij}\frac{\partial}{\partial_{\matZ_{ij}}} \Hamil_{\rho,t }(X) \rangle_{t} -  \langle W_{ij} \rangle_{t}\langle_t \frac{\partial}{\partial_{\matZ_{ij}}} \Hamil_{\rho,t }(X)\rangle_{t}\right],
\end{align*}
where $(i)$ follows from writing $\langle W_{ij} \rangle_t = \frac{\Exp_{X  \sim \calD}W_{ij} e^{H_{\rho,t}(X)} }{\Exp_{X  \sim \calD}e^{H_{\rho,t}(X)}}$, and differentiating. Next, we compute 
\begin{align*}
\frac{\partial}{\partial_{\matZ_{ij}}} \Hamil_{\rho,t }(X) = \sqrt{\frac{\rho t A_{ij}}{d}} X_iX_j
\end{align*}
In particular, we can compute  
\begin{align*}
\Exp[\langle W_{ij}\frac{\partial}{\partial_{\matZ_{ij}}} \Hamil_{\rho,t }(X) \rangle_t] &= \Exp[\langle \frac{\rho}{2d}  A_{ij} X_i^2X_j^2\rangle_t]\\
\Exp[\langle W_{ij} \rangle_{t}\langle \frac{\partial}{\partial_{\matZ_{ij}}} \Hamil_{\rho,t }(X)\rangle_{t}] &= \frac{\rho A_{ij}}{2d} \Exp[\langle X_{i}X_j \rangle_{t} \langle X_{i} X_j \rangle_t]\\
&= \frac{\rho A_{ij}}{2d} \Exp[\langle X_{i}X_j\matX_i\matX_j \rangle_{t}], 
\end{align*}
where the last line follows from the Nishimori Idenity (see e.g. \citet[Proposition 16]{lelarge2016fundamental}). Hence, by linearity of $\langle \cdot \rangle_t$,
\begin{align*}
\Exp[\langle \matZ_{ij}W_{ij} \rangle_t ] = \frac{\rho A_{ij}}{2d} \Exp\left[\langle X_i^2 + X_j^2 - X_{i}X_j\matX_i\matX_j \rangle_{t}\right] ,
\end{align*}
as needed.

\end{proof}

\subsection{Proof of Proposition~\ref{prop:conj_rates}\label{sec:conj_rates}}  Recall that by Proposition~\ref{prop:ovlap_conversion}, we have that
\begin{align}\label{eq:ovlap_bound}
\ovlapd(\taub) \le \frac{ \sqrt{\Exp_{\boldalpha}\I(|\boldalpha - 1| \le d^{-\frac{1}{4}})\crossch_{d}(\rho_{\taub};\boldalpha \mu_{\taub}) }}{1+\taub}  + c_1d^{-c_2}~,
\end{align}
where, 
\begin{align}
\crossch_{d}(\rho;\mu) &:= \Exp_{\uch,\Mch}\left[ \|\Exp[\uch\uch^{\top} | \Mch] \|_{\F}^2 \right]~ \nonumber\\
\text{where } \Mch &:= \matW + \sqrt{\rho} \uch\uch^\top, \uch_i \overset{(i.i.d)}{\sim} \mathcal{N}(\mu/\sqrt{d},1/d)~. \label{eq:ovlap_mu}
\end{align} 
With a little bit of algebra, we see that it suffices to show that under either Conjecture~\ref{conj:nonasymp} Part (a) or (b), that there are constants $c_1,\dots,c_5$ such that for $\mu \in (0,1)$ abd $\rho \in (1,\sqrt{2})$,
\begin{align}\label{eq:cross_term_bound_nonasym}
|\crossch_{d}(\rho;\mu) - \lim_{d \to \infty}\crossch_{d}(\rho;\mu)| \le c_0 d^{-c_1}(\rho^{c_5} - 1)^{-c_2} \taub^{-c_3}(1+\mu^{-c_4})~.
\end{align}
Because the proof is quite similar to the proof of Theorem~\ref{thm:asmp_ovlap_thm}, while keeping track of polynomial error terms, we shall keep the remainder of proof to a sketch. 

\textbf{Proof from Conjecture~\ref{conj:nonasymp} Part (a)}
We simply write $\Exp_{\uch,\Mch}\left[ \|\uch \uch^\top -  \Exp[\uch\uch^{\top} | \Mch] \|_{\F}^2 \right] = \Exp_{\uch,\Mch}\left[ \|\uch \uch^\top  \|_{\F}^2 \right] - \Exp_{\uch,\Mch}\left[\|\Exp[\uch\uch^{\top} | \Mch] \|_{\F}^2 \right]$, and note that $\Exp_{\uch,\Mch}\left[ \|\uch \uch^\top  \|_{\F}^2 \right] = (1+\mu^2) + \BigOh{1/d}$. Hence, under our conjecture with $\lambda = \rho^2$,
\begin{align*}
|\crossch_{d}(\rho;\mu) - \lim_{d \to \infty}\crossch_{d}(\rho;\mu)| \le c_0 d^{-c_1}(\rho^{2} - 1)^{-c_2} \taub^{-c_3}(1+\mu^{-c_4})~ + \BigOh{1/d},
\end{align*}
as needed.

\textbf{Proof from Conjecture~\ref{conj:nonasymp} Part (b)}
Let us rescale to the $\Psfch_d(\rho;\mu)$ distribution from the previous section. Let $\Mch = \matW + \sqrt{\rho} \uch\uch^\top, \uch_i \overset{(i.i.d)}{\sim} \mathcal{N}(\mu/\sqrt{d},1/d)$, and let $(\matXch,\Ych,\Zch) \sim \Psfch_d(\rho;\mu)$. By the same argument as in Lemma~\ref{lem:cross_eq}, we have that
\begin{align*}
\info(\matW + \sqrt{\rho} \uch\uch^\top;\uch\uch^\top) = \info(\Ych;\matXch)
\end{align*}
By an analogue of \citet[Corollary 14]{lelarge2016fundamental}, we see that
\begin{align*}
\info(\Ych;\matXch) = \frac{1}{d^2}\Exp_{1 \le i \le j \le d}[\matXch_i^2] - \Freech_d(\rho;\mu) = (1+\mu^2)^2 - \Freech_d(\rho;\mu) + \BigOh{1/d},
\end{align*}
where we recall the free energy $ \Freech_d(\rho;\mu)$ from Definition~\ref{def:free_energy}. Hence, Conjecture~\ref{conj:nonasymp} Part (b) with $\rho^2 = \lambda$ implies that 
\begin{align*}
|\Freech_d(\rho;\mu) - \lim_{d \to \infty}\Freech_d(\rho;\mu)|  \le c_0 d^{-c_1}(\rho^{c_5} - 1)^{-c_2} \taub^{-c_3}(1+\mu^{-c_4}.
\end{align*}
Let $\Freech_{\infty}(\rho;\mu) = \lim_{d \to \infty}\Freech_d(\rho;\mu)$. As shown in the previous section, we have that 
\begin{align*}
4\Freech_{\infty}'(\rho;\mu) = (q_{\mu}(\rho))^2 \le 1+\mu^2 - \frac{1}{\rho} + \frac{|\mu|}{\sqrt{\rho}}.
\end{align*}
Since $\rho \mapsto (q_{\mu}(\rho))^2$ is $c_1 + c_2\mu^{-c_3}$-Lipschitz for $\rho \ge 1$, we can use Lemma~\ref{lem:cvx_approximation} can be used to show the derivatives 
\begin{align*}
|\Freech_d'(\rho;\mu) - \Freech_{\infty}'(\rho;\mu)|
\end{align*}
converge at the requisite  rate. Recognizing that  $4\Freech_d(\rho;\mu) = \crossch_{d}(\rho;\mu)$ (see Lemma~\ref{lem:free_cross}) and, as shown in the previous section, $\Freech_{\infty}'(\rho;\mu) = 4\lim_{d\to \infty } \Freech_d'(\rho;\mu) = 4\lim_{d\to \infty }\crossch_{d}(\rho;\mu)$, we see that
\begin{align*}
|\crossch_d(\rho;\mu) - \lim_{d \to \infty}\crossch_{d}(\rho;\mu)|
\end{align*}
coverges at the desired rate.

\end{document}